\documentclass[journal]{IEEEtran}
\ifCLASSINFOpdf
\else
\fi

\usepackage{epsfig}\usepackage{epsfig}
\usepackage{amssymb}
\usepackage{amsthm}
\usepackage{amsmath}
\usepackage{cite}
\usepackage{mathrsfs}
\usepackage{cases}
\usepackage{algorithm}
\usepackage{multirow}
\usepackage{makecell}
\usepackage{url}
\usepackage{algorithmic}
\usepackage{subeqnarray}
\usepackage{subfigure}
\graphicspath{ {figure/} }
\usepackage{pifont,bbm,amsfonts,mathrsfs,amsmath,stmaryrd,amssymb}
\usepackage{amscd,graphicx,array,dsfont,texdraw,tikz,footnote,verbatim}%
\usepackage[normalem]{ulem}
\usepackage{mathtools}
\usepackage{bbm,bbding,amssymb,pifont,mathrsfs,amsfonts,graphicx,mathtools,color}%
\usepackage{amscd,array,enumerate,dsfont,texdraw,tikz,multicol,bbm}

\usetikzlibrary{arrows,shapes,snakes,shadows,positioning,automata,patterns}
\usetikzlibrary{trees,decorations.pathmorphing,decorations.markings}
\usepackage{setspace} 
\usepackage{tikz}
\usepackage{pgf}
\usepackage{amsmath, amsfonts, amssymb, amsthm}
\usepackage{pgfplots}
\pgfplotsset{compat=newest}
\usepackage{xxcolor}
\usetikzlibrary{arrows, decorations.markings}
\usetikzlibrary{arrows,shadows,petri}

\newcommand{\beq}{\begin{equation}}
\newcommand{\eeq}{\end{equation}}
\newcommand{\bqa}{\begin{eqnarray}}
\newcommand{\eqa}{\end{eqnarray}}

\definecolor{maroon}{rgb}{0.7,0,0}

\definecolor{ngreen}{rgb}{0.3,0.7,0.3}

\definecolor{golden}{rgb}{0.8,0.6,0.1}

\newtheorem{theorem}{\indent Theorem}
\newtheorem{lemma}{\indent Lemma}

\newtheorem{definition}{\indent Definition}
\newtheorem{assumption}{\indent Assumption}
\newtheorem{myremark}{\indent Remark}

\newenvironment{remark}{\begin{myremark}\normalfont}
	{\end{myremark}}

\begin{document}
\title{Locally Differentially Private Gradient Tracking for Distributed Online Learning over Directed Graphs}

\author{Ziqin Chen and Yongqiang Wang, \textit{Senior Member, IEEE}
\thanks{The work was supported in part by the National Science Foundation under Grants  ECCS-1912702, CCF-2106293, CCF-2215088, and CNS-2219487.}
\thanks{Ziqin Chen and Yongqiang Wang are with the Department
of Electrical and Computer Engineering, Clemson University, Clemson, SC 29634 USA (e-mail: ziqinc@clemson.edu; yongqiw@clemson.edu).}}

\maketitle
\begin{abstract}
Distributed online learning has been proven extremely effective in solving large-scale machine learning problems over streaming data. However, information sharing between learners in distributed learning also raises concerns about the potential leakage of individual learners' sensitive data. To~mitigate this risk, differential privacy, which is widely regarded as the ``gold standard" for privacy protection, has been widely employed in many existing results on distributed online learning. However, these results often face a fundamental tradeoff between learning accuracy and privacy. In this paper, we propose a locally differentially private gradient tracking based distributed online learning algorithm that successfully circumvents this tradeoff. We prove that the proposed algorithm converges in mean square to the exact optimal solution while ensuring rigorous local differential privacy, with the cumulative privacy budget guaranteed to be finite even when the number of iterations tends to infinity. The algorithm is applicable even when the communication graph among learners is directed. To the best of our knowledge, this is the first result that simultaneously ensures learning accuracy and rigorous local differential privacy in distributed online learning over directed graphs. We evaluate our algorithm's performance by using multiple benchmark machine-learning applications, including logistic regression of the ``Mushrooms" dataset and 
CNN-based image classification of the ``MNIST" and ``CIFAR-10" datasets, respectively. The experimental results confirm that the proposed algorithm outperforms existing counterparts in both training and testing accuracies.
\end{abstract}
\IEEEpeerreviewmaketitle
\section{Introduction}\label{Introduction}
Machine learning is rapidly reshaping the landscape of various engineering domains, ranging from wireless sensor networks~\cite{sensor}, autonomous driving~\cite{automaticdrive} to image classification~\cite{image}. Different from the conventional centralized learning scheme, where all data are stored on one device, distributed learning enables multiple participating learners to cooperatively learn a common optimal solution while each participating learner only trains on its own local dataset. Hence, compared with centralized learning, distributed learning provides inherent advantages in scalability and privacy, and thereby has garnered increased attention over the past decade~\cite{distributedconcept1,offline1,offline2,offline3}.

In existing distributed learning approaches, the most commonly used algorithm is distributed stochastic gradient descent (DSGD)~\cite{DSGD1,DSGD2}. While DSGD is communication-efficient and simple to implement, it suffers from slow convergence when data are heterogeneous among learners~\cite{DSGDpoor1,DSGDpoor2}. To mitigate the 
issue brought by data heterogeneity, gradient tracking based distributed optimization algorithms have emerged~\cite{offlineGT1,offlineGT2,offlineGT3,offlineGT4}, which replace the local gradient in every learner's update in DSGD with an estimated global gradient. Besides the classical gradient-tracking approach which requires balanced network topologies, this approach has also been extended to the case with general directed network topologies in both others' works~\cite{pushpull,Pushpull2,pushi,offlineGTdirect,timevaryingofflineGT1,timevaryingofflineGT2} and our prior work~\cite{Wanggradient}. All aforementioned gradient tracking based algorithms consider a fixed and static objective function, which, in machine learning, amounts to requiring all training data to be available beforehand. However, in numerous real-world applications, the data are sequentially acquired~\cite{stream1}, which prompts the investigation of online gradient tracking based algorithms~\cite{onlineGT1,onlineGT2,onlineGT3}.

Moreover, in existing online gradient tracking based algorithms, repeated message exchanges are required among neighboring learners, which poses significant privacy threats to individual learners' sensitive datasets. As shown in~\cite{rawinfer1,rawinfer2}, even though raw data are not shared during distributed training, external adversaries could infer individuals' sensitive information from shared messages. To address privacy concerns in distributed learning/optimization, various approaches have been proposed. For example, partially homomorphic cryptography has been widely considered in distributed optimization~\cite{rawinfer2,homomorphic1}. But this approach incurs a high communication and computation cost. Another approach involves the injection of spatially- or temporally-correlated noises to obfuscate information shared in distributed optimization~\cite{Spatially1,Spatially2,Spatially3,Spatially4,Spatially5}. However, to protect the privacy of a learner, this approach requires this learner to have at least one trustworthy neighbor, which is undesirable for fully distributed applications. As the de facto standard for privacy protection, differential privacy (DP) is gaining increased traction, and has been employed in many distributed learning and optimization algorithms~\cite{DPdefinition, distributedofflineDP1,distributedofflineDP2,Dingtie,DOLA,DMOD,DOwu,lihuaqing,Youkeyou,distributedonlineDP3,shiling,Tailoring,Wangquantization}. 

Most existing DP solutions for distributed learning and optimization only consider undirected or balanced graph topologies~\cite{distributedofflineDP1,distributedofflineDP2,Dingtie,DOLA,DMOD,DOwu}. Recently, results have emerged on DP design for distributed learning under directed graph topologies, for both online learning~\cite{lihuaqing,Youkeyou,distributedonlineDP3} and conventional (offline) learning where data are predetermined~\cite{shiling,Tailoring}. However, existing results on differentially private distributed online learning over directed graphs usually build on the combination of push-sum \cite{lihuaqing} based or eigenvector-estimation~\cite{Youkeyou,distributedonlineDP3} based and DSGD, which limits their achievable convergence speeds. To the best of our knowledge, no results 
have been reported on DP design for gradient tracking based online learning over directed graphs.

Another limitation with existing DP solutions~\cite{distributedofflineDP1,distributedofflineDP2,Dingtie,DOLA,DMOD,DOwu,lihuaqing,Youkeyou,distributedonlineDP3,shiling} for distributed learning and optimization applications is that these solutions are subject to a fundamental tradeoff between privacy and learning/optimization accuracy. Recently, our work~\cite{Tailoring} successfully circumvents this tradeoff and achieves both optimality and privacy. Nevertheless, this approach relies on incorporating two weakening factors into inter-agent interaction to mitigate the impact of DP noises, which consequently slows down algorithmic convergence. Furthermore, this approach is designed for an offline setting, where all data are predetermined. Our recent work~\cite{ziji1} proposed a local differential privacy (LDP) approach for distributed online learning that can ensure learning accuracy and privacy simultaneously. However, this approach requires undirected graphs. In addition, it also hinges on a weakening factor, which significantly decreases the speed of algorithm convergence.

In this work, we introduce an LDP approach for distributed online learning over directed graphs that ensures both learning accuracy and rigorous LDP (with the privacy budget guaranteed to be finite even when the number of iterations tends to infinity). Specifically, we first modify the conventional architecture of gradient tracking to ensure learning accuracy despite the presence of persistent DP noises. This modification is crucial because DP noises will accumulate in the estimate of the global gradient in conventional gradient-tracking algorithms. In fact, in the presence of persistent DP noises, the variance of accumulated noises will grow to infinity in conventional gradient tracking based distributed optimization, which significantly affects learning accuracy, as confirmed in our theoretical analysis in~Sec.~\ref{motivation} and experimental results in~Sec.~\ref{experiment}. It is worth noting that while the approach in~\cite{pushi} for offline optimization can prevent the accumulated noise variance in gradient estimation from growing to infinity, it cannot entirely eliminate the influence of noises on optimization accuracy. In contrast, our algorithm effectively eliminates the influence of persistent DP noises on local gradient estimation, and thus ensures accurate convergence. Then, we prove that the proposed algorithm can ensure LDP with a finite cumulative privacy budget, even in the infinite time horizon. Furthermore, by leveraging the online eigenvector-estimation technique in~\cite{eigenvector}, our proposed algorithm enables each learner to locally estimate the left normalized Perron eigenvector of the interaction graph, which allows the treatment of imbalanced graphs and hence applications in general directed networks. To the best of our knowledge, this is the first work that successfully achieves LDP in gradient tracking based distributed online learning and optimization over directed graphs. The main contributions are summarized as follows:

\begin{itemize}
\item We prove that the proposed distributed online learning algorithm converges in mean square to the optimal solution, even when the DP noises are persistent and the communication graph is directed. Note that existing online gradient-tracking algorithms in~\cite{onlineGT1,onlineGT2,onlineGT3} employ the conventional gradient-tracking approach, which is susceptible to noises due to the accumulation of variance in gradient estimation~\cite{pushi,Wanggradient,Tailoring}.

\item In addition to ensuring accurate convergence, our algorithm also achieves rigorous LDP with a finite cumulative privacy budget, even in the infinite time horizon. This stands in stark contrast to most existing DP solutions for distributed learning and optimization \cite{distributedofflineDP1,distributedofflineDP2,Dingtie,DOLA,DMOD,DOwu,lihuaqing,Youkeyou,distributedonlineDP3,shiling}, where the cumulative privacy budget grows to infinity as the number of iterations tends to infinity (implying diminishing DP protection). 

\item Compared with existing DP solutions for distributed learning and optimization in~\cite{distributedofflineDP1,distributedofflineDP2,Dingtie,DOLA,DMOD,DOwu} which require balanced network topologies, our proposed algorithm is applicable to general directed network topologies. 

\item The adopted LDP framework preserves agent-level privacy for each learner's dataset without relying on any trusted third parties. This differs from the traditional DP framework employed in~\cite{distributedofflineDP1,distributedofflineDP2,Dingtie,DOLA,DMOD,DOwu,lihuaqing,Youkeyou,distributedonlineDP3,shiling,Tailoring, Wangquantization},~where a ``centralized" data aggregator is implicitly assumed to determine the amount of injected noises. 

\item Different from our prior results in~\cite{Tailoring} and~\cite{ziji1} relying on weakening factors in inter-agent coupling to ensure both learning accuracy and privacy, which unavoidably reduce the speed of convergence, our algorithm here avoids using any weakening factors, and hence can attain faster convergence speed, as confirmed in our analytical comparison in~Sec.~\ref{learningaccuracy} and experimental results in~Sec.~\ref{experiment}.  

\item We evaluate the performance of our algorithm using multiple benchmark machine-learning applications, including online logistic regression of the ``Mushrooms" dataset and image classification of the ``MNIST" and~``CIFAR-10" datasets, respectively. Moreover, the experimental results show that compared with existing state-of-the-art DP solutions in~\cite{Dingtie,DOwu,Youkeyou,Tailoring,ziji1},
our proposed algorithm provides better training and testing accuracies.
\end{itemize}

The remainder of the paper is organized as follows. Sec. \ref{problemstatement} introduces some preliminaries and the problem formulation. Sec. \ref{algorithm} proposes our LDP approach for online gradient tracking. Sec. \ref{convergence} analyzes the learning accuracy of the proposed algorithm. Sec.~\ref{Privacy} establishes rigorous LDP guarantees. Sec. \ref{experiment} presents experimental results using standard machine-learning applications on benchmark datasets. Sec. \ref{conclusion} concludes the paper.

\section{Preliminaries and Problem Statement}\label{problemstatement}
\subsection{Notations}
We use~$\mathbb{R}$,~$\mathbb{R}^{+}$,~$\mathbb{N}$,~and~$\mathbb{N}^{+}$ to denote the sets of real numbers, positive real numbers, nonnegative integers, and positive integers, respectively. $\mathbb{R}^{n}$ represents the~$n$-dimensional real Euclidean space.~$\mathbf{1}_{n}$ and~$I_{n}$ denote the~$n$-dimensional column vector of all ones and the identity matrix, respectively. For an arbitrary vector~$x$, we denote its~$i$th element by~$[x]_{i}$. We write~$\langle\cdot,\cdot\rangle$ for the inner product of two vectors and~$\|\cdot\|$ for the standard Euclidean norm of a vector. For an arbitrary matrix $A$, we denote its transpose by~$A^{T}$ and its Frobenius norm by~$\|A\|_{F}$. We also use other vector/matrix norms defined under a certain transformation determined by a matrix $W$, which will be represented as $\|\cdot\|_{W}$. We write $\mathbb{P}[\mathcal{A}]$ for the probability of an event $\mathcal{A}$ and $\mathbb{E}[x]$ for the expected value of a random variable $x$. The notation~$\lceil a \rceil$ refers to the smallest integer not less than~$a$ and~$\lfloor a \rfloor$ represents the largest integer not greater than~$a$. We use $\text{Lap}(\nu_{i})$ to denote the Laplace distribution with a parameter $\nu_{i}\!>\!0$, featuring a probability density function $\frac{1}{2\nu_{i}}e^{\frac{-|x|}{\nu_{i}}}$. $\text{Lap}(\nu_{i})$ has a mean of zero and a variance of $2 \nu_{i}^2$.
\subsection{Network model}
We model the topology of the network over which learners communicate with each other as a directed graph~$\mathcal{G}=([m],\mathcal{E})$, where $[m]=\{1,\cdots,m\}$ denotes the set of agents (learners) and $\mathcal{E}\subseteq [m]\times[m]$ represents the set of edges consisting of ordered pairs of agents. Given a nonnegative matrix $W=\{w_{ij}\}\in\mathbb{R}^{m\times m}$, we define its induced directed graph as~$\mathcal{G}_{W}\!=\!([m],\mathcal{E}_{W})$, where~$(i,j)\in\mathcal{E}_{W}$ if and only if $w_{ij}>0$. For a learner~$i\in[m]$, it is able to receive messages from the learners in its in-neighbor set~$\mathcal{N}_{W,i}^{\text{in}}=\{j\in[m]|w_{ij}>0\}$; Similarly, learner $i$ can also send messages to learners in its out-neighbor set~$\mathcal{N}_{W,i}^{\text{out}}\!=\!\{j\in[m]|w_{ji}>0\}$. Graph~$\mathcal{G}_{W}$ is called strongly connected if there exists a directed path between any pair of distinct agents. In this paper, we consider a gradient tracking based algorithm which maintains two optimization variables~\cite{pushpull} that can be shared on two different graphs. We represent the two directed graphs as~$\mathcal{G}_{R}$ and $\mathcal{G}_{C}$, which are induced by matrices~$R\in\mathbb{R}^{m\times m}$ and~$C\in \mathbb{R}^{m\times m}$, respectively.
\begin{assumption}\label{a1}
The matrices~$R\in \mathbb{R}^{m\times m}$ and~$C\in \mathbb{R}^{m\times m}$ have nonnegative off-diagonal entries, i.e.,~$R_{ij}\geq 0$ and~$C_{ij} \geq 0$ for all $i\neq j$. Their diagonal entries are negative, satisfying
\begin{equation}
R_{ii}=-\sum_{j\in{\mathcal{N}_{R,i}^{\text{in}}}}R_{ij},\quad\quad C_{ii}=-\sum_{j\in{\mathcal{N}_{C,i}^{\text{out}}}}C_{ji},\nonumber
\end{equation}
such that $R\mathbf{1}=\mathbf{0}$ and $\mathbf{1}^{T}C=\mathbf{0}^{T}$. Moreover, the induced graph $\mathcal{G}_{R}$ is strongly connected and $\mathcal{G}_{C^{T}}$ contains at least one spanning tree.
\end{assumption}

Assumption~\ref{a1} is weaker than requiring both $\mathcal{G}_{R}$ and $\mathcal{G}_{C}$ to be strongly connected in~\cite{lihuaqing,Youkeyou,distributedonlineDP3}. We have the following lemma on matrices $R$ and $C$:

\begin{lemma}\label{l1}
\!\!\cite{pushpull} Under Assumption~\ref{a1}, the matrix $\mathbf{R}\triangleq I+R$ has a unique positive left eigenvector $u^{T}$ (corresponding to eigenvalue $1$) satisfying $u^{T}\mathbf{1}=m$, and the matrix $\mathbf{C}\triangleq I+C$ has a unique positive right eigenvector $\omega$ (corresponding to eigenvalue $1$) satisfying $\mathbf{1}^{T}\omega=m$.
\end{lemma}
\subsection{Local differential privacy}
Differential privacy guarantees that when two datasets differ by only one data point (record), the output of a DP implementation does not reveal whether that specific data point was utilized. This property makes it difficult for an external adversary to identify individual data entries among all possible ones, thereby providing strong privacy protection. 

In this paper, we consider 
an agent-level LDP framework, and thus, changes in a dataset are formalized by an adjacency relation pertaining to the local dataset of learner $i\in[m]$:
\begin{definition}\label{d1}
(Adjacency) Given two local datasets $\mathcal{D}_{i}=\{\xi_{i,1},\cdots,\xi_{i,T}\}$ and ${\mathcal{D}}'_{i}=\{{\xi}'_{i,1},\cdots,{\xi}'_{i,T}\}$ for all $i\in[m]$ and any time $T\in{\mathbb{N}^{+}}$, $\mathcal{D}_{i}$ and ${\mathcal{D}}'_{i}$ are adjacent if there exists a time instant $k\in\{1,\cdots,T\}$ such that $\xi_{i,k}\neq {\xi}'_{i,k}$ while $\xi_{i,t}={\xi}'_{i,t}$ for all $t\neq k,~t\in\{1,\cdots,T\}.$
\end{definition}

According to Definition~\ref{d1}, two local datasets $\mathcal{D}_{i}$ and ${\mathcal{D}}'_{i}$ are adjacent if and only if they differ by only one entry while all other entries are the same. We denote the adjacency relationship between $\mathcal{D}_{i}$ and ${\mathcal{D}}'_{i}$ by $\text{Adj}(\mathcal{D}_{i},{\mathcal{D}}'_{i})$. With this understanding, we formally define LDP as follows:
\begin{definition}\label{definition2}
(Local Differential Privacy) We say that an implementation $\mathcal{A}_{i}$ of a randomized algorithm by learner $i$ provides $\epsilon_{i}$-local differential privacy if for any adjacent datasets $\mathcal{D}_{i}$ and ${\mathcal{D}}'_{i}$, the following inequality holds:
\begin{equation}
\mathbb{P}[\mathcal{A}_{i}(\mathcal{D}_{i},\theta_{-i})\in \mathcal{O}_{i}] \leq e^{\epsilon_{i}} \mathbb{P}[\mathcal{A}_{i}({\mathcal{D}}'_{i},\theta_{-i})\in \mathcal{O}_{i}],\label{LDP}
\end{equation}
where $\theta_{-i}$ denotes all messages received by learner $i$ and $\mathcal{O}_{i}$ represents the set of all possible observations on learner $i$.
\end{definition}

The privacy budget of learner $i$'s implementation is quantified by $\epsilon_{i}$. It can be seen that a smaller $\epsilon_{i}$ indicates closer distributions of observations under adjacent datasets, thereby ensuring a higher level of privacy protection. 

\begin{remark}
The conventional ``centralized" DP framework used in~\cite{distributedofflineDP1,distributedofflineDP2,Dingtie,DOLA,DMOD,DOwu,lihuaqing,Youkeyou,distributedonlineDP3,shiling,Tailoring,Wangquantization} implicitly assumes mutual trust among learners to cooperatively decide each learner's DP-noise needed to satisfy an aggregative privacy budget~$\epsilon=\sum_{i=1}^{m}\epsilon_{i}$. In contrast, our LDP framework removes the need for such trust and allows individual learners to independently set (potentially heterogeneous) privacy budgets~$\epsilon_{i}$ and choose the corresponding DP noises according to their individual needs. Therefore, our LDP framework provides a stronger and more user-friendly privacy framework.
\end{remark}

\subsection{LDP approach for distributed online learning}
We consider a distributed online learning problem involving $m$ learners. Each learner only has access to its own private dataset. At each iteration $t$, learner~$i\in[m]$ acquires a data point~$\xi_{i,t}=(x_{i,t},y_{i,t})$, which is independently and identically sampled from an unknown distribution $\mathcal{P}_{i}$. Using the sample $x_{i,t}$ and the current model parameter $\theta_{i,t}$, learner $i$ predicts a label $\hat{y}_{i,t}=\langle\theta_{i,t},x_{i,t}\rangle$ with an associated loss $l(\theta_{i,t},\xi_{i,t})$, which quantifies the deviation between $\hat{y}_{i,t}$ and the true label $y_{i,t}$. This loss prompts learner $i$ to update its model parameter from $\theta_{i,t}$ to $\theta_{i,t+1}$. The objective is that, based on sequentially acquired data, all learners converge to the same optimal solution $\theta^*$ to the following stochastic optimization problem:
\begin{equation}
\text{min}_{\theta\in \mathbb{R}^{n}} F(\theta), \quad  F(\theta)=\frac{1}{m}\sum_{i=1}^{m}f_{i}(\theta),\label{primal}
\end{equation}
where $f_{i}(\theta )=\mathbb{E}_{\xi_{i}\sim\mathcal{P}_{i}}\left[l(\theta,\xi_{i})\right]$ represents the local objective function of learner $i$.

Some standard assumptions are introduced as follows:
\begin{assumption} \label{a2}
(i) Problem~\eqref{primal} has at least one optimal solution $\theta^*$.

(ii) The gradients of local objective functions are uniformly bounded, i.e., there exists some positive constant $D$ such that we have $\|\nabla f_{i}(\theta)\|_{2}\leq D$ for all $i\in[m]$ and $\theta\in \mathbb{R}^{n}$.

(iii) For any $\theta_{1},\theta_{2}\in \mathbb{R}^{n}$, there exists some $\mu\geq 0$ such that $F(\theta_{2})\geq F(\theta_{1})+\nabla F(\theta_{1})^{T}(\theta_{2}-\theta_{1})+\frac{\mu}{2}\|\theta_{1}-\theta_{2}\|_{2}^2$ holds.
\end{assumption}
\begin{assumption}\label{a3}
We assume that the data points $\{\xi_{i,t}\}$ are independent across iterations. In addition,

(i) $\mathbb{E}[\nabla l(\theta_{i,t},\xi_{i,t})|\theta_{i,t}]=\nabla f_{i}(\theta_{i,t})$;

(ii) $\mathbb{E}[\|\nabla l(\theta_{i,t},\xi_{i,t})-\nabla f_{i}(\theta_{i,t})\|_{2}^2|\theta_{i,t}]\leq \kappa^2$;

(iii) $\|\nabla l(\theta_{1},\xi_{i,t})-\nabla l(\theta_{2},\xi_{i,t})\|_{2}\leq L\|\theta_{1}-\theta_{2}\|_{2}$ for any $\theta_{1},\theta_{2}\in\mathbb{R}^{n}$.
\end{assumption}
Assumption~\ref{a2}(iii) is weaker than requiring all local objective functions $f_{i}(\theta)$ to be strongly convex (in, e.g.,~\cite{shiling,DOLA,DMOD}) or convex (in, e.g.,~\cite{DOwu,Youkeyou,lihuaqing,distributedonlineDP3}). Assumption~\ref{a3} is commonly used in distributed stochastic optimization~\cite{DSGD2,offlineGT2}.

Since the local objective function $f_{i}(\theta)$ is defined as an expectation over random data $\xi_{i}$ sampled from an unknown distribution $\mathcal{P}_{i}$, it is inaccessible in practice and an analytical solution to problem~\eqref{primal} is unattainable. To tackle this issue, we focus on solving the following empirical risk minimization problem with sequentially arriving data:
\begin{equation}
	\text{min}_{\theta\in \mathbb{R}^{n}} F_{t}(\theta),\quad F_{t}(\theta)=\frac{1}{m}\sum_{i=1}^{m}f_{i,t}(\theta),\label{primalt}
\end{equation}
where $f_{i,t}(\theta)=\frac{1}{t+1}\sum_{k=0}^{t}l(\theta,\xi_{i,k})$ denotes the empirical local objective function of each learner $i\in[m]$.

According to the law of large numbers, one has $\lim_{t\rightarrow\infty}\frac{1}{t+1}\sum_{k=0}^{t}l(\theta,\xi_{i,k})\!=\!\mathbb{E}_{\xi_{i}\sim\mathcal{P}_{i}}\!\left[l(\theta,\xi_{i})\right]\!=\!f_{i}(\theta)$~\cite{centraltheorem}. Hence, problem~\eqref{primalt} serves as an approximation to the original problem~\eqref{primal}. In the following lemma, we quantify the distance between the optimal solution $\theta_{t}^*$ to problem~\eqref{primalt} and the true optimal solution $\theta^*$ to problem~\eqref{primal}:
\begin{lemma}\label{l9}
Denote $\theta_{t}^*$ as the optimal solution to problem~\eqref{primalt} at time $t$ and $\theta^*$ as the optimal solution to the original stochastic optimization problem~\eqref{primal}. Under Assumption~\ref{a2} with $\mu>0$ and Assumption~\ref{a3}, we have
\begin{equation}
	\mathbb{E}\big[\|\theta_{t+1}^*-\theta_{t}^*\|_{2}^2\big]\leq16(\kappa^2+D^2)\Big(\frac{2}{\mu^2}+\frac{1}{L^2}\Big)(t+1)^{-2}.\label{L1result}
\end{equation}
Moreover, the following inequality always holds:
\begin{equation}
\mathbb{E}[\|\theta_{t}^{*}-\theta^{*}\|_{2}^2]\leq \frac{4\kappa^2}{\mu^2}(t+1)^{-1}.\label{approximateresult2}
\end{equation} 
\end{lemma}
\begin{proof}
The inequality~\eqref{L1result} can be obtained following the same line of reasoning of Lemma~1 in our prior work~\cite{ziji1}. 
	
We proceed to prove inequality~\eqref{approximateresult2}. Considering the relationship $F_t(\theta^*_t)\leq F_t(\theta^*)$, we obtain
	\begin{equation}
		F(\theta_t^*)-F(\theta^*)\leq\big(F(\theta^*_t)-F_t(\theta^*_t)\big)   -\big(F(\theta^*)-F_t(\theta^*)\big).\label{1L1}
	\end{equation}
	By applying the mean value theorem to~\eqref{1L1}, one has
	\begin{equation}
		\begin{aligned}
			F(\theta_t^*)-F(\theta^*)&\leq\big\langle \nabla F(\chi)-\nabla F_t (\chi), \theta_t^*-\theta^*\big\rangle_{2}\\
			&\leq \|\nabla F(\chi)-\nabla F_t (\chi)\|_{2}\|\theta_t^*-\theta^*\|_{2},\label{1L2}
		\end{aligned}
	\end{equation}
	where the variable $\chi$ is given by $\chi=\alpha\theta_t^*+(1-\alpha)\theta^*$ for some constant $\alpha\in(0,1)$.
	
	The definition $\nabla F(\chi)=\frac{1}{m}\sum_{i=1}^{m}\mathbb{E}[\nabla l(\chi,\xi_{i})]$ implies
	\begin{flalign}
			&\mathbb{E}\big[\|\nabla F_t (\chi)-\nabla F(\chi)\|_{2}\big]= \mathbb{E}\Bigg[\Bigg\|\frac{1}{m}\sum_{i=1}^{m}\nabla f_{i,t}(\chi)-\nabla F(\chi)\Bigg\|_{2}\Bigg]\nonumber\\
			&\leq \frac{1}{m}\sum_{i=1}^{m}\frac{1}{t+1}\sum_{k=0}^{t} \mathbb{E}\big[\|\nabla l(\chi,\xi_{i,k})-\mathbb{E}[\nabla l(\chi,\xi_{i,k})]\|_{2}\big].\label{1L3}
	\end{flalign}
	Given that the data points $\xi_{i,k}$ are independently and identically distributed across iterations, we use Assumption~\ref{a3}(ii) and the Lyapunov inequality $E[\|X\|]\!\leq\!(E[\|X\|^{p}])^{\frac{1}{p}},~\forall p\!\geq\!1$ to obtain
	\begin{equation}
		\begin{aligned}
			&\sum_{k=0}^{t} \mathbb{E}\left[\|\nabla l(\chi,\xi_{i,k})-\mathbb{E}[\nabla l(\chi,\xi_{i,k})]\|_{2}\right]\\
			&\leq\sqrt{\mathbb{E}\Bigg[\Bigg(\sum_{k=0}^{t}\left\|\nabla l(\chi,\xi_{i,k})-\mathbb{E}[\nabla l(\chi,\xi_{i,k})]\right\|_{2}\Bigg)^2\Bigg]}\\
			&= \sqrt{ \mathbb{E}\Bigg[\sum_{k=0}^{t}\left\|\nabla l(\chi,\xi_{i,k})-\nabla f_{i}(\chi)\right\|_{2}^2\Bigg]}\leq \kappa \sqrt{t+1}.\label{1L4}
		\end{aligned}
	\end{equation}
	
	Substituting~\eqref{1L4} into~\eqref{1L3} yields $\mathbb{E}\left[\|\nabla F_t (\chi)-\nabla F(\chi)\|_{2}\right]\!\leq\!\frac{\kappa}{\sqrt{t+1}}$. By using~\eqref{1L1} and~\eqref{1L2}, we have
	\begin{equation}
		\mathbb{E}\left[F(\theta_t^*)-F(\theta^*)\right]\le \frac{\kappa}{\sqrt{t+1}}\mathbb{E}\left[\|\theta_{t}^*-\theta^*\|_{2}\right].\label{1L5}
	\end{equation}
	Assumption~\ref{a2}(iii) with $\mu>0$ implies
	$\frac{\mu}{2}\|\theta_{t}^*-\theta^*\|_{2}^2\leq F(\theta_{t}^*)-F(\theta^*).$
	Combing this relation and~\eqref{1L5}, we arrive at
	\begin{equation}
		\frac{\mu}{2}\mathbb{E}\left[\|\theta_{t}^*-\theta^*\|_{2}^2\right]\leq \frac{\kappa}{\sqrt{t+1}}\mathbb{E}\left[\|\theta_{t}^*-\theta^*\|_{2}\right],\label{12}
	\end{equation}
	which implies $\mathbb{E}\left[\|\theta_{t}^*-\theta^*\|_{2}\right]\leq \frac{2\kappa}{\mu}(t+1)^{-\frac{1}{2}}$ and inequality~\eqref{approximateresult2}.
\end{proof}

With this understanding, our goal is to design a distributed learning algorithm on general directed graphs which enables individual learners to track the optimal solution $\theta_{t}^*$ to problem~\eqref{primalt} under the constraints of LDP and sequentially arriving data samples. Based on the convergence result in~\eqref{approximateresult2}, individual learners' parameters will also converge to the true optimal solution to problem~\eqref{primal}, even under the constraints of LDP and sequentially arriving data samples.

\section{Online gradient tracking with LDP}\label{algorithm}
In this section, we develop an online gradient tracking based distributed learning algorithm over directed graphs to solve problem~\eqref{primal} with ensured $\epsilon_{i}$-LDP. Before introducing our algorithm, we first show the limitation of conventional gradient-tracking algorithms under LDP constraints.
\subsection{The conventional gradient tracking accumulates DP noises in gradient estimation}\label{motivation}
To preserve privacy, persistent DP noises have to be added to messages shared in each iteration of distributed online learning. In conventional gradient tracking based algorithms, the injected DP noises will accumulate in the global gradient estimation, thereby significantly affecting learning accuracy.

We use the classic Push-Pull gradient-tracking algorithm in~\cite{pushpull} as an example to illustrate the idea. In the absence of LDP constraints, i.e., when no DP noise is introduced into the information exchange among learners, the Push-Pull algorithm can be described in matrix form as follows:
\begin{equation}
\left\{\begin{aligned}
&\boldsymbol{\theta}_{t+1}=\mathbf{R}\boldsymbol{\theta}_{t}-\lambda_{t}\boldsymbol{y}_{t},\\
&\boldsymbol{y}_{t+1}=\mathbf{C}\boldsymbol{y}_{t}+\nabla\boldsymbol{f}_{t+1}(\boldsymbol{\theta}_{t+1})-\nabla\boldsymbol{f}_{t}(\boldsymbol{\theta}_{t}),\nonumber
\end{aligned}
\right.
\end{equation}
where the matrices $\boldsymbol{\theta}_{t}$, $\boldsymbol{y}_{t}$, and $\nabla\boldsymbol{f}_{t}(\boldsymbol{\theta}_{t})$ are defined as $\boldsymbol{\theta}_{t}\!=\![\theta_{1,t},\cdots,\theta_{m,t}]^{T}\!\in\!\mathbb{R}^{m\times n}$, $\boldsymbol{y}_{t}\!=\![y_{1,t},\cdots,y_{m,t}]^{T}\!\in\! \mathbb{R}^{m\times n}$, and $\nabla\boldsymbol{f}_{t}(\boldsymbol{\theta}_{t})\!\!=\!\![\nabla f_{1,t}(\theta_{1,t}),\cdots,\nabla f_{m,t}(\theta_{m,t})]^{T}\!\!\in\!\! \mathbb{R}^{m\times n},$ respectively. The matrices $\mathbf{R}$ and $\mathbf{C}$ are defined in the statement of Lemma~\ref{l1}.

Using initialization $\boldsymbol{y}_{0}=\nabla\boldsymbol{f}_{0}(\boldsymbol{\theta}_{0})$, we obtain
\begin{equation}
\mathbf{1}^{T}\boldsymbol{y}_{t}=\mathbf{1}^{T}\nabla\boldsymbol{f}_{t}(\boldsymbol{\theta}_{t}),\nonumber
\end{equation}
which means that ensuring the consensus of all $y_{i,t}$, i.e., $y_{i,t}=\frac{1}{m}\mathbf{1}^{T}\boldsymbol{y}_{t}$, is sufficient to guarantee each learner to track the global gradient, i.e., $y_{i,t}=\frac{1}{m}\mathbf{1}^{T}\nabla\boldsymbol{f}_{t}(\boldsymbol{\theta}_{t})=\frac{1}{m}\sum_{i=1}^{m}\nabla f_{i,t}(\theta_{i,t})$.

To achieve $\epsilon_{i}$-LDP, DP noises have to be added to both shared variables $\boldsymbol{\theta}_{t}$ and $\boldsymbol{y}_{t}$. Then, the update of the conventional Push-Pull algorithm becomes
\begin{subequations}\label{classGTnoise}
\begin{numcases}{}
\boldsymbol{\theta}_{t+1}=\mathbf{R}\boldsymbol{\theta}_{t}+\boldsymbol{\vartheta}_{R,t}-\lambda_{t}\boldsymbol{y}_{t},\label{classGTnoisex}\\
\boldsymbol{y}_{t+1}=\mathbf{C}\boldsymbol{y}_{t}+\boldsymbol{\zeta}_{C,t}+\nabla\boldsymbol{f}_{t+1}(\boldsymbol{\theta}_{t+1})-\nabla\boldsymbol{f}_{t}(\boldsymbol{\theta}_{t}),\label{classGTnoisey}
\end{numcases}
\end{subequations}
where the matrices $\boldsymbol{\vartheta}_{R,t}\!\!\in\!\! \mathbb{R}^{m\times n}$ and $\boldsymbol{\zeta}_{C,t}\!\!\in\!\! \mathbb{R}^{m\times n}$ represent the DP noises injected on variables $\boldsymbol{\theta}_{t}$ and $\boldsymbol{y}_{t}$, respectively.

It can be seen that even under the condition $\boldsymbol{y}_{0}=\nabla\boldsymbol{f}_{0}(\boldsymbol{\theta}_{0})$, we can only establish the following relation through induction:
\begin{equation}
\mathbf{1}^{T}\boldsymbol{y}_{t}=\mathbf{1}^{T}\left(\nabla\boldsymbol{f}_{t}(\boldsymbol{\theta}_{t})+\sum_{k=0}^{t-1}\boldsymbol{\zeta}_{C,k}\right),\label{accuerror}
\end{equation}
which implies that the DP noise accumulates over time in the estimate of the global gradient. Therefore, when the gradient-estimate variable $\boldsymbol{y}_{t}$ is directly fed into the model parameter update~\eqref{classGTnoisex}, learning accuracy will be compromised. This prediction is corroborated by our experimental results in Fig. 2-Fig. 4. The issue of DP-noise accumulations also exists in other gradient tracking based algorithms for distributed learning and optimization.

\begin{remark}
To circumvent the accumulation of noises in gradient estimation, recent work~\cite{pushi} proposes a robust gradient-tracking method for distributed offline optimization. However, this method cannot completely eliminate the~influence of persistent information-sharing noises, and thus is subject to steady-state errors. Although our recent work~\cite{Wanggradient}~employs a weakening factor in inter-agent interaction to attenuate noise influence and ensure optimization accuracy, such a weakening factor decreases the coupling strength among agents, which in turn reduces the speed of algorithmic convergence. Notably, neither of these works consider privacy analysis. In fact, under the constant stepsize and noise variances employed in~\cite{pushi} or the single weakening factor used in~\cite{Wanggradient}, it is impossible to ensure rigorous LDP in the infinite time horizon.

Recent works~\cite{Dingtie,shiling} have investigated DP design for gradient-tracking algorithms. However, both of the results face the dilemma of trading optimization accuracy for privacy. To tackle this dilemma, our recent work~\cite{Tailoring} achieves accurate convergence and privacy protection simultaneously. However, this approach relies on two carefully designed weakening factors to attenuate the impact of DP noises. Such weakening factors significantly slow down algorithmic convergence, as substantiated by our experimental results in Fig. 2-Fig. 4. 

Moreover, all the aforementioned works~\cite{pushi,Wanggradient,Dingtie,shiling,Tailoring} require static and predetermined datasets, making them unsuitable to online learning scenarios where data arrives sequentially. To the best of our knowledge, no existing work has explored LDP design for gradient tracking based algorithms in an online setting.
\end{remark}

\subsection{LDP design for online gradient tracking}
We present Algorithm~1 to address problem~\eqref{primal} over directed graphs under the constraints of LDP and sequentially arriving data. The injected DP noises satisfy Assumption~\ref{a4}.
\begin{table}[H]
\renewcommand\arraystretch{1.2}
\centering
\begin{tabular}{p{0.95\linewidth}}
\Xhline{1.0pt}
{\textbf{Algorithm 1:} LDP design for distributed online learning (from learner $i$'s perspective)} \\ \hline 
\textbf{Initialization:} $\theta_{i,0}\in \mathbb{R}^{n}$, $s_{i,0}\in \mathbb{R}^{n}$, $z_{i,0}=\boldsymbol{e}_{i}\in \mathbb{R}^{m},$ where $\boldsymbol{e}_{i}$ has the $i$th element equal to one and all other elements equal to zero,
$R\in \mathbb{R}^{m\times m}$, $C\in \mathbb{R}^{m\times m}$, and the stepsize $\lambda_{t}=\frac{\lambda_{0}}{(t+1)^{v}}$ with $\lambda_{0}>0$ and $v\in(\frac{1}{2},1)$.
		
1: \textbf{for} $t=0,1,2,\cdots$ \textbf{do}\\
2: Using all available data up to time $t$, i.e., $\xi_{i,k}$ for $k\in [0,t]$ and the current parameter $\theta_{i,t}$, learner $i$ computes the gradient $\nabla f_{i,t}(\theta_{i,t})=\frac{1}{t+1}\sum_{k=0}^{t}\nabla l(\theta_{i,t},\xi_{i,k})$.\\
3: Push $s_{i,t}+\zeta_{i,t}$ to neighbors $j,~j\in{\mathcal{N}}_{C,i}^{\text{out}}$ and pull $s_{j,t}+\zeta_{j,t}$ from neighbors $j,~j\in{\mathcal{N}}_{C,i}^{\text{in}}$. The Laplace DP-noise $\zeta_{i,t}$ satisfies Assumption~\ref{a4}.\\
4: \textbf{Update tracking variable:}
\begin{equation}
s_{i,t+1}=(1+C_{ii})s_{i,t}+\sum_{j\in{\mathcal{N}}_{C,i}^{\text{in}}}C_{ij}(s_{j,t}+\zeta_{j,t})+\lambda_{t}\nabla f_{i,t}(\theta_{i,t}).\label{dynamics}
\end{equation}
5. Push $\theta_{i,t}+\vartheta_{i,t}$ to neighbors $j,~j\in{\mathcal{N}_{R,i}^{\text{out}}}$ and pull $\theta_{j,t}+\vartheta_{j,t}$ from $j,~j\in{\mathcal{N}_{R,i}^{\text{in}}}$. The Laplace DP-noise $\vartheta_{i,t}$ satisfies Assumption~\ref{a4}.\\
6. \textbf{Update model parameter:}
\begin{equation}
\theta_{i,t+1}=(1+R_{ii})\theta_{i,t}+\sum_{j\in{\mathcal{N}_{R,i}^{\text{in}}}}R_{ij}(\theta_{j,t}+\vartheta_{j,t})-\frac{s_{i,t+1}-s_{i,t}}{m[z_{i,t}]_{i}},\label{dynamictheta}
\end{equation}
where $[z_{i,t}]_{i}$ denotes the $i$th element of $z_{i,t}$.\\
7. \textbf{Locally estimate the left eigenvector of $\mathbf{R}$:}
\begin{equation}
z_{i,t+1}=z_{i,t}+\sum_{j\in{\mathcal{N}_{R,i}^{\text{in}}}}R_{ij}(z_{j,t}-z_{i,t}).\label{dynamicthetaz}
\end{equation}
8: \textbf{end} \\
\Xhline{0.9pt}	
\end{tabular}
\end{table}
\begin{assumption}\label{a4}
For every $i\in[m]$ and any time $t\geq 0$, the DP-noises $\zeta_{i,t}$ and  $\vartheta_{i,t}$ are zero-mean and independent across iterations. The noise variance $\mathbb{E}[\|\zeta_{i,t}\|_{2}^2]=(\sigma_{i,t,\zeta})^2$ satisfies
$\sigma_{i,t,\zeta}=\frac{\sigma_{i,0,\zeta}}{(t+1)^{\varsigma_{i,\zeta}}}$ with $\sigma_{i,0,\zeta}>0$ and $\varsigma_{i,\zeta}\in(\frac{1}{2},1).$ The noise variance $\mathbb{E}[\|\vartheta_{i,t}\|_{2}^2]=(\sigma_{i,t,\vartheta})^2$ 
satisfies $\sigma_{i,t,\vartheta}=\frac{\sigma_{i,0,\vartheta}}{(t+1)^{\varsigma_{i,\vartheta}}}$ with $\sigma_{i,0,\vartheta}>0$ and $\varsigma_{i,\vartheta}\in(\frac{1}{2},1).$ Moreover, the following inequality holds:
\begin{equation}
\max_{i\in[m]}\{\varsigma_{i,\zeta},\varsigma_{i,\vartheta}\}<v<1,
\end{equation}
where the parameter $v$ is the decaying rate of stepsize $\lambda_{t}$ in Algorithm~1.
\end{assumption}

The dynamics~\eqref{dynamics}-\eqref{dynamictheta} in Algorithm~1 can be written in the following matrix form:
\begin{subequations}
\begin{numcases}{}
\boldsymbol{s}_{t+1}=\mathbf{C}\boldsymbol{s}_{t}+\boldsymbol{\zeta}_{C,t}+\lambda_{t}\nabla\boldsymbol{f}_{t}(\boldsymbol{\theta}_{t}),\label{dcompacts}\\
\boldsymbol{\theta}_{t+1}=\mathbf{R}\boldsymbol{\theta}_{t}+\boldsymbol{\vartheta}_{R,t}-Z_{t}^{-1}(\boldsymbol{s}_{t+1}-\boldsymbol{s}_{t}),\label{dcompacttheta}
\end{numcases}
\end{subequations}
where the matrices $\boldsymbol{\theta}_{t}$,~$\boldsymbol{s}_{t}$, and $Z_{t}$ are given by $\boldsymbol{\theta}_{t}\!=\![\theta_{1,t},\!\cdots\!,\theta_{m,t}]^{T}\!\in\!\mathbb{R}^{m\times n}$, $\boldsymbol{s}_{t}\!\!=\!\![s_{1,t},\cdots,s_{m,t}]^{T}\!\!\in\!\!\mathbb{R}^{m\times n}$, and $Z_{t}\!=\!\text{diag}(m[z_{1,t}]_{1},\!\cdots\!,m[z_{m,t}]_{m})\!\!\in\!\!\mathbb{R}^{m\times m}$, respectively. The noise variances $\boldsymbol{\zeta}_{C,t}$ and $\boldsymbol{\vartheta}_{R,t}$ are defined as  $\boldsymbol{\zeta}_{C,t}\!=\![\zeta_{C1,t},\cdots,\zeta_{Cm,t}]^{T}\in \mathbb{R}^{m\times n}$ and $\boldsymbol{\vartheta}_{R,t}\!=\![\vartheta_{R1,t},\cdots,\vartheta_{Rm,t}]^{T}\in \mathbb{R}^{m\times n}$ with $\zeta_{Ci,t}\!=\!\sum_{j\in{\mathcal{N}_{C,i}^{\text{in}}}}C_{ij}\zeta_{i,t}$ and $\vartheta_{Ri,t}\!=\!\sum_{j\in{\mathcal{N}_{R,i}^{\text{in}}}}R_{ij}\vartheta_{i,t}$ for $1\leq i\leq m$, respectively.

In \eqref{dcompacttheta}, the difference $\boldsymbol{s}_{t+1}-\boldsymbol{s}_{t}$ is incorporated into the parameter update. This modification effectively addresses the issue of accumulating DP noises in global gradient estimation, as substantiated by the following relation:
\begin{equation}
\begin{aligned}
\mathbf{1}^{T}(\boldsymbol{s}_{t+1}-\boldsymbol{s}_{t})&\!=\!\mathbf{1}^{T}\left(\mathbf{C}\boldsymbol{s}_{t}+\boldsymbol{\zeta}_{C,t}+\lambda_{t}\nabla\boldsymbol{f}_{t}(\boldsymbol{\theta}_{t})-\boldsymbol{s}_{t}\right)\\
&\!=\!\mathbf{1}^{T}(\boldsymbol{\zeta}_{C,t}+\lambda_{t}\nabla\boldsymbol{f}_{t}(\boldsymbol{\theta}_{t})),\label{1s}
\end{aligned}
\end{equation}
where in the derivation we have used~\eqref{dcompacts} and $\mathbf{1}^{T}C=\mathbf{0}^{T}$ from Assumption~\ref{a1}. It is clear that unlike the conventional Push-Pull gradient-tracking algorithm~\eqref{classGTnoise}, where global gradient estimation $\boldsymbol{y}_{t}$ (which is subject to accumulating DP noises as per~\eqref{accuerror}) is directly incorporated into the model parameter update, thereby affecting learning accuracy, our Algorithm~1 effectively circumvents this issue.

In addition, we introduce a local variable $z_{i,t}$ in Algorithm~1 to enable each learner to locally estimate the left eigenvector $u^{T}$ of $\mathbf{R}$. This eliminates the need for global information $u^{T}$, ensuring that our algorithm can be implemented in a fully distributed manner. It is worth noting that since $z_{i,t}$ does not contain sensitive information, adding DP noises to it is unnecessary. Next, we present the following lemma to characterize the error of the eigenvector estimator:

\begin{lemma}\label{l2}
\!\!\cite{Wanggradient} Under Assumption~\ref{a1}, the variables $z_{i,t}$ in~\eqref{dynamicthetaz}, after scaled by $m$, converge to the left eigenvector $u^{T}=[u_{1},\cdots,u_{m}]^{T}$ of $\mathbf{R}$ with a geometric rate, i.e., there exist some constants $c_{z}>0$ and $\gamma_{z}\in(0,1)$ such that the following inequality holds for all $i\in[m]$ and $t\geq 0$:
\begin{equation}
\left|\frac{1}{m[z_{i,t}]_{i}}-\frac{1}{u_{i}}\right|\leq c_{z}\gamma_{z}^{t},
\end{equation}
where $[z_{i,t}]_{i}$ denotes the $i$th element of $z_{i,t}$.
\end{lemma}
\begin{remark}
Algorithm 1 avoids using weakening factors on inter-agent interaction to attenuate the influence of DP noises, which is key in our prior results~\cite{Tailoring} and~\cite{ziji1} to ensure both optimization accuracy and rigorous DP. Given that a weakening factor will gradually reduce the strength of inter-agent coupling, and hence unavoidably decrease the convergence speed, our algorithm can ensure faster convergence compared with~\cite{Tailoring} and~\cite{ziji1}, which is corroborated by our analytical comparison in Sec.~\ref{learningaccuracy} and experimental results in Sec.~\ref{experiment}.
\end{remark}

\section{Online Learning Accuracy Analysis}\label{convergence}
In this section, we quantify the learning accuracy of Algorithm~1. To this end, we present some useful lemmas.
\vspace{-0.5em}
\subsection{Supporting lemmas}\label{supportlemma}
\begin{lemma}\label{l3}
\!\!\cite{pushpull} Under Assumption~\ref{a1}, there exist vector norms $\|x\|_{R}\triangleq\|\tilde{R}x\|_{2}$ and $\|x\|_{C}\triangleq\|\tilde{C}x\|_{2}$ for all $x\in \mathbb{R}^{m}$, where $\tilde{R},~\tilde{C}\in{\mathbb{R}^{m\times m}}$ are some reversible matrices\footnote{As indicated in~\cite{pushpull} and~\cite{Restimate}, $\tilde{R}$ and $\tilde{C}$ are determined by $R$ and $C$, respectively. For our analysis of learning accuracy, only the existence of $\tilde{R}$ ($\tilde{C}$) suffices and its explicit expression is not necessary. A detailed discussion on $\tilde{R}$ ($\tilde{C}$) is available in Lemma 5 of~\cite{pushi}, as well as Lemma 5.6.10 of~\cite{matrixanalysis}.}, such that $\|\mathbf{R}-\frac{\mathbf{1}u^{T}}{m}\|_{R}<1$ and $\|\mathbf{C}-\frac{\omega\mathbf{1}^{T}}{m}\|_{C}<1$ are arbitrarily close to the spectral radius of $\mathbf{R}-\frac{\mathbf{1}u^{T}}{m}$ and $\mathbf{C}-\frac{\omega\mathbf{1}^{T}}{m}$, respectively. 
\end{lemma}
According to Lemma~\ref{l3} in~\cite{pushpull} and~\cite{Wanggradient}, we can know that the spectral radius of the matrix $\mathbf{R}-\frac{\mathbf{1}u^{T}}{m}$ is equal to $1-|\nu_{R}|<1$, where $\nu_{R}$ is an eigenvalue of $R$. Lemma~4 indicates that $\|\mathbf{R}-\frac{\mathbf{1}u^{T}}{m}\|_{R}$ is arbitrarily close to the spectral radius of $\mathbf{R}-\frac{\mathbf{1}u^{T}}{m}$, i.e., $1-|\nu_{R}|$. Without loss of generality, we denote $\|\mathbf{R}-\frac{\mathbf{1}u^{T}}{m}\|_{R}=1-\rho_{R}<1$, where $\rho_{R}$ serves as an arbitrarily close approximation of $|\nu|_{R}$. Similarly, we denote $\|\mathbf{C}-\frac{\omega\mathbf{1}^{T}}{m}\|_{C}=1-\rho_{C}<1$, where $\rho_{C}$ is an arbitrarily close approximation of $|\nu_{C}|$ with $\nu_{C}$ an eigenvalue of $C$.

Following~\cite{pushpull} and~\cite{Wanggradient}, we proceed to define the following matrix norms for any matrices $\boldsymbol{x}\in \mathbb{R}^{m\times n}$ and $\boldsymbol{y}\in \mathbb{R}^{m\times n}$:
\begin{equation}
	\begin{aligned}
		\|\boldsymbol{x}\|_{R}&=\left\|\left[\|\boldsymbol{x}^{(1)}\|_{R},\cdots,\|\boldsymbol{x}^{(n)}\|_{R}\right]\right\|_{2},\\
		\|\boldsymbol{y}\|_{C}&=\left\|\left[\|\boldsymbol{y}^{(1)}\|_{C},\cdots,\|\boldsymbol{y}^{(n)}\|_{C}\right]\right\|_{2},\label{normdef}
	\end{aligned}
\end{equation}
where $\boldsymbol{x}^{(i)}$ and $\boldsymbol{y}^{(i)}$ denote the $i$th column of $\boldsymbol{x}$ and $\boldsymbol{y}$ for $1\leq i\leq n$, respectively. The subscript $2$ denotes the $2$-norm.

\begin{lemma}\label{l4}
\!\!\cite{pushpull} Given an arbitrary norm $\|\cdot\|$, for any $M\in{\mathbb{R}^{m\times m}}$ and $\boldsymbol{x}\in{\mathbb{R}^{m\times n}}$, we have $\|M\boldsymbol{x}\|\leq\|M\|\|\boldsymbol{x}\|$. In particular, for any $m\in\mathbb{R}^{m\times 1}$ and $x\in\mathbb{R}^{1\times n}$, we have $\|mx\|=\|m\|\|x\|_{2}$.
\end{lemma}
\begin{lemma}\label{l5}
\!\!\cite{pushpull} According to the equivalence of all norms in a finite-dimensional space, there exist constants $\delta_{F,R},\delta_{R,F},\delta_{C,F},\delta_{R,C}, \delta_{F,C}>0$ such that for all $\boldsymbol{x}\in \mathbb{R}^{m\times n}$, we have $\|\boldsymbol{x}\|_{F}\leq \delta_{F,R}\|\boldsymbol{x}\|_{R}$, $\|\boldsymbol{x}\|_{R}\leq \delta_{R,F}\|\boldsymbol{x}\|_{F}$, $\|\boldsymbol{x}\|_{C}\leq \delta_{C,F}\|\boldsymbol{x}\|_{F}$, $\|\boldsymbol{x}\|_{R}\leq \delta_{R,C}\|\boldsymbol{x}\|_{C}$, and $\|\boldsymbol{x}\|_{F}\leq \delta_{F,C}\|\boldsymbol{x}\|_{C}$. 
\end{lemma}
\begin{lemma}\label{l6}
The relation $a\gamma^{t}\leq \frac{1}{t^2}$ always holds for all $t>0$ and $\gamma\!\in\!(0,1)$, where the constant $a$ is given by $a=\frac{(\ln(\gamma)e)^2}{4}$.
\end{lemma}
\begin{proof}
We consider a convex function $f(x)=-2\ln(x)-x\ln(\gamma):\mathbb{R}^{+} \rightarrow \mathbb{R}$, whose derivative is $f'(x)=-\frac{2}{x}-\ln(\gamma)$, implying the minimum point at $x^*=-\frac{2}{\ln(\gamma)}$ with the minimal value $f(x^*)=-2\ln(-\frac{2}{\ln(\gamma)})+\frac{2}{\ln(\gamma)}\ln(\gamma)=\ln(a).$ Hence, for any $t>0$, we have $f(t)\geq \ln a$, i.e., $-2\ln(t)-t\ln(\gamma)\geq \ln(a),$ which is equivalent to 
$\ln(\gamma^{t})\leq \ln(\frac{1}{at^2})$ and further implies Lemma~\ref{l6}.
\end{proof}
\subsection{Online learning accuracy analysis}\label{learningaccuracy}
In this subsection, we analyze the learning accuracy of Algorithm~1 under strongly convex and general convex objective functions, respectively.

For notational simplicity, we define $\bar{s}_{t}=\frac{\mathbf{1}^{T}\boldsymbol{s}_{t}}{m}$, $\bar{\theta}_{t}=\frac{u^{T}\boldsymbol{\theta}_{t}}{m}$, $\sigma_{\zeta}^{+}\!=\!\max_{i\in[m]}\{\sigma_{i,0,\zeta}\}$,  $\sigma_{\vartheta}^{+}\!=\!\max_{i\in[m]}\{\sigma_{i,0,\vartheta}\}$,
$\varsigma_{\zeta}\!=\!\min_{i\in[m]}\{\varsigma_{i,\zeta}\}$, $\varsigma_{\vartheta}\!=\!\min_{i\in[m]}\{\varsigma_{i,\vartheta}\}$, $\Pi_{\omega}=I-\frac{\omega\mathbf{1}^{T}}{m}$, $\Pi_{u}=I-\frac{\mathbf{1}u^{T}}{m}$, $\Pi_{U}=U^{-1}-\frac{\mathbf{1}\mathbf{1}^{T}}{m}$, and $\Pi_{U}^{e}=(I-\frac{\mathbf{1}u^{T}}{m})(Z_{t}^{-1}-U^{-1}).$

The following lemmas establish the convergence properties for $\mathbb{E}[\|\boldsymbol{s}_{t}-\omega\bar{s}_{t}\|_{C}^2]$ and $\mathbb{E}[\|\boldsymbol{\theta}_{t}-\mathbf{1}\bar{\theta}_{t}\|_{R}^2]$ under general convex objective functions.
\begin{lemma}\label{l7}
Under Assumptions~\ref{a1}-\ref{a4} with~$\mu\geq0$, the following relation holds for Algorithm~1:
\begin{equation}
	\vspace{-0.3em}
\mathbb{E}\big[\|\boldsymbol{s}_{t}\!-\!\omega\bar{s}_{t}\|_{C}^2\big]\leq c_{\boldsymbol{s}1}t^{-2}+c_{\boldsymbol{s}2}t^{-2v}+c_{\boldsymbol{s}3}t^{-2\varsigma_{\zeta}}.\label{l7result}
\vspace{-0.3em}
\end{equation}
Here, the constant $c_{\boldsymbol{s}1}$, $c_{\boldsymbol{s}2}$, and $c_{\boldsymbol{s}3}$ are given by
\begin{equation}
		\small
\left\{\begin{aligned}
&c_{\boldsymbol{s}1}=\max\left\{1-\rho_{C},c_{s0}\right\}\mathbb{E}\left[\|\boldsymbol{s}_{0}-\omega\bar{s}_{0}\|_{C}^2\right],\\
&c_{\boldsymbol{s}2}=\bar{c}_{s0}+\tau_{s1},\quad c_{\boldsymbol{s}3}=\bar{c}_{s0}+\tau_{s2},
\end{aligned}
\right.\nonumber
\end{equation}
with $c_{s0}\!\!=\!\!\frac{4}{(e\ln(1-\rho_{C}))^2}$, $\bar{c}_{s0}\!\!=\!\!\frac{8}{\rho_{C}(e\ln(1-\frac{\rho_{C}}{2}))^2}$,  $\tau_{s1}\!\!=\!\!\frac{2m\delta_{C,F}^2\|\Pi_{\omega}\|_{C}^2\|(\kappa^2+D^2)\lambda_{0}^2}{\rho_{C}}$, and $ \tau_{s2}=\frac{\tau_{s1}\rho_{C}\sum_{i,j}(C_{ij})^2(\sigma_{\zeta}^{+})^2}{2m(\kappa^2+D^2)\lambda_{0}^2}$.
\end{lemma}
\begin{proof}
See Appendix~A.
\end{proof}
\begin{lemma}\label{l8}
Under Assumptions~\ref{a1}-\ref{a4} with $\mu\geq0$, the following relation holds for Algorithm~1:
\begin{equation}
\mathbb{E}\left[\|\boldsymbol{\theta}_{t}-\mathbf{1}\bar{\theta}_{t}\|_{R}^2\right]\leq c_{\boldsymbol{\theta}1}t^{-2}+c_{\boldsymbol{\theta}2}t^{-2v}+c_{\boldsymbol{\theta}3}t^{-2\varsigma_{\vartheta}}+c_{\boldsymbol{\theta}4}t^{-2\varsigma_{\zeta}}.\label{l8result}
\end{equation}
Here, the constant $c_{\boldsymbol{\theta}1}$, $c_{\boldsymbol{\theta}2}$, $c_{\boldsymbol{\theta}3}$, and $c_{\boldsymbol{\theta}4}$ are given by
\begin{equation}
		\small
\left\{\begin{aligned}
&c_{\boldsymbol{\theta}1}=\max\left\{1-\rho_{R},c_{\theta0}\right\}\mathbb{E}\left[\|\boldsymbol{\theta}_{0}-\mathbf{1}\bar{\theta}_{0}\|_{R}^2\right]\\
&\quad\quad+\max\left\{1,c_{\theta0}\right\}\tau_{\theta1}\mathbb{E}\left[\|\boldsymbol{s}_{0}-\omega\bar{s}_{0}\|_{C}^2\right]+\tau_{\theta1}c_{\boldsymbol{s}1}(\bar{c}_{\theta0}+1),\\
&c_{\boldsymbol{\theta}2}=(\tau_{\theta1}c_{\boldsymbol{s}2}+\tau_{\theta2})(\bar{c}_{\theta0}+1),\\
&c_{\boldsymbol{\theta}3}=\tau_{\theta3}(\bar{c}_{\theta0}+1),\\
&c_{\boldsymbol{\theta}4}=(\tau_{\theta1}c_{\boldsymbol{s}3}+\tau_{\theta4})(\bar{c}_{\theta0}+1),\nonumber
\end{aligned}
\right.
\end{equation}
with $c_{\theta0}\!\!=\!\!\frac{4}{(e\ln(1-\rho_{R}))^2}$, $\bar{c}_{\theta0}\!\!=\!\!\frac{8}{\rho_{R}(e\ln(1-\frac{\rho_{R}}{2}))^2}$, $\tau_{\theta1}\!\!=\!\!\frac{4\delta_{R,C}^2\|C\|_{R}^2(\|\Pi_{U}\|_{R}^2+\|\Pi_{U}^{e}\|_{R}^2)}{\rho_{R}}$, $\tau_{\theta2}\!\!=\!\!\frac{2m\delta_{R,F}^2(\kappa^2+D^2)\tau_{\theta1}
\lambda_{0}^2}{\delta_{R,C}^2\|C\|_{R}^2}$, $\tau_{\theta3}\!\!=\!\!2\|\Pi_{u}\|_{R}^2\delta_{R,F}^2\sum_{i,j}(R_{ij})^2(\sigma_{\vartheta}^{+})^2$, and $\tau_{\theta4}=2\|\Pi_{U}+\Pi_{U}^{e}\|_{R}^2\delta_{R,F}^2\sum_{i,j}(C_{ij})^2(\sigma_{\zeta}^{+})^2$.
\end{lemma}
\begin{proof}
See Appendix~B.
\end{proof}
Based on Lemma~\ref{l7} and Lemma~\ref{l8}, we can establish the expected distance between the learned parameter $\theta_{i,t}$ and the optimal solution $\theta_{t}^*$ to problem~\eqref{primalt} in Theorem~\ref{t1}:
\begin{theorem}\label{t1}
Under Assumptions~\ref{a1}-\ref{a4} with $\mu>0$, the expected distance between the learned parameter $\theta_{i,t}$ and the optimal solution $\theta_{t}^{*}$ to problem~\eqref{primalt} satisfies
\begin{flalign}
	&\mathbb{E}\left[\|\theta_{i,t}-\theta_{t}^*\|_{2}^2\right]\leq 2\delta_{F,R}^2\left(c_{\boldsymbol{\theta}1}t^{-2}+c_{\boldsymbol{\theta}2}t^{-2v}+c_{\boldsymbol{\theta}3}t^{-2\varsigma_{\vartheta}}\right.\nonumber\\
	&\left.\quad+c_{\boldsymbol{\theta}4}t^{-2\varsigma_{\zeta}}\right)+2m\big(\max\{c_{\bar{\theta}1},c_{\bar{\theta}2},c_{\bar{\theta}17}\}t^{-1}+c_{\bar{\theta}3}t^{\alpha-v-1}\nonumber\\
	&\quad+c_{\bar{\theta}4}t^{\alpha-v-2}+c_{\bar{\theta}5}t^{\alpha-3v}+c_{\bar{\theta}6}t^{\alpha-v-2\varsigma_{\vartheta}}+c_{\bar{\theta}7}t^{\alpha-v-2\varsigma_{\zeta}}\nonumber\\
	&\quad +c_{\bar{\theta}8}t^{\alpha-6+v}+c_{\bar{\theta}9}t^{\alpha-v-4}
	+c_{\bar{\theta}10}t^{\alpha-4+v-2\varsigma_{\zeta}}+c_{\bar{\theta}11}t^{\alpha-2\varsigma_{\vartheta}}\nonumber\\
	&\quad+c_{\bar{\theta}12}t^{\alpha-2\varsigma_{\zeta}}+c_{\bar{\theta}13}t^{\alpha-v-\frac{1}{2}}+c_{\bar{\theta}14}t^{\alpha-2v}+c_{\bar{\theta}15}t^{\alpha-2}\nonumber\\
	&\quad+c_{\bar{\theta}16}t^{\alpha-2+v}=O(t^{-\beta}),\label{t1result}
\end{flalign}
for all $t>0$, where the rate $\beta$ satisfies~$\beta=\min\{v+\frac{1}{2}-\alpha,2-v-\alpha,2\varsigma_{\vartheta}-\alpha,2\varsigma_{\zeta}-\alpha\}$ with $\alpha\in(v,\frac{1+v}{2})$. The constants $c_{\boldsymbol{\theta}1}$ to $c_{\boldsymbol{\theta}4}$ are given in Lemma~\ref{l8} and the constants $c_{\bar{\theta}1}$ to $c_{\bar{\theta}17}$ are defined in~\eqref{cbart1},~\eqref{cbart2},~\eqref{cbart315}, and~\eqref{cbart16}.
\end{theorem}
\begin{proof}
See Appendix~C.
\end{proof}
We are now in position to present the learning accuracy of Algorithm~1 against the original optimal solution to problem~\eqref{primal} under strongly convex objective functions: 
\begin{theorem}\label{t2}
Denote $\theta^*$ as the optimal solution to the original stochastic optimization problem~\eqref{primal}. Under the conditions in Theorem~\ref{t1}, the parameters $\theta_{i,t}$ in Algorithm~1 will converge in mean square to $\theta^*$, i.e.,
\begin{equation}
\mathbb{E}\big[\|\theta_{i,t}-\theta^{*}\|^2\big]\leq \frac{8\kappa^2}{\mu^2}t^{-1}+2C_{1}^2t^{-\beta}=O(t^{-\beta}),\label{t2result}
\end{equation}
for all $t>0$. Moreover, the objective function value $F(\theta_{i,t})$ and the minimal objective function value $F(\theta^*)$ satisfy
\begin{equation}
\mathbb{E}\left[F(\theta_{i,t})-F(\theta^{*})\right]\!<\!\frac{2\sqrt{2}\kappa D}{\mu}t^{-\frac{1}{2}}+\sqrt{2}C_{1}Dt^{-\frac{\beta}{2}}=O(t^{-\frac{\beta}{2}}),\label{t2result2}
\end{equation}
where $C_{1}$ is given by~$C_{1}\!\!=\!\!\max_{1\leq i\leq 4, 1\leq j\leq 17}\{\sqrt{c_{\boldsymbol{\theta}i}},\sqrt{c_{\bar{\theta}j}}\}$ with $c_{\boldsymbol{\theta}1}$ to $c_{\boldsymbol{\theta}4}$ given in Lemma~\ref{l8} and $c_{\bar{\theta}1}$ to $c_{\bar{\theta}17}$ defined in~\eqref{cbart1},~\eqref{cbart2},~\eqref{cbart315}, and~\eqref{cbart16}.
\end{theorem}
\begin{proof}
Substituting~\eqref{approximateresult2} from Lemma~\ref{l9} and \eqref{t1result} from Theorem~\ref{t1} into the triangle inequality $\|\theta_{i,t}-\theta^*\|_{2}^2\leq 2\|\theta_{i,t}-\theta_{t}^*\|_{2}^2+2\|\theta_{t}^{*}-\theta^*\|_{2}^2$, we arrive at~\eqref{t2result}. 

Moreover, Assumptions~\ref{a2}(ii) and~\ref{a2}(iii) imply
\begin{equation}
	F(\theta_{i,t})-F(\theta^{*})\leq \|\nabla F(\theta_{i,t})\|_{2}\|\theta_{i,t}-\theta^{*}\|_{2}\leq D\|\theta_{i,t}-\theta^{*}\|_{2},\nonumber
\end{equation}
for all $t>0$. By taking the expectation and combining inequality~\eqref{t2result} with the Lyapunov inequality $E[\|X\|]\leq (E[\|X\|^{p}])^{\frac{1}{p}}$ for any $p\geq1$, we obtain~\eqref{t2result2}.
\end{proof}
Theorem~\ref{t2} establishes the convergence of Algorithm~1 to the optimal solution to problem~\eqref{primal} under persistent DP noises. This differs from most existing DP solutions for distributed learning and optimization~\cite{Dingtie,shiling,DOwu,DMOD,DOLA,lihuaqing,Youkeyou,distributedonlineDP3}, which are always subject to optimization errors under rigorous DP constraints. In fact, besides ensuring convergence accuracy, our algorithm guarantees rigorous LDP even in the infinite time horizon, which will be substantiated in Sec.~\ref{Privacy}.

Unlike most existing results on distributed online optimization~\cite{DOwu,DMOD,DOLA,lihuaqing,Youkeyou,distributedonlineDP3} which focus on dynamic or static regrets with respect to the optimal solution to problem~\eqref{primalt} (which only approximates the optimal solution to~\eqref{primal}), Theorem~\ref{t2} provides a direct quantitative measure of the learning error with respect to the optimal solution $\theta^*$ to the problem~\eqref{primal} at each iteration. Moreover, Theorem~\ref{t2} shows that the convergence speed of
Algorithm 1 is $O(t^{-\beta})$ with $\beta\!\!=\!\!\min\{v+\frac{1}{2}-\alpha,2-v-\alpha,2\varsigma_{\vartheta}-\alpha,2\varsigma_{\zeta}-\alpha\}$. This speed outpaces that of the distributed online learning algorithm in our prior work~\cite{ziji1} by a factor of $O(t^{\frac{v+1-2\alpha}{2}})$ (the convergence speed in~\cite{ziji1} is $O(t^{-\beta})$ with~$\beta\!\!=\!\min\{1\!-\!v,2\varsigma_{\vartheta}\!-\!1\}$). In addition, the algorithm in~\cite{ziji1} only characterizes the deviation between the learned parameter $\theta_{i,t}$ and the optimal solution $\theta_{t}^*$ to an approximated formulation of~\eqref{primal}. Hence, Theorem~\ref{t2} provides stronger and more precise convergence than the result in~\cite{ziji1}.

Next, we establish the convergence result for general convex objective functions.
\begin{theorem}\label{t3}
Denote $\theta^*$ as the optimal solution to the original stochastic optimization problem~\eqref{primal}. Under Assumptions~\ref{a1}-\ref{a4} with $\mu\geq0$, the objective function value $F(\theta_{i,t})$ converges in mean to the minimal objective function value $F(\theta^{*})$, i.e.,
\begin{equation}
\mathbb{E}\big[F(\theta_{i,t})-F(\theta^*)\big]\leq \frac{(1-v)\sum_{i=1}^{4}\bar{c}_{\boldsymbol{\theta}i}t^{v-1}}{2\lambda_{0}(1-\frac{1}{2^{1-v}})}=O(t^{-\bar{\beta}}),\label{t3result}
\end{equation}
for all $t>0$, where the rate $\bar{\beta}$ satisfies $\bar{\beta}=1-v$ and the constants $\bar{c}_{\boldsymbol{\theta}1}$ to $\bar{c}_{\boldsymbol{\theta}4}$ are given by
\begin{equation}
	\small
	\left\{\begin{aligned}
		&\bar{c}_{\boldsymbol{\theta}1}=6L^2\delta_{F,R}^2\lambda_{0}(2\lambda_{0}+1)\left(\frac{4c_{\boldsymbol{\theta}1}(v-r+2)}{v-r+1}+\frac{4^{v}c_{\boldsymbol{\theta}2}(3v-r)}{3v-r-1}\right.\\
		&\left.\quad+\frac{4^{\varsigma_{\vartheta}}c_{\boldsymbol{\theta}3}(v-r+2\varsigma_{\vartheta})}{v-r+2\varsigma_{\vartheta}-1}+\frac{4^{\varsigma_{\zeta}}c_{\boldsymbol{\theta}4}(v-r+2\varsigma_{\zeta})}{v-r+2\varsigma_{\zeta}-1}\right),\\
		&\bar{c}_{\boldsymbol{\theta}2}=\frac{2^{9}\|u\|_{2}^2c_{z}^2(2\lambda_{0}+1)}{(\ln(\gamma_{z})e)^{4}m^2\lambda_{0}}\Big(\frac{12\delta_{F,C}^2\|C\|_{C}^{2}c_{\boldsymbol{s}1}(6-v-r)}{5-v-r}\\
		&\quad+\frac{(3\times4^{v}\delta_{F,C}^2\|C\|_{C}^{2}c_{\boldsymbol{s}2}+6m(\kappa^2+D^2)\lambda_{0}^2)(4+v-r)}{3+v-r}\\
		&\quad+\frac{(3\!\times\!4^{\varsigma_{\zeta}}\delta_{F,C}^2\|C\|_{C}^{2}c_{\boldsymbol{s}3}\!+\!3(\sum_{i,j}C_{ij})^2(\sigma_{\zeta}^{+})^2)(4\!+\!2\varsigma_{\zeta}\!-\!v\!-\!r)}{3+2\varsigma_{\zeta}-v-r}\Big),\\
		&\bar{c}_{\boldsymbol{\theta}3}=\frac{4\|u\|_{2}^2(\sigma_{\vartheta}^{+})^2\sum_{i,j}(R_{ij})^2\varsigma_{\vartheta}}{m^2(2\varsigma_{\vartheta}-1)}+\frac{4(\sigma_{\zeta}^{+})^2\sum_{i,j}(C_{ij})^2\varsigma_{\zeta}}{m^2(2\varsigma_{\zeta}-1)}\nonumber\\
		&\quad+\frac{8v\lambda_{0}^2(\kappa^2+D^2)}{2v-1}+\frac{12\kappa^2(2v+1)\lambda_{0}^2}{v}+\frac{13\kappa^2\lambda_{0}(v+1-r)}{v-r},\\
		&\bar{c}_{\boldsymbol{\theta}4}=\frac{2(v+r)e^{\frac{2\lambda_0(r+v)}{r+v-1}}}{v+r-1}\Big(\mathbb{E}\left[\|\bar{\theta}_{0}\!-\!\theta^*\|_{2}^{2}\right]\!+\!\sum_{i=1}^{3}\bar{c}_{\boldsymbol{\theta}i}\Big)\\
		&\quad+(1+2\lambda_{0})\mathbb{E}[\|\bar{\theta}_{0}\!-\!\theta^*\|_{2}^{2}],\label{barc}
	\end{aligned}
	\right.
\end{equation}
with $r\in(\frac{1}{2},v)$ and the constants $c_{\boldsymbol{s}1}$ to $c_{\boldsymbol{s}3}$ and $c_{\boldsymbol{\theta}1}$ to $c_{\boldsymbol{\theta}4}$ given in the statements of Lemma~\ref{l7} and Lemma~\ref{l8}, respectively.
\end{theorem}
\begin{proof}
See Appendix D.
\end{proof}

Theorem~\ref{t3} characterizes the convergence of~$F(\theta_{i,t})$ to the minimal objective function value $F(\theta^{*})$. Moreover, the convergence speed specified in Theorem~\ref{t3} (i.e., $O(t^{\bar{\beta}})$ with $\bar{\beta}\!=\!1-v$) is twice as fast as that in our prior work~\cite{ziji1}, which converges at a speed of  $O(t^{-\bar{\beta}'})$ with $\bar{\beta}'\!=\!\frac{1-v}{2}$ for general convex objective functions.

\section{Local differential-privacy analysis}\label{Privacy}
In this section, we prove that besides accurate convergence, Algorithm~1 can also simultaneously ensure rigorous $\epsilon_{i}$-LDP for each learner, even in the infinite time horizon. To this end, we first introduce the concept of sensitivity for learner $i$'s implementation $\mathcal{A}_i$:
\begin{definition}\label{d3}
	(Sensitivity) Let $\mathcal{D}_{i}$~and~${\mathcal{D}}'_{i}$ be any two adjacent datasets for learner~$i$. The sensitivity of learner $i$'s implementation $\mathcal{A}_{i}$ at time $t$ is
	\begin{equation}
		\Delta_{i,t}=\max_{\text{Adj}(\mathcal{D}_{i},{\mathcal{D}}'_{i})}\|\mathcal{A}_{i}(\mathcal{D}_{i},\theta_{-i,t})-\mathcal{A}_{i}({\mathcal{D}}'_{i},\theta_{-i,t})\|_{1},\label{sensitive}
	\end{equation}
	where $\theta_{-i,t}$ represents all information learner~$i$~received from its neighbors. 
\end{definition}
According to Definition~\ref{d3}, under Algorithm~1, learner $i$'s implementation involves sensitivities: $\Delta_{i,t,s}$ and $\Delta_{i,t,\theta}$, which correspond to the two shared variables $s_{i,t}$ and $\theta_{i,t}$, respectively.

With this understanding, we have the following lemma:
\begin{lemma}\label{l11}
At each time $t\geq 0$, if learner $i$ injects into each of its shared variables $s_{i,t}$ and $\theta_{i,t}$ noise vectors $\zeta_{i,t}$ and $\vartheta_{i,t}$ consisting of $n$ independent Laplace noises with parameters $\nu_{i,t,\zeta}$ and $\nu_{i,t,\vartheta}$, respectively, such that $\sum_{t=1}^{T}\left(\frac{\Delta_{i,t,s}}{\nu_{i,t,\zeta}}+\frac{\Delta_{i,t,\theta}}{\nu_{i,t,\vartheta}}\right)\leq \epsilon_{i}$, then learner $i$'s implementation $\mathcal{A}_i$ of Algorithm 1 is $\epsilon_{i}$-LDP from time $t=0$ to $t=T$.
\end{lemma}
\begin{proof}
The lemma can be obtained following the same line of reasoning of Lemma~2 in~\cite{rawinfer2}.
\end{proof}

For privacy analysis, we also need the following lemma:
\begin{lemma}\label{l12}
Denote $\{v_{t}\}$ as a nonnegative sequence. If there exists a sequence $\beta_{t}=\frac{\beta_{0}}{(t+1)^{q}}$ with some $\beta_{0}>0$ and $q>0$ such that $v_{t+1}\leq (1-c)v_{t}+\beta_{t}$ holds for all $c\in(0,1)$, then we always have $v_{t}\leq C_{2}\beta_{t}$ for all $t\in{\mathbb{N}}$, where the constant $C_{2}$ is given by $C_{2}=(\frac{4q}{e\ln(\frac{2}{2-c})})^{q}(\frac{v_{0}(1-c)}{\beta_{0}}+\frac{2}{c}).$
\end{lemma}
\begin{proof}
See Appendix E.
\end{proof}

\begin{assumption}\label{a5}
There exists some positive constant $c_{l}$ such that $\|\nabla l(\theta_{i,t},\xi_{i})\|_{1} \leq c_{l}$ holds for all time $t\in{\mathbb{N}}$.
\end{assumption}

Assumption~\ref{a5} is commonly used in DP design for distributed optimization and learning~\cite{DOwu,lihuaqing,Youkeyou}.

Without loss of generality, we consider adjacent datasets $\mathcal{D}_{i}$ and ${\mathcal{D}}'_{i}$ that differ in the $k$-th element, i.e., $\xi_{i,k}$ in $\mathcal{D}_{i}$ and ${\xi}'_{i,k}$ in ${\mathcal{D}}'_{i}$ are different. For the sake of clarity, the parameters learned over $\mathcal{D}_{i}$ and ${\mathcal{D}}'_{i}$ are denoted as $\theta_{i,t}$ and ${\theta}'_{i,t}$, respectively. 

\begin{theorem}\label{t4}
Under Assumptions~\ref{a1}-\ref{a5}, if each element of $\vartheta_{i,t}$ and $\zeta_{i,t}$ follows the Laplace distributions~$\text{Lap}(\nu_{i,t,\vartheta})$ and~$\text{Lap}(\nu_{i,t,\zeta})$, respectively, with $(\sigma_{i,t,\vartheta})^2\!\!=\!2(\nu_{i,t,\vartheta})^2$ and $(\sigma_{i,t,\zeta})^2\!\!=\!2(\nu_{i,t,\zeta})^2$ satisfying Assumption 4, then $\theta_{i,t}$~(resp. $F(\theta_{i,t})$ in the general convex case) in Algorithm~1 converges in mean square to the optimal solution $\theta^*$~to the optimization problem~\eqref{primal} (resp. in mean to $F(\theta^*)$). Furthermore,

1) For any finite number of iterations $T$, learner $i$'s implementation of Algorithm~1 is LDP with a cumulative privacy budget bounded by $\epsilon_{i}\!\!=\!\!\epsilon_{i,s}\!+\!\epsilon_{i,\theta}$, where $\epsilon_{i,s}\!\!\leq\!\!\sum_{t=1}^{T}\frac{\sqrt{2}\varrho_{t,s}(t+1)^{\varsigma_{i,\zeta}}}{\sigma_{i,0,\zeta}}$ and $\epsilon_{i,\theta}\!\leq\!\sum_{t=1}^{T}\frac{\sqrt{2}\varrho_{t,\theta}(t+1)^{\varsigma_{i,\vartheta}}}{\sigma_{i,0,\vartheta}}$
with  $\varrho_{t,s}\!=\!2c_{l}\sum_{p=1}^{t}(1-\min_{i\in[m]}\{|C_{ii}|\})^{t-p}\lambda_{p-1}$, $\varrho_{0,s}=0$, and $\varrho_{t,\theta}\!=\!\sum_{q=1}^{t}(1-\min_{i\in[m]}\{|R_{ii}|\})^{t-q}(c_{z}\gamma_{z}^{q-1}+\frac{1}{|u_{i}|})(\varrho_{q,s}+\varrho_{q-1,s})$.

2) The cumulative privacy budget is finite even when the number of iterations $T$ tends to infinity.
\end{theorem}
\begin{proof}
The convergence result follows naturally from Theorem~\ref{t2} (resp. Theorem~\ref{t3}).

1) To prove the statements on privacy, we first analyze the sensitivity of learner $i$'s implementation under Algorithm 1. 
	
According to the definition of sensitivity in~\eqref{sensitive}, we have $s_{j,t}+\zeta_{j,t}={s}'_{j,t}+{\zeta}'_{j,t}$ and $\theta_{j,t}+\vartheta_{j,t}={\theta}'_{j,t}+{\vartheta}'_{j,t}$ for all $t\geq 0$ and $j\in{\mathcal{N}_{i}}$. Since we assume that only the $k$-th data point is different between $\mathcal{D}_{i}$ and ${\mathcal{D}}'_{i}$, when $t<k$, we have $s_{i,t}={s}'_{i,t}$ and $\theta_{i,t}={\theta}'_{i,t}$. However, when $t\geq k$, since the difference in loss functions kicks in at time $k$, i.e., $l(\theta,\xi_{i,k})\neq l(\theta,{\xi}'_{i,k})$, we have $s_{i,t}\neq{s}'_{i,t}$ and $\theta_{i,t}\neq{\theta}'_{i,t}$. Hence, for learner $i$'s implementation of Algorithm 1, we have
\begin{equation}
	\vspace{-0.3em}
\begin{aligned}
&\|s_{i,t+1}-{s}'_{i,t+1}\|_{1}=\big{\|}(1+C_{ii})(s_{i,t}-{s}'_{i,t})\\
&\quad+\frac{\lambda_{t}}{t+1}\sum_{p=k}^{t}(\nabla l(\theta_{i,t},\xi_{i,p})-\nabla l({\theta}'_{i,t},{\xi}'_{i,p}))\big{\|}_{1},\label{4T1}
\end{aligned}
\vspace{-0.3em}
\end{equation}
for all $t\geq k$. Letting $c_{C}\!=\!\min\{|C_{ii}|\},~i\!\in\![m]$, the sensitivity $\Delta_{i,t,s}$ satisfies
\begin{equation}
\begin{aligned}
&\Delta_{i,t+1,s}\leq (1-c_{C})\Delta_{i,t,s}\\
&\quad+\frac{\lambda_{t}}{t+1}\sum_{p=k}^{t}\big\|\nabla l(\theta_{i,t},\xi_{i,p})-\nabla l({\theta}'_{i,t},{\xi}'_{i,p})\big\|_{1}\\
&\leq (1-c_{C})\Delta_{i,t,s}\\
&\quad+\frac{\lambda_{t}}{t+1}\sum_{p=0}^{t}\big\|\nabla l(\theta_{i,t},\xi_{i,p})-\nabla l({\theta}'_{i,t},{\xi}'_{i,p})\big\|_{1},
\label{4T2}
\end{aligned}
\end{equation}
where we used $\sum_{p=0}^{k-1}\nabla l(\theta_{i,t},\xi_{i,p})=\sum_{p=0}^{k-1}\nabla l({\theta}'_{i,t},{\xi}'_{i,p})$ in the second inequality.

By using Assumption~\ref{a5} and the relation $\Delta_{i,0,s}=0$, we iterate~\eqref{4T2} from $t=1$ to $t=T$ to obtain
\begin{equation}
	\vspace{-0.3em}
\Delta_{i,t,s}
\leq 2c_{l}\sum_{p=1}^{t}(1-c_{C})^{t-p}\lambda_{p-1},\label{4T3}
\vspace{-0.3em}
\end{equation}

Similarly, we use dynamics~\eqref{dynamictheta} to obtain
\begin{equation}
\begin{aligned}
&\|\theta_{i,t+1}-{\theta}'_{i,t+1}\|_{1}=\big{\|}(1+R_{ii})(\theta_{i,t}-{\theta}'_{i,t})\\
&\quad-\frac{1}{m[z_{i,t}]_{i}}(s_{i,t+1}-{s}'_{i,t+1})+\frac{1}{m[z_{i,t}]_{i}}(s_{i,t}-{s}'_{i,t})\big{\|}_{1},\nonumber
\end{aligned}
\end{equation}
for all $t\geq k$. Letting $c_{R}=\min\{|R_{ii}|\},~i\in[m]$ and using~Lemma~\ref{l3},  the sensitivity $\Delta_{i,t,\theta}$ satisfies
\begin{equation}
\begin{aligned}
\Delta_{i,t+1,\theta}&\leq (1-c_{R})\Delta_{i,t,\theta}+c_{z}\gamma_{z}^{t}\Delta_{i,t+1,s}+c_{z}\gamma_{z}^{t}\Delta_{i,t,s}\\
&\quad+\frac{1}{|u_{i}|}\Delta_{i,t+1,s}+\frac{1}{|u_{i}|}\Delta_{i,t,s},
\label{4T5}
\end{aligned}
\end{equation}
for all $t\geq k$. Since $\theta_{i,t}={\theta}'_{i,t}$ and $s_{i,t}=s'_{i,t}$ are valid for all $t<k$,~we have \eqref{4T5} for all $t>0$. By using the relation $\Delta_{i,0,\theta}=0$ and iterating~\eqref{4T5} from $t=1$ to $t=T$, we obtain
\begin{equation}
\Delta_{i,t,\theta}
\leq\sum_{q=1}^{t}(1-c_{R})^{t-q}(c_{z}\gamma_{z}^{q-1}+\frac{1}{|u_{i}|})(\Delta_{i,q,s}+\Delta_{i,q-1,s}).\label{4T6}
\end{equation}

The inequalities~\eqref{4T3} and~\eqref{4T6} imply that for learner $i$, the $T$-iteration cumulative privacy budgets $\epsilon_{i,s}$ and $\epsilon_{i,\theta}$ are bounded by $\sum_{t=1}^{T}\frac{\sqrt{2}\varrho_{t,s}(t+1)^{\varsigma_{i,\zeta}}}{\sigma_{i,0,\zeta}}$ and $\sum_{t=1}^{T}\frac{\sqrt{2}\varrho_{t,\theta}(t+1)^{\varsigma_{i,\vartheta}}}{\sigma_{i,0,\vartheta}}$, respectively, with $\varrho_{t,s}$ and $\varrho_{t,\theta}$ given in the theorem statement.

2) By leveraging inequality~\eqref{4T2} and the relation $\xi_{i,p}={\xi}'_{i,p}$ for $p\neq k$, we have 
\begin{flalign}
&\Delta_{i,t+1,s}\leq (1-c_{C})\Delta_{i,t,s}\nonumber\\
&\quad+\frac{\lambda_{t}}{t+1}\|\nabla l(\theta_{i,t},\xi_{i,k})-\nabla l({\theta}'_{i,t},{\xi}'_{i,k})\|_{1}\nonumber\\
&\quad+\frac{\lambda_{t}}{t+1}\sum_{p=0,p\neq k}^{t}\|\nabla l(\theta_{i,t},\xi_{i,p})-\nabla l({\theta}'_{i,t},\xi_{i,p})\|_{1},\label{4T7}
\end{flalign}
for all $t\geq k$. The Lipschitz property in Assumption~\ref{a3}(iii) implies that for the same data $\xi_{i,p}$, we can rewrite~\eqref{4T7} as follows:
\begin{equation}
\Delta_{i,t+1,s}\leq (1-c_{C})\Delta_{i,t,s}+\frac{\sqrt{n}Lt\lambda_{t}}{t+1}\Delta_{i,t,\theta}+\frac{2c_{l}\lambda_{t}}{t+1}.\label{4T8}
\end{equation}
where in the derivation we have used Assumption~\ref{a5}.

By substituting~\eqref{4T8} into~\eqref{4T5}, we have
\begin{flalign}
&\Delta_{i,t+1,\theta}\leq \left(1-c_{R}+\frac{\sqrt{n}Lc_{z}(t\gamma_{z}^{t}\lambda_{t})}{t+1}\right)\Delta_{i,t,\theta}+\frac{1}{|u_{i}|}\Delta_{i,t+1,s}\nonumber\\
&\quad+(2-c_{C})c_{z}\gamma_{z}^{t}\Delta_{i,t,s}+\frac{2c_{l}c_{z}\gamma_{z}^{t}\lambda_{t}}{t+1}+\frac{1}{|u_{i}|}\Delta_{i,t,s}.\label{4T9}
\end{flalign}

Select positive constants $C_{3}<\min\{\frac{c_{R}}{2},\frac{c_{C}}{2}\}$ and $C_{0}\geq\min\{\frac{4}{|u_{i}|(c_{C}-2C_{3})},\frac{1}{|u_{i}|}\}$. We multiply both sides of~\eqref{4T8} by $C_{0}$ and combine~\eqref{4T8} and~\eqref{4T9} to obtain
\begin{equation}
	\begin{aligned}
		&\Delta_{i,t+1,\theta}+\Big(C_{0}-\frac{1}{|u_{i}|}\Big)\Delta_{i,t+1,s}\\
		&\leq \left(1-c_{R}+\frac{\sqrt{n}Lc_{z}(t\gamma_{z}^{t}\lambda_{t})}{t+1}+\frac{C_{0}\sqrt{n}Lt\lambda_{t}}{t+1}\right)\Delta_{i,t,\theta}\\
		&\quad+\left((2-c_{C})c_{z}\gamma_{z}^{t}+C_{0}(1-c_{C})+\frac{1}{|u_{i}|}\right)\Delta_{i,t,s}\\
		&\quad+\frac{2c_{l}c_{z}\gamma_{z}^{t}\lambda_{t}+2C_{0}c_{l}\lambda_{t}}{t+1}.
	\end{aligned}
\end{equation}
Since $\lambda_{t}$ and $\gamma_{z}^{t}$ are decaying sequences, there must exist some $T_{0}\geq 0$ such that  $\frac{c_{R}}{2}\geq\frac{\sqrt{n}Lc_{z}(t\gamma_{z}^{t}\lambda_{t})+C_{0}\sqrt{n}Lt\lambda_{t}}{t+1}$ and $(2-c_{C})c_{z}\gamma_{z}^{t}\leq\frac{C_{0}c_{C}}{2}$ hold for all $t\geq T_{0}$. Hence, we arrive at
\begin{equation}
	\begin{aligned}
		&\Delta_{i,t+1,\theta}+\Big(C_{0}-\frac{1}{|u_{i}|}\Big)\Delta_{i,t+1,s}\\
		&\leq \left(1-\frac{c_{R}}{2}\right)\Delta_{i,t,\theta}+\left(C_{0}\left(1-\frac{c_{C}}{2}\right)+\frac{1}{|u_{i}|}\right)\Delta_{i,t,s}\\
		&\quad+\frac{2c_{l}c_{z}\gamma_{z}^{t}\lambda_{t}+2C_{0}c_{l}\lambda_{t}}{t+1}\\
		&\leq (1-C_{3})\left(\Delta_{i,t,\theta}+\left(C_{0}-\frac{1}{|u_{i}|}\right)\Delta_{i,t,s}\right)\\
		&\quad+\frac{2c_{l}c_{z}\gamma_{z}^{t}\lambda_{t}+2C_{0}c_{l}\lambda_{t}}{t+1},\label{4333}
	\end{aligned}
\end{equation}
for all $t\geq T_{0}$, where we used  $(1-C_{3})(C_{0}-\frac{1}{|u_{i}|})>C_{0}(1-\frac{c_{C}}{2})+\frac{1}{|u_{i}|}$ according to the definitions of $C_{0}$ and $C_{3}$.

We further define a constant $C_{4}>0$ as follows:
\begin{equation}
	\small
	\begin{aligned}	
		C_{4}&=\max\Bigg\{\left(\frac{4(1+v)}{e\ln(\frac{2}{2-C_{3}})}\right)^{1+v}\frac{2}{1-C_{3}},\\
		&\quad\max_{0\leq t\leq T_{0},i\in[m]}\Bigg\{\frac{\left(\Delta_{i,t,\theta}+\left(C_{0}-\frac{1}{|u_{i}|}\right)\Delta_{i,t,s}\right)(t+1)}{2c_{l}c_{z}\gamma_{z}^{t}\lambda_{t}+2C_{0}c_{l}\lambda_{t}}\Bigg\}\Bigg\}.\nonumber
	\end{aligned}
\end{equation}
Combining Lemma~\ref{l12} and~\eqref{4333}, we obtain
\begin{equation}
	\Delta_{i,t+1,\theta}+(C_{0}-\frac{1}{|u_{i}|})\Delta_{i,t+1,s}\leq C_{4}\frac{2c_{l}c_{z}\gamma_{z}^{t}\lambda_{0}+2C_{0}c_{l}\lambda_{0}}{(t+1)^{1+v}},\nonumber
\end{equation}
for all $t>0$. By using Lemma~\ref{l11}, we arrive at
\begin{equation}
	\epsilon_{i,\theta}\leq \sum_{t=1}^{T}\frac{2\sqrt{2}C_{4}c_{l}\lambda_{0}(c_{z}\gamma_{z}^{t}+C_{0})}{\sigma_{i,0,\vartheta}(t+1)^{1+v-\varsigma_{i,\vartheta}}},\nonumber
\end{equation}
and
\begin{equation}
	\epsilon_{i,s}\leq \sum_{t=1}^{T}\frac{2\sqrt{2}C_{4}c_{l}\lambda_{0}(c_{z}\gamma_{z}^{t}+C_{0})}{(C_{0}-\frac{1}{|u_{i}|})\sigma_{i,0,\zeta}(t+1)^{1+v-\varsigma_{i,\zeta}}},\nonumber
\end{equation}
implying that $\epsilon_{i}=\epsilon_{i,\theta}+\epsilon_{i,s}$ is finite even when $T$ tends to infinity since $v>\max\{\varsigma_{i,\vartheta},\varsigma_{i,\zeta}\}$ always holds.
\end{proof}
Theorem~\ref{t4} proves that the privacy budget $\epsilon_{i}$ is finite even when the number of iterations $T$ tends to infinity, thereby establishing rigorous privacy protection in the infinite time horizon. We have thus shown that Algorithm~1 can simultaneously ensure accurate learning and rigorous $\epsilon_{i}$-LDP for each learner. This is fundamentally different from existing DP solutions for distributed learning and optimization~\cite{DOwu,DMOD,DOLA,lihuaqing,Youkeyou,distributedonlineDP3}, which allow the cumulative privacy budget to grow to infinity, implying diminishing privacy protection as the number of iterations tends to infinity.

\begin{remark}
Compared with the privacy analysis in our prior work~\cite{ziji1} for undirected graphs, which only involves a single optimization variable, the privacy analysis here is much more complicated due to the involvement of two optimization variables $s_{i,t}$ and $\theta_{i,t}$, whose dynamics are strongly coupled.
\end{remark}

\section{Numerical Experiments}\label{experiment}
We evaluated the performance of Algorithm 1 through three machine-learning applications: linear regression using the ``Mushrooms" dataset and image classification using the ``MNIST" and ``CIFAR-10" datasets, respectively. In each experiment, we compared Algorithm 1 with existing DP solutions for distributed learning and optimization, including the DiaDSP algorithm~\cite{Dingtie}, the DP-oriented gradient tracking based algorithm~\cite{Tailoring}, the distributed online stochastic subgradient algorithm~\cite{DOwu}, the distributed online optimization algorithm~\cite{Youkeyou}, and the LDOL algorithm~\cite{ziji1}. For a fair comparison, we set the privacy budget for these algorithms as the maximum $\epsilon_{i}$ across all learners used in our Algorithm 1, which corresponds to the weakest level of privacy protection among all learners. Additionally, we evaluated the conventional Push-Pull gradient-tracking algorithm~\cite{pushpull} (i.e., algorithm~\eqref{classGTnoise}) under the same DP noises as those used in Algorithm 1. The interaction pattern associated with the weight matrix $R$ was consistent across all experiments and is depicted in Fig.~1. The weight matrix $C$ was set as the transpose of $R$.
\begin{figure}
	\centering 
	\includegraphics[width=0.28\textwidth]{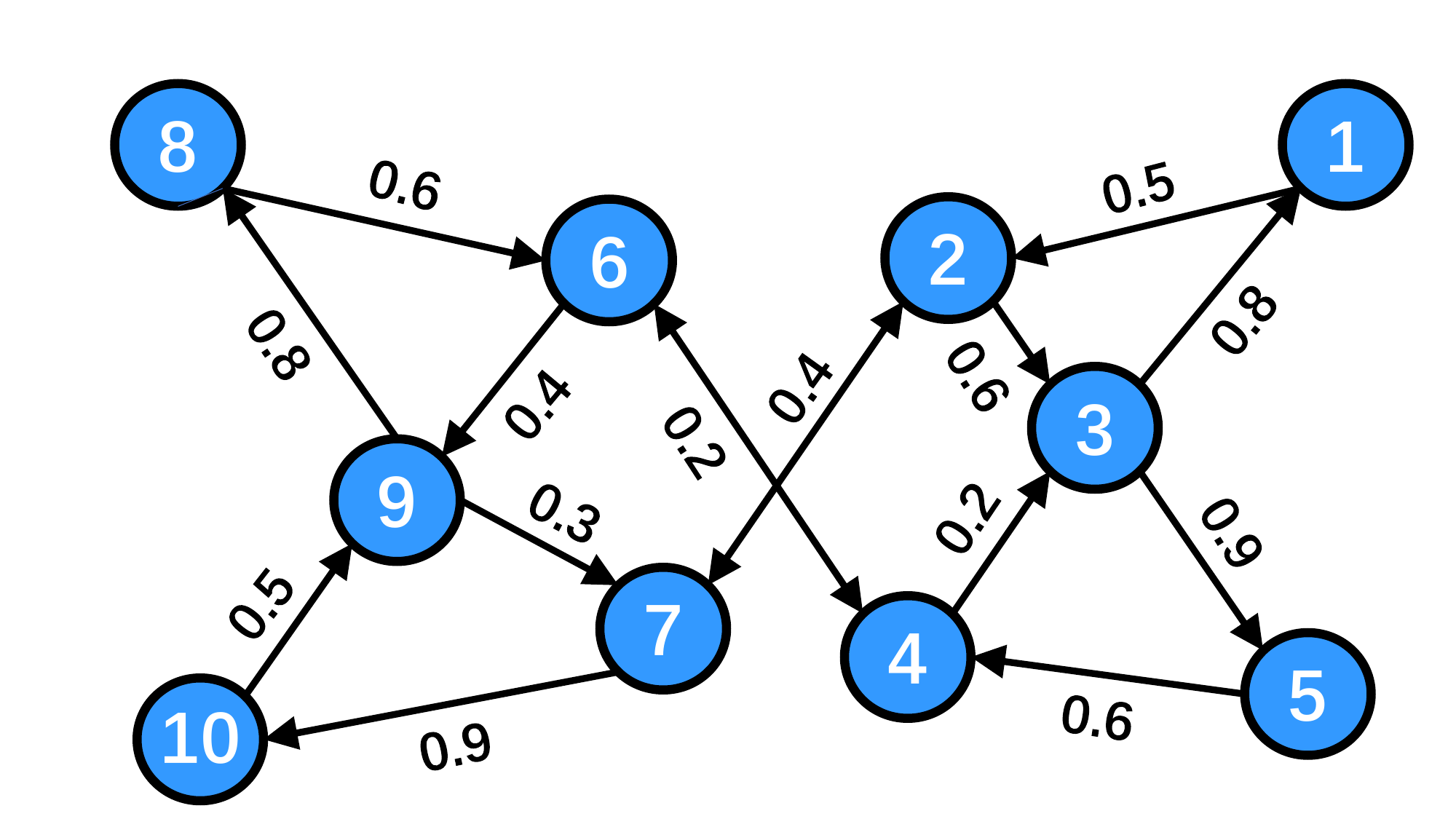} 
	\caption{The interaction graph $\mathcal{G}_{R}$ of ten learners.} 
	\label{Fig1} 
\end{figure}
\subsection{Logistic regression using the ``Mushrooms" dataset}
We first evaluated the performance of Algorithm 1 using $l_{2}$-logistic regression based classification of the ``Mushrooms" dataset~\cite{mushrooms}. The loss function for learner $i$ is given by 
\begin{equation} l(\theta,\xi_{i,t})=\frac{1}{N_{i,t}}\sum_{s=1}^{N_{i,t}}(1-b_{i,s}a_{i,s}^{T}\theta-\log(s((a_{i,s})^{T}\theta))+\frac{r_{i,t}}{2}\|\theta\|_{2}^2,\nonumber
\end{equation}
where~$N_{i,t}$ is the number of samples at time $t$, $s(a)$ is the sigmoid function defined as $s(a)=\frac{1}{1+e^{-a}}$, $(a_{i,t},b_{i,t})\in{\mathbb{R}^{n}\!\times\!\mathbb{R}}$ is a data point, and $r_{i,t}\!>\!0$ is a regularization parameter proportional to~$N_{i,t}$. In each iteration, we randomly selected $10$ samples and distributed them among the $10$ learners. 

In this experiment, we set the stepsize as $\lambda_{t}=\frac{1}{(t+1)^{0.6}}$ and the DP-noise variances as $\nu_{i,t,\vartheta}=\frac{1}{(t+1)^{\varsigma_{i,\vartheta}}}$ and $\nu_{i,t,\zeta}=\frac{1}{(t+1)^{\varsigma_{i,\zeta}}}$ with $\varsigma_{i,\vartheta}=\varsigma_{i,\zeta}=0.5+0.01i$ for $i=1,2,\cdots,10$ The optimal solution $\theta^*$ was obtained using a centralized gradient descent algorithm in the absence of DP noises. In our comparison, we employed the same stepsize and DP noises for the conventional Push-Pull gradient-tracking algorithm~\cite{pushpull}. For other algorithms, we selected near-optimal stepsizes, ensuring that doubling the stepsizes would lead to non-convergent behaviors for these algorithms. In particular, the weakening factor for the LDOL algorithm was set to $\gamma_{t}=\frac{1}{(t+1)^{0.7}}$, in accordance with the guidelines provided in~\cite{ziji1}. The weakening factors for the algorithm in~\cite{Tailoring} were set to $\gamma_{1,t}=\frac{1}{(t+1)^{0.95}}$ and $\gamma_{2,t}=\frac{1}{(t+1)^{0.75}}$, in line with the guidelines provided in~\cite{Tailoring}.

Fig.~2(a) shows the evolution of the average tracking errors $\frac{1}{10}\sum_{i=1}^{10}\|\theta_{i,t}-\theta^*\|_{2}$ and Fig.~2(b) depicts the evolution of the average objective function loss $\frac{1}{10}\sum_{i=1}^{10}F(\theta_{i,t})-F(\theta^*)$ [see Theorem~\ref{t2}]. Clearly, our Algorithm 1 outperforms existing results in terms of learning accuracy. Moreover, it can be seen that the DP noise indeed accumulates in the conventional Push-Pull algorithm in~\cite{pushpull}, leading to non-convergent learning results [see Sec.~\ref{motivation}].
\begin{figure}\label{mushrooms}
	\centering
	\subfigure[``Mushrooms", Average tracking error]{\includegraphics[height=4.8cm,width=8.5cm]{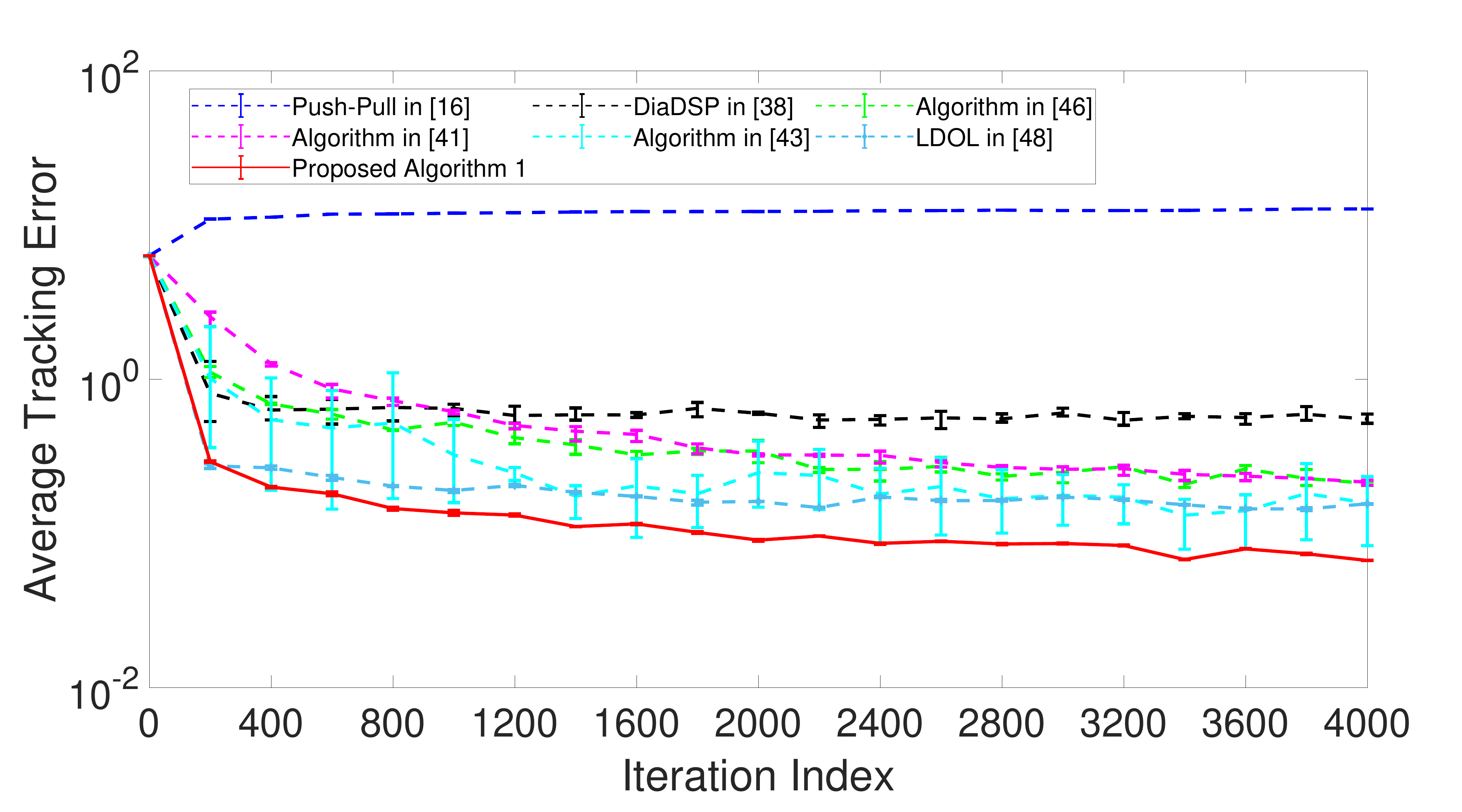}}\label{Mushroomtracking}
	\subfigure[``Mushrooms", Average tracking loss]{\includegraphics[height=4.8cm,width=8.5cm]{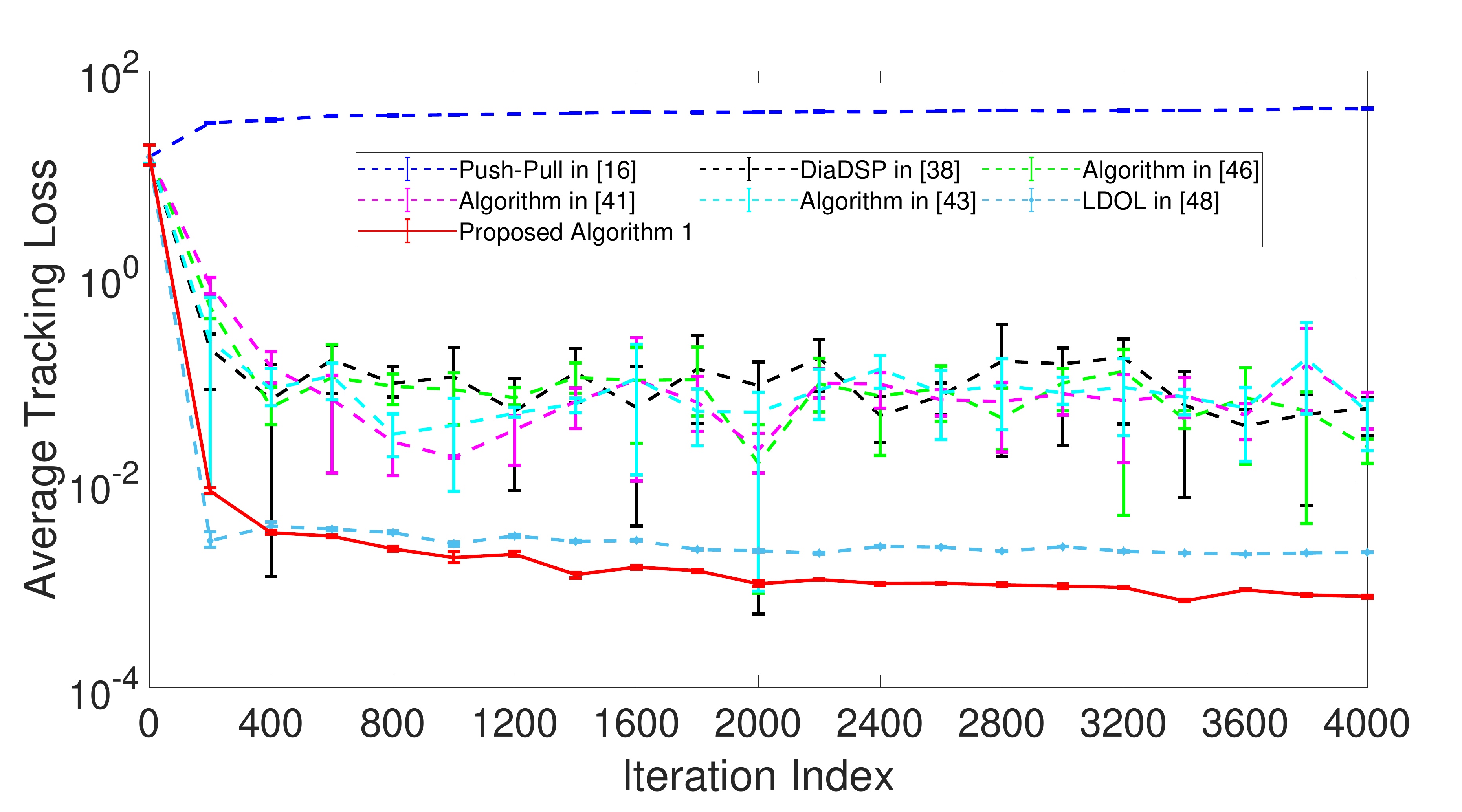}}\label{Mushroomcost}
	\caption{Comparison of online logistic regression results by using the ``Mushrooms" dataset.}
\end{figure}
\subsection{Neural-network training using the ``MNIST" dataset}
In the second experiment, we evaluated Algorithm 1 by training a convolutional neural-network (CNN) ResNet-18 on the ``MNIST" dataset~\cite{mnist}. During each iteration, each learner was trained on $40$ randomly selected images. 

In this experiment, we chose the stepsize as $\lambda_{t}=\frac{0.6}{(t+1)^{0.6}}$ and the DP-noise variances as $\nu_{i,t,\vartheta}=\frac{0.01}{(t+1)^{\varsigma_{i,\vartheta}}}$ and $\nu_{i,t,\zeta}=\frac{0.01}{(t+1)^{\varsigma_{i,\zeta}}}$ with $\varsigma_{i,\vartheta}=\varsigma_{i,\zeta}=0.5+0.01i$ for $i=1,2,\cdots,10$. We used the best stepsizes that we could find for the existing algorithms used in the comparison. The weakening factors for the algorithms in~\cite{ziji1} and~\cite{Tailoring} remained consistent with those employed in the previous logistic regression experiment.

Fig.~3 shows that the conventional Push-Pull algorithm~\cite{pushpull}, the DiaDSP algorithm~\cite{Dingtie}, and the algorithm in~\cite{DOwu} are incapable of effectively training the CNN model under persistent DP-noise injections. Moreover, our Algorithm 1 has better training and testing accuracies than the DP-oriented gradient tracking based algorithm~\cite{Tailoring}, the distributed online optimization algorithm~\cite{Youkeyou}, and the LDOL algorithm~\cite{ziji1}.
\begin{figure}\label{mnist}
	\centering
	\subfigure[``MNIST", Training accuracy]{\includegraphics[height=4.8cm,width=8.5cm]{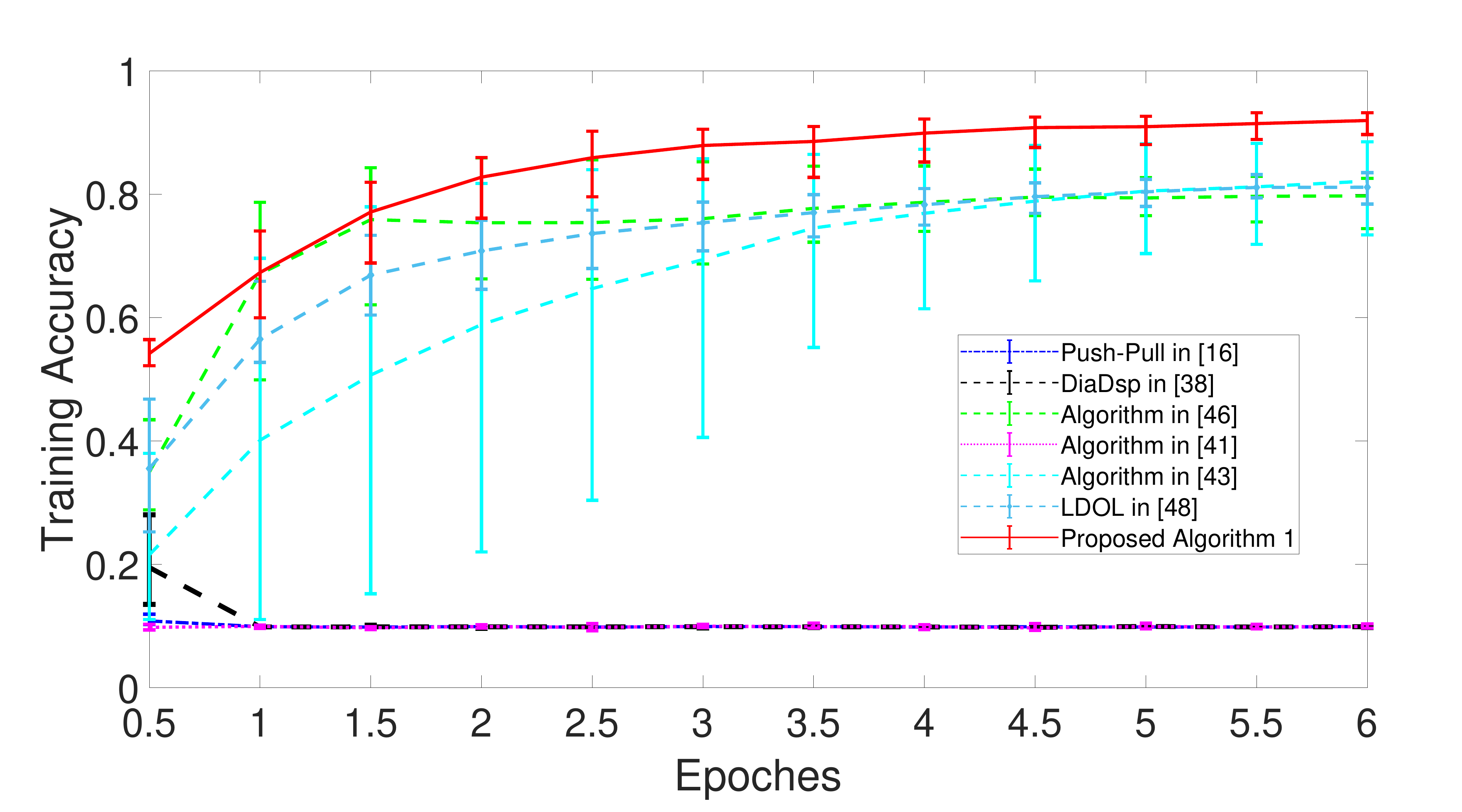}}\label{mnisttrain}
	\subfigure[``MNIST", Testing accuracy]{\includegraphics[height=4.8cm,width=8.5cm]{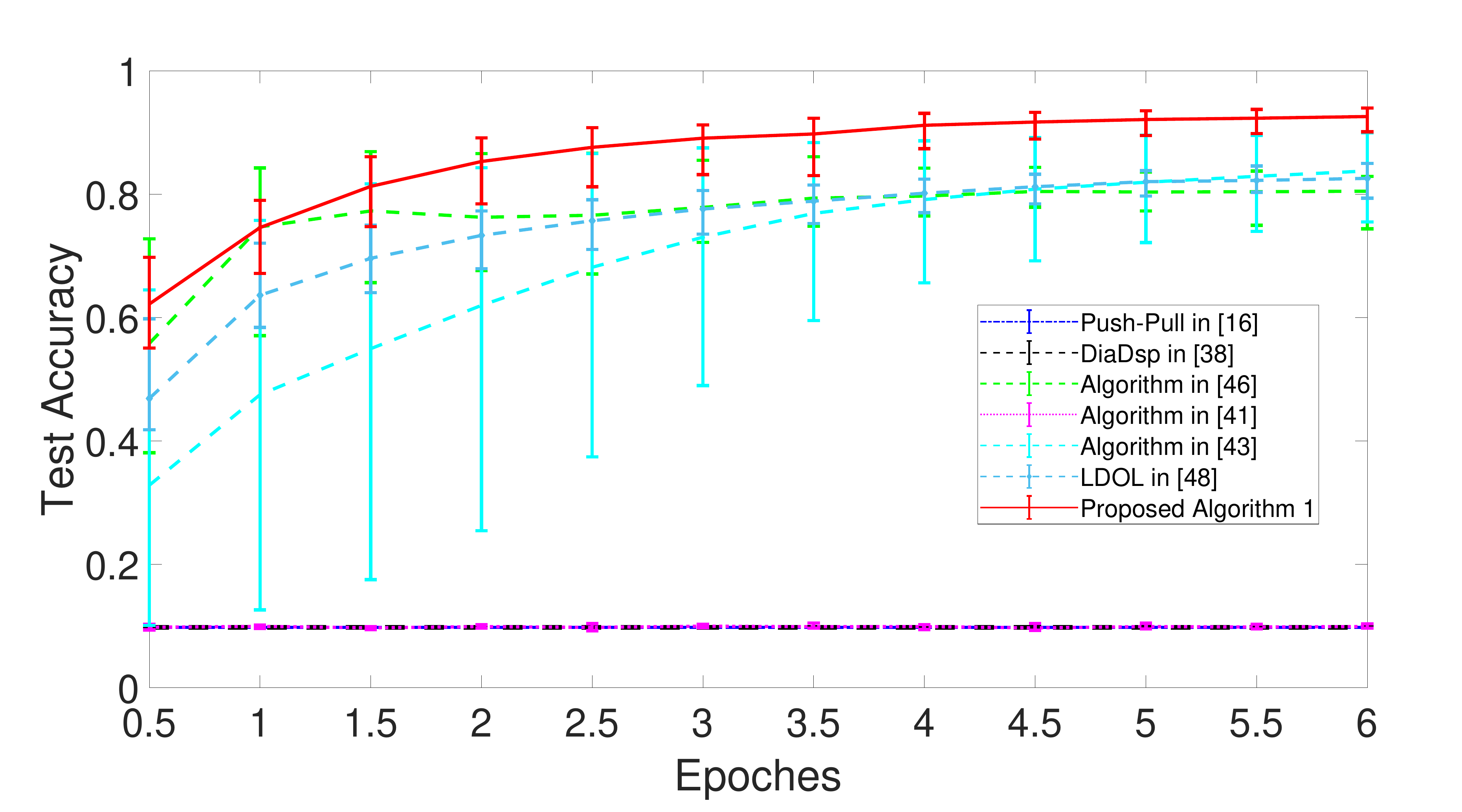}}\label{mnisttest}
	\caption{Comparison of CNN classification results by using the ``MNIST" dataset.}
\end{figure}
\subsection{Neural-network training using the ``CIFAR-10" dataset}
The third experiment evaluated Algorithm~1 using a CNN model and the ``CIFAR-10" dataset~\cite{cifar10}, which provides a greater diversity and complexity than the ``MNIST" dataset. The CNN architecture and parameters were the same as those used in the previous experiment on the ``MNIST" dataset.

The results are summarized in Fig. 4, which once again confirms the advantage of our proposed algorithm over existing counterparts in terms of both learning and testing accuracies.
\begin{figure}\label{cifar10}
	\centering
	\subfigure[``CIFAR-10", Training Accuracy]{\includegraphics[height=4.8cm,width=8.5cm]{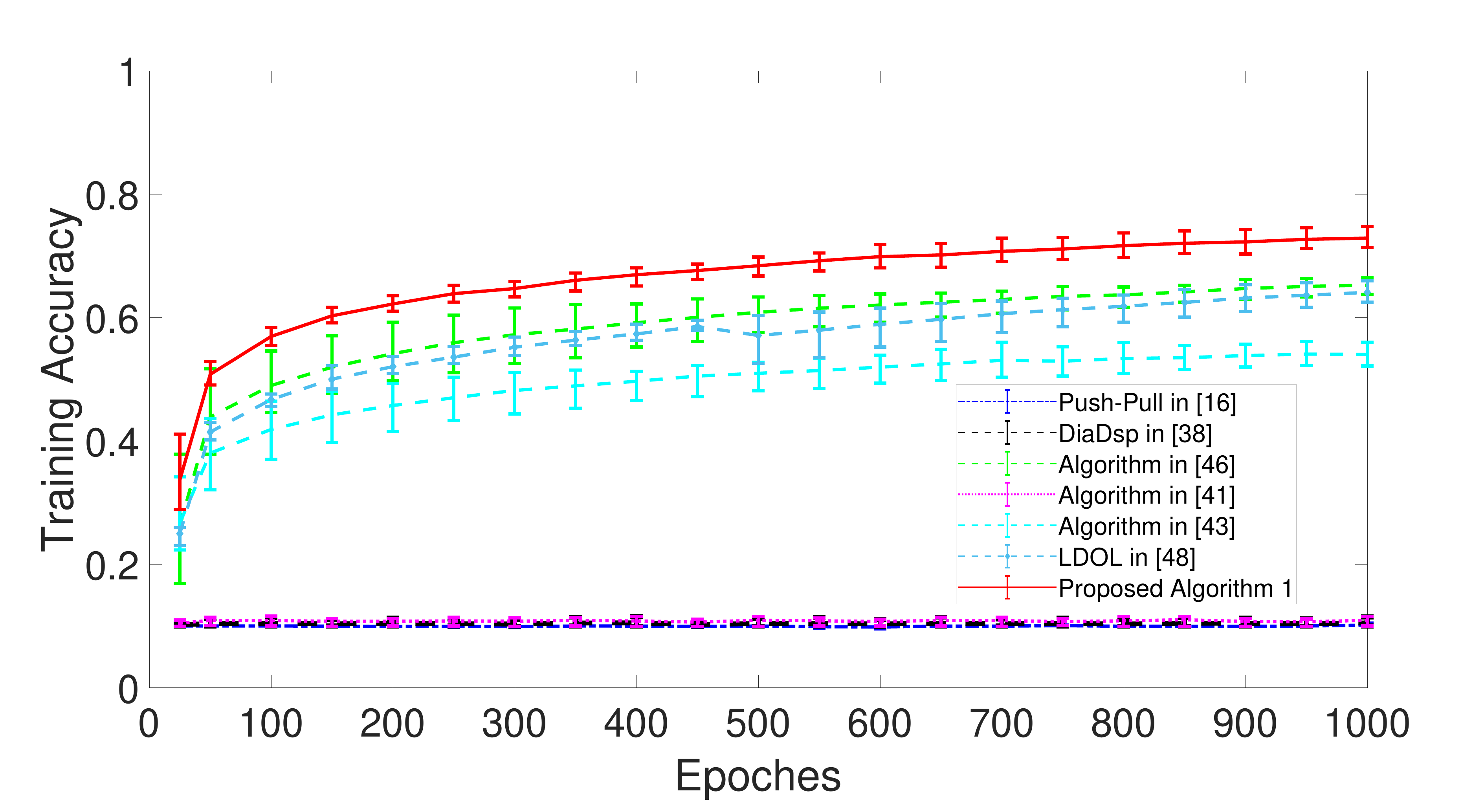}}\label{cifartrain}
	\subfigure[``CIFAR-10", Testing Accuracy]{\includegraphics[height=4.8cm,width=8.5cm]{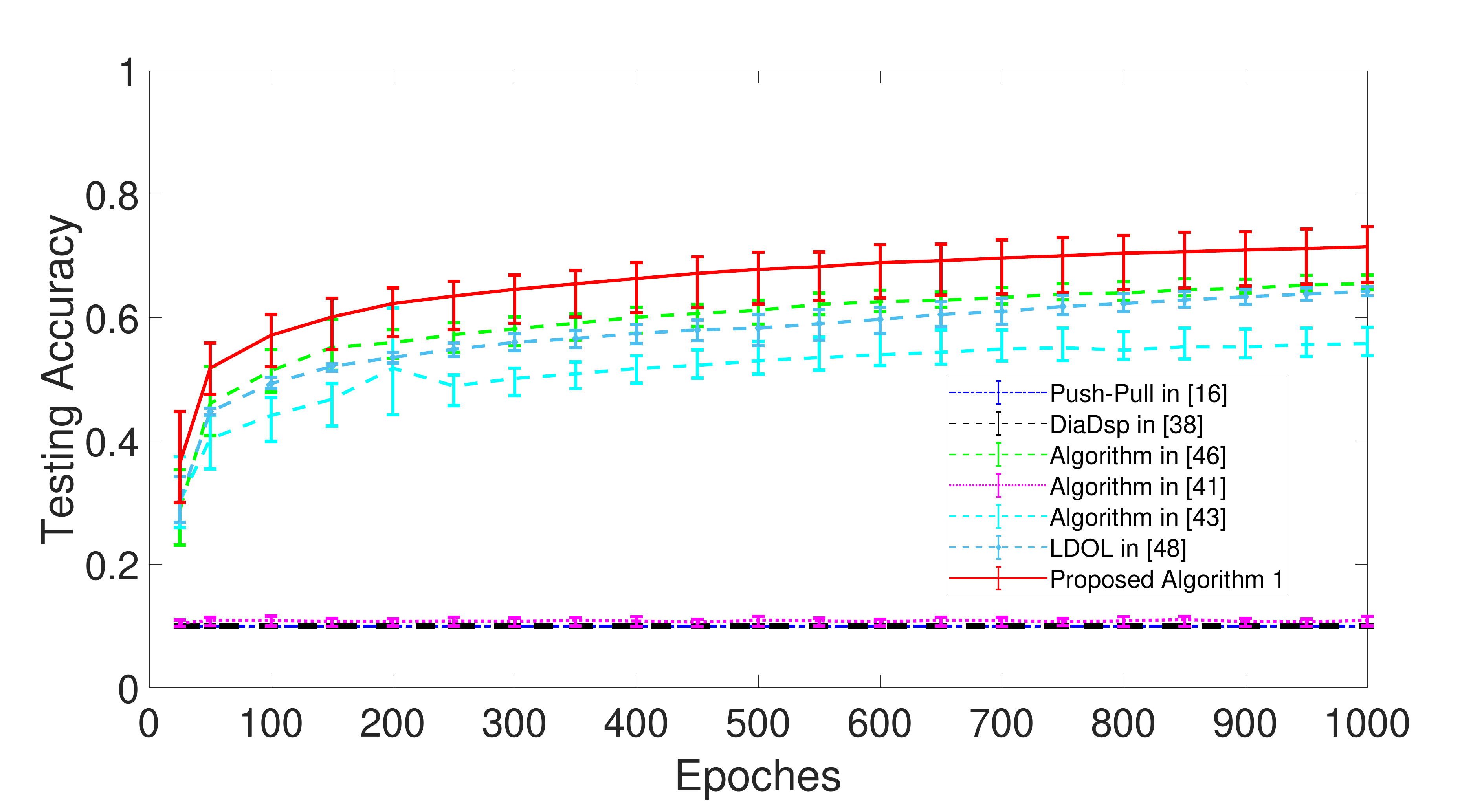}}\label{cifartest}
	\caption{Comparison of CNN classification results by using the ``CIFAR-10" dataset.}
\end{figure}
\section{Conclusions}\label{conclusion}
In this study, we proposed a distributed learning algorithm under the constraints of differential privacy and sequentially arriving data samples. We proved that the proposed algorithm converges in mean square to the accurate optimal solution, even in the presence of persistent DP noises and general directed graphs. Simultaneously, we also proved that the proposed algorithm can ensure rigorous $\epsilon_{i}$-LDP with a finite cumulative privacy budget, even when the number of iterations grows to infinity. To the best of our knowledge, this is the first algorithm that is able to simultaneously achieve provable convergence and rigorous $\epsilon_{i}$-LDP (with a finite cumulative privacy budget) in distributed online learning over directed graphs. Experimental comparisons using multiple benchmark machine-learning applications confirm the advantage of our proposed algorithm over existing counterparts.


\section*{Appendix}
For the convenience of derivation, we define  $\bar{\zeta}_{C,t}=\frac{\mathbf{1}^{T}\boldsymbol{\zeta}_{C,t}}{m}$, $\bar{\vartheta}_{R,t}=\frac{u^{T}\boldsymbol{\vartheta}_{R,t}}{m}$,
$\nabla\bar{f}_{t}(\boldsymbol{\theta}_{t})=\frac{\mathbf{1}^{T}\nabla\boldsymbol{f}_{t}(\boldsymbol{\theta}_{t})}{m}$, $\nabla\bar{f}(\boldsymbol{\theta}_{t})=\frac{\mathbf{1}^{T}\nabla\boldsymbol{f}(\boldsymbol{\theta}_{t})}{m}$, $U=\text{diag}(u_{1},\cdots,u_{m})$, $\bar{\mathbf{C}}=\mathbf{C}-\frac{\omega\mathbf{1}^{T}}{m}$, and $\bar{\mathbf{R}}=\mathbf{R}-\frac{\mathbf{1}u^{T}}{m}$.
\subsection{Proof of Lemma~\ref{l7}}
Left multiplying both sides of~\eqref{dcompacts} by $\frac{1}{m}\mathbf{1}^{T}$ and using the relation $\mathbf{1}^{T}C=\mathbf{0}$, we obtain
\begin{equation}
\begin{aligned}
\bar{s}_{t+1}=\frac{\mathbf{1}^{T}}{m}\big(\boldsymbol{s}_{t}+\boldsymbol{\zeta}_{C,t}+\lambda_{t}\nabla\boldsymbol{f}_{t}(\boldsymbol{\theta}_{t})\big).\label{1T1S1}
\end{aligned}
\end{equation}
Combing~\eqref{1T1S1} with~\eqref{dcompacts} and $\bar{\mathbf{C}}\omega=\mathbf{0}$ leads to
\begin{equation}
\boldsymbol{s}_{t+1}-\omega\bar{s}_{t+1}=\bar{\mathbf{C}}(\boldsymbol{s}_{t}-\omega\bar{s}_{t})+\Pi_{\omega}\boldsymbol{\zeta}_{C,t}+\lambda_{t}\Pi_{\omega}\nabla\boldsymbol{f}_{t}(\boldsymbol{\theta}_{t}),\label{1T1S2}
\end{equation}
where we have used the definition $\Pi_{\omega}=I-\frac{\omega\mathbf{1}^{T}}{m}$. Taking the norm $\|\cdot\|_{C}$ on both sides of~\eqref{1T1S2} yields
\begin{equation}
	\begin{aligned}
&\|\boldsymbol{s}_{t+1}-\omega\bar{s}_{t+1}\|_{C}^2\\
&\leq \left(\|\bar{\mathbf{C}}\|_{C}\|\boldsymbol{s}_{t}-\omega\bar{s}_{t}\|_{C}+\lambda_{t}\|\Pi_{\omega}\|_{C}\|\nabla\boldsymbol{f}_{t}(\boldsymbol{\theta}_{t})\|_{C}\right)^2\\
&\quad+\|\Pi_{\omega}\boldsymbol{\zeta}_{C,t}\|_{C}^2+2\left\langle\bar{\mathbf{C}}(\boldsymbol{s}_{t}-\omega\bar{s}_{t})\right.\\
&\left.\quad+\lambda_{t}\Pi_{\omega}\nabla\boldsymbol{f}_{t}(\boldsymbol{\theta}_{t}),\Pi_{\omega}\boldsymbol{\zeta}_{C,t}\right\rangle_{C},\label{1T1S3}
\end{aligned}
\end{equation}
where $\langle\cdot,\cdot\rangle_{C}$ denotes the inner product induced~\footnote{Since one can verify $\|\boldsymbol{s}_{t}\|_{C}=\|\tilde{C}\boldsymbol{s}_{t}\|_{2}$ where $\tilde{C}$ is discussed in Sec.~\ref{supportlemma}, we have the norm $\|\cdot\|_{C}$ satisfying the Parallelogram law, implying that it has an associated inner product $\langle\cdot,\cdot\rangle_{C}$. The detailed discussions can be found in~\cite{Tailoring,spectral}.} by the norm $\|\cdot\|_{C}$.

By using the definition $\|\bar{\mathbf{C}}\|_{C}=1-\rho_{C}<1$ and the inequality $(a+b)^2\leq (1+\epsilon)a^2+(1+\epsilon^{-1})b^2$ for any scalars $a$, $b$, and $\epsilon>0$ (setting $\epsilon=\frac{1}{1-\rho_{C}}-1$, implying $1+\epsilon^{-1}=\frac{1}{\rho_{C}}$), we obtain the following relationship from~\eqref{1T1S3}:
\begin{equation}
	\begin{aligned}
&\|\boldsymbol{s}_{t+1}-\omega\bar{s}_{t+1}\|_{C}^2\\
&\leq (1-\rho_{C})\|\boldsymbol{s}_{t}-\omega\bar{s}_{t}\|_{C}^2+\frac{\lambda_{t}^2}{\rho_{C}}\|\Pi_{\omega}\|_{C}^2\|\nabla\boldsymbol{f}_{t}(\boldsymbol{\theta}_{t})\|_{C}^2\\
&\quad+\|\Pi_{\omega}\boldsymbol{\zeta}_{C,t}\|_{C}^2+2\left\langle\bar{\mathbf{C}}(\boldsymbol{s}_{t}-\omega\bar{s}_{t})\right.\\
&\left.\quad+\lambda_{t}\Pi_{\omega}\nabla\boldsymbol{f}_{t}(\boldsymbol{\theta}_{t}),\Pi_{\omega}\boldsymbol{\zeta}_{C,t}\right\rangle_{C}.\label{1T1S4}
\end{aligned}
\end{equation}

Given $\mathbb{E}[\boldsymbol{\zeta}_{C,t}]=0$ according to Assumption~\ref{a4}, we take the expectation on both sides of~\eqref{1T1S4} to obtain
\begin{equation}
\mathbb{E}\big[\|\boldsymbol{s}_{t+1}-\omega\bar{s}_{t+1}\|_{C}^2\big]\leq (1-\rho_{C})\mathbb{E}\big[\|\boldsymbol{s}_{t}-\omega\bar{s}_{t}\|_{C}^2\big]+\Phi_{s,t},\label{1T1S5}
\end{equation}
where the term $\Phi_{s,t}$ is given by
\begin{equation} \Phi_{s,t}=\frac{\lambda_{t}^2}{\rho_{C}}\|\Pi_{\omega}\|_{C}^2\mathbb{E}[\|\nabla\boldsymbol{f}_{t}(\boldsymbol{\theta}_{t})\|_{C}^2]+\|\Pi_{\omega}\|_{C}^2\mathbb{E}[\|\boldsymbol{\zeta}_{C,t}\|_{C}^2].\label{phis}
\end{equation} 

We proceed to characterize the $\Phi_{s,t}$ in~\eqref{1T1S5}. 

By using the definition $\nabla f_{i}(\theta_{i,t})=\mathbb{E}_{\xi_{i}\sim\mathcal{P}_{i}}[\nabla l(\theta_{i,t},\xi_{i})]$, Assumption~\ref{a2}(ii), and Assumption~\ref{a3}(ii), we have
\begin{equation}
\begin{aligned}
&\mathbb{E}\left[\|\nabla l(\theta_{i,t},\xi_{i})\|_{2}^2\right]\\
&=\mathbb{E}\left[\left\|\nabla l(\theta_{i,t},\xi_{i})-\mathbb{E}\left[\nabla l(\theta_{i,t},\xi_{i})\right]+\mathbb{E}\left[\nabla l(\theta_{i,t},\xi_{i})\right]\right\|_{2}^2\right]\\
&\leq 2\mathbb{E}\left[\|\nabla l(\theta_{i,t},\xi_{i})-\nabla f_{i}(\theta_{i,t})\|_{2}^2\right]+2\mathbb{E}\left[\|\nabla f_{i}(\theta_{i,t})\|_{2}^2\right]\\
&\leq 2(\kappa^2+D^2).\label{lgradient}
\end{aligned}
\end{equation}
Given $\mathbb{E}[\|\nabla \boldsymbol{f}_{t}(\boldsymbol{\theta}_{t})\|_{F}^2]\leq\sum_{i=1}^{m}\frac{1}{t+1}\sum_{k=0}^{t}\mathbb{E}[\|\nabla l(\theta_{i,t},\xi_{i,k})\|_{2}^2]$, we obtain
\begin{flalign}
\mathbb{E}\big[\|\nabla \boldsymbol{f}_{t}(\boldsymbol{\theta}_{t})\|_{C}^2\big]\leq 2m\delta_{C,F}^2(\kappa^2+D^2).\label{D14}
\end{flalign}

Substituting~\eqref{D14} into $\Phi_{s,t}$ in~\eqref{phis} and using $\mathbb{E}[\|\boldsymbol{\zeta}_{C,t}\|_{C}^2]\leq \delta_{C,F}^2\sum_{i,j}(C_{ij}\sigma_{j,t,\zeta})^2$, we arrive at
\begin{equation}
\begin{aligned}
\Phi_{s,t}&=\frac{2m\delta_{C,F}^2\|\Pi_{\omega}\|_{C}^2\|(\kappa^2+D^2)\lambda_{t}^2}{\rho_{C}}\\
&\quad+\delta_{C,F}^2\|\Pi_{\omega}\|_{C}^2\sum_{i,j}(C_{ij}\sigma_{j,t,\zeta})^2. \label{1T1Sdelta}
\end{aligned}
\end{equation}

Now, we iterate~\eqref{1T1S5} from $0$ to $t$ to obtain
\begin{equation}
	\begin{aligned}
&\mathbb{E}\big[\|\boldsymbol{s}_{t+1}-\omega\bar{s}_{t+1}\|_{C}^2\big]\leq (1-\rho_{C})^{t+1}\mathbb{E}\big[\|\boldsymbol{s}_{0}-\omega\bar{s}_{0}\|_{C}^2\big]\\
&\quad+\sum_{p=0}^{t-1}(1-\rho_{C})^{t-p}\Phi_{s,p}+\Phi_{s,t}.\label{1T1S6}
	\end{aligned}
\end{equation}

When $t=0$, using the definitions of $\sigma_{\zeta}^{+}$ and $\lambda_{t}$, we have
\begin{equation}
\mathbb{E}\big[\|\boldsymbol{s}_{1}-\omega\bar{s}_{1}\|_{C}^2\big]\leq (1-\rho_{C})\mathbb{E}\big[\|\boldsymbol{s}_{0}-\omega\bar{s}_{0}\|_{C}^2\big]+\tau_{s1}+\tau_{s2}.\label{1T1S8}
\end{equation}
where the constants $\tau_{s1}$ and $\tau_{s2}$ are given in the statement of Lemma~\ref{l7}.

When $t>0$, we estimate each item on the right hand side of~\eqref{1T1S6}. 

Lemma~\ref{l6} and $(t+1)^{-2}< t^{-2}$ for all $t>0$ imply 
\begin{equation}
(1-\rho_{C})^{t+1}\leq \frac{4(t+1)^{-2}}{(e\ln(1-\rho_{C}))^2}< \frac{4t^{-2}}{(e\ln(1-\rho_{C}))^2}.\label{1T1S6i}
\end{equation}

Then we estimate the second item on the right hand side of~\eqref{1T1S6}. Combining~\eqref{1T1Sdelta} with the definitions 
$\sigma_{\zeta}^{+}\!=\!\max_{i\in[m]}\{\sigma_{i,0,\zeta}\}$,
$\varsigma_{\zeta}\!=\!\min_{i\in[m]}\{\varsigma_{i,\zeta}\}$, and $\lambda_{t}=\frac{\lambda_{0}}{(t+1)^{v}}$, we have
\begin{equation}
\begin{aligned}
&\sum_{p=0}^{t-1}(1-\rho_{C})^{t-p}\Phi_{s,p}\\
&=\tau_{s1}\sum_{p=0}^{t-1}\frac{(1-\rho_{C})^{t-p}}{(p+1)^{2v}}+\tau_{s2}\sum_{p=0}^{t-1}\frac{(1-\rho_{C})^{t-p}}{(p+1)^{2\varsigma_{\zeta}}}.\label{1T1S7}
\end{aligned}
\end{equation}
where the constants $\tau_{s1}$ and $\tau_{s2}$ are given in the statement of Lemma~\ref{l7}.

We further analyze each item on the right hand side of~\eqref{1T1S7}:

1) For scalars $a,b,c,d\!>\!0$ satisfying $\frac{c}{a}\!>\!\frac{d}{b}$, the relationship $\frac{d}{b}\!<\!\frac{c+d}{a+b}\!<\!\frac{c}{a}$ always holds. This inequality implies $\frac{1}{(t-p)^2}\!<\!(\frac{p+1}{t})^2$ for all $p\in[0,t)$. Using this inequality, Lemma~\ref{l6}, and the relation $(\frac{p+1}{t})^2\!<\!(\frac{p+1}{t})^{2v}$ (where $\frac{p+1}{t}\!\in\!(0,1]$), we obtain
\begin{equation}
\begin{aligned}
\frac{(1-\frac{\rho_{C}}{2})^{t-p}}{(p+1)^{2v}}&\leq \frac{4}{(e\ln(1-\frac{\rho_{C}}{2}))^2(t-p)^{2}}\frac{1}{(p+1)^{2v}}\\
&<\frac{4}{(e\ln(1-\frac{\rho_{C}}{2}))^2}\Big(\frac{p+1}{t}\Big)^{2v}\frac{1}{(p+1)^{2v}}\\
&=\frac{4t^{-2v}}{(e\ln(1-\frac{\rho_{C}}{2}))^2}.\label{1T1S6ii}
\end{aligned}
\end{equation}
By applying inequality~\eqref{1T1S6ii}, $1-\rho_{C}\leq (1-\frac{\rho_{C}}{2})^2$, and $\sum_{p=0}^{t-1}(1-\frac{\rho_{C}}{2})^{t-p}<\frac{1-(1-\frac{\rho_{C}}{2})^{t}}{1-(1-\frac{\rho_{C}}{2})}\leq \frac{2}{\rho_{C}}$ to the first term on the right hand side of~\eqref{1T1S7}, we obtain
\begin{equation}
	\begin{aligned}
\sum_{p=0}^{t-1}\frac{(1-\rho_{C})^{t-p}}{(p+1)^{2v}}&< \left(\sum_{p=0}^{t-1}\left(1-\frac{\rho_{C}}{2}\right)^{t-p}\right)\frac{4t^{-2v}}{(e\ln(1-\frac{\rho_{C}}{2}))^2}\\
&\leq \frac{8t^{-2v}}{\rho_{C}(e\ln(1-\frac{\rho_{C}}{2}))^2}.\label{1T1S6iii}
\end{aligned}
\end{equation}

2) Using an argument similar to the derivation of~\eqref{1T1S6iii}, the second term on the right hand side of~\eqref{1T1S7} satisfies
\begin{equation}
	\vspace{-0.2em}
\sum_{p=0}^{t-1}\frac{(1-\rho_{C})^{t-p}}{(p+1)^{2\varsigma_{\zeta}}}< \frac{8t^{-2\varsigma_{\zeta}}}{\rho_{C}(e\ln(1-\frac{\rho_{C}}{2}))^2}.\label{1T1S6iv}
	\vspace{-0.2em}
\end{equation}

According to the definition of $\Phi_{s,t}$ in~\eqref{1T1Sdelta}, we obtain that the third term on the right hand side of~\eqref{1T1S6} satisfies
\begin{equation}
	\vspace{-0.2em}
\Phi_{s,t}\leq \tau_{s1}t^{-2v}+\tau_{s2}t^{-2\varsigma_{\zeta}}.\label{third}
	\vspace{-0.2em}
\end{equation}

By substituting~\eqref{1T1S6iii} and~\eqref{1T1S6iv} into~\eqref{1T1S7}, and then substituting~\eqref{1T1S8}-\eqref{1T1S7} and~\eqref{third} into~\eqref{1T1S6}, we arrive at~\eqref{l7result} in Lemma~\ref{l7}.
\subsection{Proof of Lemma~\ref{l8}}
Left multiplying both sides of~\eqref{dcompacttheta} by $\frac{u^{T}}{m}$ and using the relations $u^{T}U^{-1}=\mathbf{1}^{T}$ and $u^{T}\mathbf{R}=u^{T}$, we obtain
\begin{equation}
	\vspace{-0.3em}
\begin{aligned}
\bar{\theta}_{t+1}&=\frac{u^{T}}{m}\big(\mathbf{R}\boldsymbol{\theta}_{t}+\boldsymbol{\vartheta}_{R,t}-(U^{-1}+Z_{t}^{-1}-U^{-1})(\boldsymbol{s}_{t+1}-\boldsymbol{s}_{t})\big)\\
&=\frac{u^{T}\boldsymbol{\theta}_{t}}{m}+\frac{u^{T}\boldsymbol{\vartheta}_{R,t}}{m}-\frac{\mathbf{1}^{T}(\boldsymbol{s}_{t+1}-\boldsymbol{s}_{t})}{m}\\
&\quad-\frac{u^{T}(Z_{t}^{-1}-U^{-1})}{m}(\boldsymbol{s}_{t+1}-\boldsymbol{s}_{t}).\label{1T2S1}
\end{aligned}
\vspace{-0.3em}
\end{equation}

Using~\eqref{dcompacttheta},~\eqref{1T2S1}, and $\bar{\mathbf{R}}\mathbf{1}=(I+R-\frac{\mathbf{1}u^{T}}{m})\mathbf{1}=\mathbf{0}$ yields
\begin{equation}
\begin{aligned}
&\boldsymbol{\theta}_{t+1}-\mathbf{1}\bar{\theta}_{t+1}\\
&=\bar{\mathbf{R}}(\boldsymbol{\theta}_{t}\!-\!\mathbf{1}\bar{\theta}_{t})+(I-\frac{\mathbf{1}u^{T}}{m})\boldsymbol{\vartheta}_{R,t}-(U^{-1}\!-\!\frac{\mathbf{1}\mathbf{1}^{T}}{m})(\boldsymbol{s}_{t+1}-\boldsymbol{s}_{t})\\
&\quad\quad -\Big(Z_{t}^{-1}-U^{-1}-\frac{\mathbf{1}u^{T}(Z_{t}^{-1}-U^{-1})}{m}\Big)(\boldsymbol{s}_{t+1}-\boldsymbol{s}_{t})\\
&=\bar{\mathbf{R}}(\boldsymbol{\theta}_{t}-\mathbf{1}\bar{\theta}_{t})+\Pi_{u}\boldsymbol{\vartheta}_{R,t}-\Pi_{U}(\boldsymbol{s}_{t+1}-\boldsymbol{s}_{t})\!-\!\Pi_{U}^{e}(\boldsymbol{s}_{t+1}-\boldsymbol{s}_{t}),\label{1T2S2}
\end{aligned}
\end{equation}
where $\Pi_{u}$, $\Pi_{U}$, and $\Pi_{U}^{e}$ are defined in Sec.~\ref{learningaccuracy}.

We proceed to estimate the term $\boldsymbol{s}_{t+1}-\boldsymbol{s}_{t}$ in~\eqref{1T2S2}. Based on dynamics~\eqref{dcompacts} and the relation $Cv=\mathbf{0}$, we have
\begin{equation}
	\vspace{-0.3em}
\begin{aligned}
\boldsymbol{s}_{t+1}-\boldsymbol{s}_{t}&=\mathbf{C}\boldsymbol{s}_{t}+\boldsymbol{\zeta}_{C,t}+\lambda_{t}\nabla \boldsymbol{f}_{t}(\boldsymbol{\theta}_{t})-\boldsymbol{s}_{t}\\
&=C(\boldsymbol{s}_{t}-\omega\bar{s}_{t})+\boldsymbol{\zeta}_{C,t}+\lambda_{t}\nabla \boldsymbol{f}_{t}(\boldsymbol{\theta}_{t}).\label{1T2S3}
\end{aligned}
\vspace{-0.3em}
\end{equation}

Substituting~\eqref{1T2S3} into~\eqref{1T2S2} and using $\mathbf{1}^{T}C=\mathbf{0}^{T}$ lead to
\begin{equation}
\begin{aligned}
&\boldsymbol{\theta}_{t+1}-\mathbf{1}\bar{\theta}_{t+1}=\bar{\mathbf{R}}(\boldsymbol{\theta}_{t}-\mathbf{1}\bar{\theta}_{t})+\Pi_{u}\boldsymbol{\vartheta}_{R,t}\\
&\quad-\Pi_{U}C(\boldsymbol{s}_{t}-\omega\bar{s}_{t})-\Pi_{U}\boldsymbol{\zeta}_{C,t}-\Pi_{U}\lambda_{t}\nabla \boldsymbol{f}_{t}(\boldsymbol{\theta}_{t})\\
&\quad-\Pi_{U}^{e}C(\boldsymbol{s}_{t}-\omega\bar{s}_{t})-\Pi_{U}^{e}\boldsymbol{\zeta}_{C,t}-\Pi_{U}^{e}\lambda_{t}\nabla \boldsymbol{f}_{t}(\boldsymbol{\theta}_{t}).\label{1T2S4}
\end{aligned}
\end{equation}
Taking the norm $\|\cdot\|_{R}$ on both sides of~\eqref{1T2S4} yields
\begin{equation}
\begin{aligned}
&\|\boldsymbol{\theta}_{t+1}-\mathbf{1}\bar{\theta}_{t+1}\|_{R}^2\\
&\leq\left(\|\bar{\mathbf{R}}\|_{R}\|\boldsymbol{\theta}_{t}-\mathbf{1}\bar{\theta}_{t}\|_{R}+(\|\Pi_{U}C\|_{R}+\|\Pi_{U}^{e}C\|_{R})\|\boldsymbol{s}_{t}-\omega\bar{s}_{t}\|_{R}\right.\\
&\left.\quad+\lambda_{t}(\|\Pi_{U}\|_{R}+\|\Pi_{U}^{e}\|_{R})\|\nabla\boldsymbol{f}_{t}(\boldsymbol{\theta}_{t})\|_{R}\right)^2\\
&\quad+\big\|\Pi_{u}\boldsymbol{\vartheta}_{R,t}-(\Pi_{U}+\Pi_{U}^{e})\boldsymbol{\zeta}_{C,t}\big\|_{R}^2\\
&\quad+2\big\langle \bar{\mathbf{R}}(\boldsymbol{\theta}_{t}-\mathbf{1}\bar{\theta}_{t})-(\Pi_{U}+\Pi_{U}^{e})C(\boldsymbol{s}_{t}-\omega\bar{s}_{t})\\
&\quad-\lambda_{t}(\Pi_{U}+\Pi_{U}^{e})\nabla\boldsymbol{f}_{t}(\boldsymbol{\theta}_{t}),\Pi_{u}\boldsymbol{\vartheta}_{R,t}-(\Pi_{U}+\Pi_{U}^{e})\boldsymbol{\zeta}_{C,t}\big\rangle_{R},\nonumber
\end{aligned}
\end{equation}
where $\langle\cdot,\cdot\rangle_{R}$ denotes the inner product induced~\footnote{Since one can verify $\|\boldsymbol{\theta}_{t}\|_{R}=\|\tilde{R}\boldsymbol{\theta}_{t}\|_{2}$ where $\tilde{R}$ is discussed in Sec.~\ref{supportlemma}, we have the norm $\|\cdot\|_{R}$ satisfying the Parallelogram law, implying that it has an associated inner product $\langle\cdot,\cdot\rangle_{R}$. The detailed discussions can be found in~\cite{Tailoring,spectral}.} by the norm $\|\cdot\|_{R}$.

Using the definition $\|\bar{\mathbf{R}}\|_{R}=1-\rho_{R}$ and the relation $(a+b)^2\leq (1+\epsilon)a^2+(1+\epsilon^{-1})b^2$ for any scalars $a$, $b$ and $\epsilon>0$ (setting $\epsilon=\frac{1}{1-\rho_{R}}-1$, implying $1+\epsilon^{-1}=\frac{1}{\rho_{R}}$), we have
\begin{equation}
\begin{aligned}
&\|\boldsymbol{\theta}_{t+1}-\mathbf{1}\bar{\theta}_{t+1}\|_{R}^2\\
&\leq(1-\rho_{R})\|\boldsymbol{\theta}_{t}-\mathbf{1}\bar{\theta}_{t}\|_{R}^2\\
&\quad+\frac{4\|C\|_{R}^2\big(\|\Pi_{U}\|_{R}^2+\|\Pi_{U}^{e}\|_{R}^2\big)}{\rho_{R}}\|\boldsymbol{s}_{t}-\omega\bar{s}_{t}\|_{R}^2\\
&\quad+\frac{4\lambda_{t}^2\big(\|\Pi_{U}\|_{R}^2+\|\Pi_{U}^{e}\|_{R}^2\big)}{\rho_{R}}\|\nabla \boldsymbol{f}_{t}(\boldsymbol{\theta}_{t})\|_{R}^2\\
&\quad+\big\|\Pi_{u}\boldsymbol{\vartheta}_{R,t}-(\Pi_{U}+\Pi_{U}^{e})\boldsymbol{\zeta}_{C,t}\big\|_{R}^2\\
&\quad+2\big\langle \bar{\mathbf{R}}(\boldsymbol{\theta}_{t}-\mathbf{1}\bar{\theta}_{t})-(\Pi_{U}+\Pi_{U}^{e})C(\boldsymbol{s}_{t}-\omega\bar{s}_{t})\\
&\quad-\lambda_{t}(\Pi_{U}+\Pi_{U}^{e})\nabla\boldsymbol{f}_{t}(\boldsymbol{\theta}_{t}),\Pi_{u}\boldsymbol{\vartheta}_{R,t}-(\Pi_{U}+\Pi_{U}^{e})\boldsymbol{\zeta}_{C,t}\big\rangle_{R}.\label{1T2S6}
\end{aligned}
\end{equation}

By taking the expectation on both sides of~\eqref{1T2S6} and using inequality~\eqref{D14}, Lemma~\ref{l5}, and~Assumption~\ref{a4}, one has
\begin{equation}
\mathbb{E}\big[\|\boldsymbol{\theta}_{t+1}-\mathbf{1}\bar{\theta}_{t+1}\|_{R}^2\big]\leq(1-\rho_{R})\mathbb{E}\big[\|\boldsymbol{\theta}_{t}-\mathbf{1}\bar{\theta}_{t}\|_{R}^2\big]+\Phi_{\theta,t},\label{1T2S7}
\end{equation}
where
\begin{flalign}
\Phi_{\theta,t}&=\frac{4\delta_{R,C}^2\|C\|_{R}^2\big(\|\Pi_{U}\|_{R}^2+\|\Pi_{U}^{e}\|_{R}^2\big)}{\rho_{R}}\mathbb{E}\big[\|\boldsymbol{s}_{t}-\omega\bar{s}_{t}\|_{C}^2\big]\nonumber\\
&\quad+\frac{8m\delta_{R,F}^2(\kappa^2+D^2)\big(\|\Pi_{U}\|_{R}^2+\|\Pi_{U}^{e}\|_{R}^2\big)}{\rho_{R}}\lambda_{t}^2\nonumber\\
&\quad+2\|\Pi_{u}\|_{R}^2\delta_{R,F}^2\sum_{i,j}(R_{ij}\sigma_{j,t,\vartheta})^2\nonumber\\
&\quad+2\|\Pi_{U}+\Pi_{U}^{e}\|_{R}^2\delta_{R,F}^2\sum_{i,j}(C_{ij}\sigma_{j,t,\zeta})^2.\label{1T2Sdelta}
\end{flalign}

Iterating~\eqref{1T2S7} from $0$ to $t$, we obtain
\begin{equation}
\begin{aligned}
\mathbb{E}\big[\|\boldsymbol{\theta}_{t+1}-\mathbf{1}\bar{\theta}_{t+1}\|_{R}^2\big]&\leq (1-\rho_{R})^{t+1}\mathbb{E}\big[\|\boldsymbol{\theta}_{0}-\mathbf{1}\bar{\theta}_{0}\|_{R}^2\big]\\
&\quad+\sum_{p=0}^{t-1}(1-\rho_{R})^{t-p}\Phi_{\theta,p}+\Phi_{\theta,t}.\label{1T2S8}
\end{aligned}
\end{equation}

When $t=0$, using the definitions of $\lambda_{t}$ and $\sigma_{\vartheta}^{+}$ yields
\begin{equation}
\begin{aligned}
&\mathbb{E}\big[\|\boldsymbol{\theta}_{1}-\mathbf{1}\bar{\theta}_{1}\|_{R}^2\big]\leq(1-\rho_{R})\mathbb{E}\big[\|\boldsymbol{\theta}_{0}-\mathbf{1}\bar{\theta}_{0}\|_{R}^2\big]\\
&\quad+\tau_{\theta1}\mathbb{E}\big[\|\boldsymbol{s}_{0}-\omega\bar{s}_{0}\|_{C}^2\big]+\tau_{\theta2}+\tau_{\theta3}+\tau_{\theta4},\label{1T2S10}
\end{aligned}
\end{equation}
where the constants $\tau_{\theta1}$ to $\tau_{\theta4}$ are given in the statement of Lemma~\ref{l8}.

When $t>0$, we analyze each item on the right hand side of~\eqref{1T2S8}. Using an argument similar to the derivation of~\eqref{1T1S6i}, the first term on the right hand side of~\eqref{1T2S8} satisfies
\begin{equation}
(1-\rho_{R})^{t+1}\mathbb{E}\big[\|\boldsymbol{\theta}_{0}-\mathbf{1}\bar{\theta}_{0}\|_{R}^2\big]\leq \frac{4\mathbb{E}\big[\|\boldsymbol{\theta}_{0}-\mathbf{1}\bar{\theta}_{0}\|_{R}^2\big](t+1)^{-2}}{(e\ln(1-\rho_{R}))^2}.\label{1T2S9i}
\end{equation}

Next we characterize the second term on the right hand side of~\eqref{1T2S8}. Using~\eqref{1T2Sdelta} and the definitions of $\lambda_{t}$, $\sigma_{\vartheta}^{+}$, and $\varsigma_{\vartheta}$ yields
\begin{flalign}
&\sum_{p=0}^{t-1}(1-\rho_{R})^{t-p}\Phi_{\theta,p}=\tau_{\theta1}\sum_{p=0}^{t-1}(1-\rho_{R})^{t-p}\mathbb{E}\big[\|\boldsymbol{s}_{p}-\omega\bar{s}_{p}\|_{C}^2\big]\nonumber\\
&\quad+\sum_{p=0}^{t-1}(1-\rho_{R})^{t-p}\Big(\frac{\tau_{\theta2}}{(p+1)^{2v}}+\frac{\tau_{\theta3}}{(p+1)^{2\varsigma_{\vartheta}}}+\frac{\tau_{\theta4}}{(p+1)^{2\varsigma_{\zeta}}}\Big).\label{1T2S9}
\end{flalign}

We estimate an upper bound on the right hand side of~\eqref{1T2S9}:

According to~\eqref{l7result} in Lemma~\ref{l7}, we obtain
\begin{flalign}
&\sum_{p=0}^{t-1}(1-\rho_{R})^{t-p}\mathbb{E}[\|\boldsymbol{s}_{p}-\omega\bar{s}_{p}\|_{C}^2]\leq(1-\rho_{R})^{t}\mathbb{E}[\|\boldsymbol{s}_{0}-\omega\bar{s}_{0}\|_{C}^2]\nonumber\\
&\quad+\sum_{p=1}^{t-1}(1-\rho_{R})^{t-p}\left(c_{\boldsymbol{s}1}p^{-2}+c_{\boldsymbol{s}2}p^{-2v}+c_{\boldsymbol{s}3}p^{-2\varsigma_{\zeta}}\right).\label{1T2S9ii1}
\end{flalign}
By using~\eqref{1T2S9i} and an argument similar to the derivation of~\eqref{1T1S6iii}, we obtain
\begin{flalign}
&\sum_{p=0}^{t-1}(1-\rho_{R})^{t-p}\mathbb{E}\left[\|\boldsymbol{s}_{p}-\omega\bar{s}_{p}\|_{C}^2\right]\leq \frac{4\mathbb{E}\big[\|\boldsymbol{s}_{0}-\omega\bar{s}_{0}\|_{C}^2\big]t^{-2}}{(e\ln(1-\rho_{R}))^2}\nonumber\\
&\quad+\frac{8\left(c_{\boldsymbol{s}1}t^{-2}+c_{\boldsymbol{s}2}t^{-2v}+c_{\boldsymbol{s}3}t^{-2\varsigma_{\zeta}}\right)}{\rho_{R}(e\ln(1-\frac{\rho_{R}}{2}))^2}.\label{1T2S9ii2}
\end{flalign}
Similarly, using again an argument similar to the derivation of~\eqref{1T1S6iii}, and then substituting~\eqref{1T2S9ii2} into~\eqref{1T2S9}, we arrive at
\begin{flalign}
&\sum_{p=0}^{t-1}(1-\rho_{R})^{t-p}\Phi_{\theta,p}\leq\frac{4\mathbb{E}\big[\|\boldsymbol{s}_{0}-\omega\bar{s}_{0}\|_{C}^2\big]t^{-2}}{(e\ln(1-\rho_{R}))^2}\nonumber\\
&\quad+\frac{8\tau_{\theta1}\left(c_{\boldsymbol{s}1}t^{-2}+c_{\boldsymbol{s}2}t^{-2v}+c_{\boldsymbol{s}3}t^{-2\varsigma_{\zeta}}\right)}{\rho_{R}(e\ln(1-\frac{\rho_{R}}{2}))^2}\nonumber\\
&\quad+\frac{8\left(\tau_{\theta2}t^{-2v}+\tau_{\theta3}t^{-2\varsigma_{\vartheta}}+\tau_{\theta4}t^{-2\varsigma_{\zeta}}\right)}{\rho_{R}(e\ln(1-\frac{\rho_{R}}{2}))^2}.\label{1T2S12}
\end{flalign}

According to the definition of $\Phi_{\theta,t}$ in~\eqref{1T2Sdelta}, we obtain that the third term on the right hand side of~\eqref{1T2S8} satisfies
\begin{equation}
\begin{aligned}
\Phi_{\theta,t}&\leq \tau_{\theta1}\left(c_{\boldsymbol{s}1}t^{-2}+c_{\boldsymbol{s}2}t^{-2v}+c_{\boldsymbol{s}3}t^{-2\varsigma_{\zeta}}\right)\\
&+\tau_{\theta2}t^{-2v}+\tau_{\theta3}t^{-2\varsigma_{\vartheta}}+\tau_{\theta4}t^{-2\varsigma_{\zeta}}.\label{third2}
\end{aligned}
\end{equation}

Substituting~\eqref{1T2S10},~\eqref{1T2S9i},~\eqref{1T2S12}, and~\eqref{third2} into~\eqref{1T2S8}, we arrive at~\eqref{l8result} in Lemma~\ref{l8}.

\subsection{Proof of Theorem~\ref{t1}}
By substituting~\eqref{1s} into~\eqref{1T2S1}, we have
\begin{flalign}
&\bar{\theta}_{t+1}-\theta_{t}^*
\!=\!\bar{\theta}_{t}-\lambda_{t}\nabla\bar{f}(\mathbf{1}\bar{\theta}_{t})-\theta_{t}^*+\lambda_{t}\nabla\bar{f}(\mathbf{1}\bar{\theta}_{t})-\lambda_{t}\nabla\bar{f}_{t}(\boldsymbol{\theta}_{t})\nonumber\\
&\quad+\bar{\vartheta}_{R,t}-\bar{\zeta}_{C,t}-\frac{u^{T}(Z_{t}^{-1}-U^{-1})}{m}(\boldsymbol{s}_{t+1}-\boldsymbol{s}_{t}).\label{1T3S1}
\end{flalign}
By taking the norm $\|\cdot\|_{2}$ and then expectation on both sides of \eqref{1T3S1}, we obtain the following inequality by using $\mathbb{E}[\bar{\vartheta}_{R,t}]=\mathbb{E}[\bar{\vartheta}_{C,t}]=0$ from Assumption~\ref{a4}:
\begin{equation}
	\begin{aligned}
		&\mathbb{E}[\|\bar{\theta}_{t+1}-\theta_{t}^*\|_{2}^{2}]\leq \mathbb{E}\Big[\Big\|\bar{\theta}_{t}-\lambda_{t}\nabla\bar{f}(\mathbf{1}\bar{\theta}_{t})-\theta_{t}^*+\lambda_{t}\nabla\bar{f}(\mathbf{1}\bar{\theta}_{t})\\
		&\quad-\lambda_{t}\nabla\bar{f}_{t}(\boldsymbol{\theta}_{t})-\frac{u^{T}(Z_{t}^{-1}-U^{-1})}{m}(\boldsymbol{s}_{t+1}-\boldsymbol{s}_{t})\Big\|_2^2\Big]\nonumber\\
		&\quad+\mathbb{E}[\|\bar{\vartheta}_{R,t}-\bar{\zeta}_{C,t}\|_{2}^2].
	\end{aligned}
\end{equation}

Since the relation $(x+y)^2\leq a_{t}x^2+b_{t}y^2$ holds for all $a_{t}>1$ and $b_{t}>1$ satisfying $(a_{t}-1)(b_{t}-1)=1$, we set $a_{t}=1+\frac{\lambda_{t}\mu}{2}$ and $b_{t}=1+\frac{2}{\lambda_{t}\mu}$ to obtain
\begin{equation}
	\begin{aligned}
		&\mathbb{E}[\|\bar{\theta}_{t+1}-\theta_{t}^*\|_{2}^{2}]\leq \Big(1+\frac{\lambda_{t}\mu}{2}\Big)\mathbb{E}\big[\|\bar{\theta}_{t}-\lambda_{t}\nabla\bar{f}(\mathbf{1}\bar{\theta}_{t})-\theta_{t}^*\|_2^2\big]\\
		&\quad+\Big(1+\frac{2}{\lambda_{t}\mu}\Big)\mathbb{E}\Big[\Big\|\lambda_{t}\nabla\bar{f}(\mathbf{1}\bar{\theta}_{t})-\lambda_{t}\nabla\bar{f}_{t}(\boldsymbol{\theta}_{t})\\
		&\quad-\frac{u^{T}(Z_{t}^{-1}-U^{-1})}{m}(\boldsymbol{s}_{t+1}-\boldsymbol{s}_{t})\Big\|_{2}^2\Big]+\mathbb{E}\big[\|\bar{\vartheta}_{R,t}-\bar{\zeta}_{C,t}\|_{2}^2\big].\label{1T3S3}
	\end{aligned}
\end{equation}

We analyze the first item on the right hand side of~\eqref{1T3S3}. The definition $\nabla\bar{f}(\mathbf{1}\bar{\theta}_{t})=\frac{\mathbf{1}^{T}\nabla\boldsymbol{f}(\mathbf{1}\bar{\theta}_{t})}{m}$ implies $\nabla\bar{f}(\mathbf{1}\bar{\theta}_{t})=\nabla F(\bar{\theta}_{t})$. Combing this relation and Assumption~\ref{a2}(ii)-(iii) yields
\begin{equation}
\begin{aligned}
&\|\bar{\theta}_{t}-\lambda_{t}\nabla\bar{f}(\mathbf{1}\bar{\theta}_{t})-\theta^*_t\|_{2}^2\leq (1-\lambda_{t}\mu)\|\bar{\theta}_{t}-\theta^*_t\|_{2}^2\\
&\quad-2\lambda_t(F(\bar{\theta}_{t})-F(\theta^*_{t}))+\lambda_t^2D^2.\label{1T3S4}
\end{aligned}
\end{equation}
Using the mean value theorem and~\eqref{approximateresult2} in Lemma~\ref{l9}, one has
\begin{equation}
	\begin{aligned}
		&\mathbb{E}\big[F(\bar{\theta}_{t})-F(\theta^*_{t})\big]\geq \mathbb{E}\big[F(\theta^*)-F(\theta^*_{t})\big]\\
		&=\mathbb{E}\big[\big\langle-\nabla F(\chi), \theta^*-\theta^*_{t}\big\rangle_{2}\big]\geq -\frac{2D\kappa}{\mu\sqrt{t+1}},\label{1aaT3S4}
	\end{aligned}
\end{equation}
where $\chi$ is given by $\chi=a\theta^*+(1-a)\theta^*_{t}$ for some $a\in(0,1).$

We proceed to estimate the second term on the right hand side of~\eqref{1T3S3}. Using Assumption~\ref{a3}(ii) and the definitions of $\nabla\bar{f}_{t}(\boldsymbol{\theta}_{t})$, $\nabla\bar{f}(\boldsymbol{\theta}_{t})$,~and $\nabla f_{i}(\theta_{i,t})$, we have
\begin{equation}
	\begin{aligned}
		&\mathbb{E}\big[\big\|\nabla\bar{f}_{t}(\boldsymbol{\theta}_{t})-\nabla\bar{f}(\boldsymbol{\theta}_{t})\big\|_2^2\big]\\
		&\!\leq\!\frac{1}{m}\sum_{i=1}^{m}\mathbb{E}\Big[\Big\|\frac{1}{t+1}\sum_{k=0}^{t}\nabla l(\theta_{i,k},\xi_{i,k})-\nabla f_{i}(\theta_{i,t})\Big\|_{2}^2\Big]\\
		&\!\leq\!\frac{1}{m(t+1)^2}\sum_{i=1}^{m}\sum_{k=0}^{t}\mathbb{E}\big[\|\nabla l(\theta_{i,t},\xi_{i,k})-\nabla f_{i}(\theta_{i,t})\|_{2}^2\big]\!\leq\! \frac{\kappa^2}{t+1},\label{1C8}
	\end{aligned}
\end{equation}
which further implies the following inequality by using Assumption~\ref{a3}(iii) and Lemma~\ref{l5}:
\begin{flalign}
	&\mathbb{E}\big[\|\nabla\bar{f}(\mathbf{1}\bar{\theta}_{t})-\nabla\bar{f}_{t}(\boldsymbol{\theta}_{t})\|_2^2\big]\nonumber\\
	&\leq2\mathbb{E}\big[\|\nabla\bar{f}_{t}(\boldsymbol{\theta}_{t})-\nabla\bar{f}(\boldsymbol{\theta}_{t})\|_2^2\big]\!+\!2\mathbb{E}\big[\|\nabla\bar{f}(\boldsymbol{\theta}_{t})-\nabla\bar{f}(\mathbf{1}\bar{\theta}_{t})\|_2^2\big]\nonumber\\
	&\leq\frac{2\kappa^2}{t+1}+2\delta_{F,R}^2L^2\mathbb{E}\left[\|\boldsymbol{\theta}_{t}-\mathbf{1}\bar{\theta}_{t}\|_{R}^2\right].\label{1aaT3aaS4}
\end{flalign}

Substituting~\eqref{1aaT3S4} into~\eqref{1T3S4}, and then plugging~\eqref{1T3S4} and~\eqref{1aaT3aaS4} into~\eqref{1T3S3} yield
\begin{flalign}
	&\mathbb{E}\big[\|\bar{\theta}_{t+1}-\theta_{t}^*\|_{2}^{2}\big]\leq \Big(1-\frac{\lambda_{t}\mu}{2}\Big)\mathbb{E}\big[\|\bar{\theta}_{t}-\theta_{t}^*\|_{2}^2\big]\nonumber\\
	&\quad+2\Big(1+\frac{2}{\lambda_{t}\mu}\Big)\Bigg(\frac{2\lambda_t^2\kappa^2}{t+1}+2\lambda_t^2\delta_{F,R}^2L^2\mathbb{E}\left[\|\boldsymbol{\theta}_{t}-\mathbf{1}\bar{\theta}_{t}\|_{R}^2\right]\nonumber\\
	&\quad+\mathbb{E}\Big[\Big\|\frac{u^{T}(Z_{t}^{-1}\!-\!U^{-1})}{m}(\boldsymbol{s}_{t+1}\!-\!\boldsymbol{s}_{t})\Big\|_{2}^2\Big]\Bigg)\!+\!\mathbb{E}\big[\|\bar{\vartheta}_{R,t}\!-\!\bar{\zeta}_{C,t}\|_{2}^2\big]\nonumber\\	&\quad+\Big(1+\frac{\lambda_{0}\mu}{2}\Big)\Big(\frac{4D\lambda_t\kappa}{\mu\sqrt{t+1}}+\lambda_t^2D^2\Big),\label{1T3S5}
\end{flalign}
where we used $\lambda_{t}\leq \lambda_{0}$ and $\big(1+\frac{\lambda_{t}\mu}{2}\big)\big(1-\lambda_{t}\mu\big)\leq 1-\frac{\lambda_{t}\mu}{2}$.

Combining~\eqref{1T3S5} and the inequality 
$\|\bar{\theta}_{t+1}-\theta_{t+1}^*\|_{2}^2\leq (1+\frac{\lambda_{t}\mu}{4})\|\bar{\theta}_{t+1}-\theta_{t}^*\|_{2}^2 +(1+\frac{4}{\lambda_{t}\mu})\|\theta_{t+1}^{*}-\theta_{t}^*\|_{2}^2$, one obtains
\begin{flalign}
	&\mathbb{E}\big[\|\bar{\theta}_{t+1}-\theta_{t+1}^*\|_{2}^{2}\big]\leq \Big(1-\frac{\lambda_{t}\mu}{4}\Big)\mathbb{E}\big[\|\bar{\theta}_{t}-\theta_{t}^*\|_{2}^2\big]+\frac{\hat{c}_{\bar{\theta}0}}{\lambda_{t}}\Bigg(\frac{2\lambda_t^2\kappa^2}{t+1}\nonumber\\
	&\quad+\!2\lambda_t^2\delta_{F,R}^2L^2\mathbb{E}\left[\|\boldsymbol{\theta}_{t}\!-\!\mathbf{1}\bar{\theta}_{t}\|_{R}^2\right]\!+\!\frac{\|u\|_{2}^2c_{z}^2\gamma_{z}^{2t}}{ m^2}\mathbb{E}\Big[\|\boldsymbol{s}_{t+1}\!-\!\boldsymbol{s}_{t}\|_{F}^2\Big]\!\Bigg)\nonumber\\
	&\quad+\Big(1+\frac{\lambda_{t}\mu}{4}\Big)\mathbb{E}\big[\big\|\bar{\vartheta}_{R,t}\!-\!\bar{\zeta}_{C,t}\big\|_{2}^2\big]\!+\!\frac{\hat{c}_{\bar{\theta}0}\mu}{4}\Big(\frac{4D\lambda_t\kappa}{\mu\sqrt{t+1}}\!+\!\lambda_t^2D^2\Big)\nonumber\\
	&\quad+\Big(1+\frac{4}{\lambda_{t}\mu}\Big)\mathbb{E}\big[\|\theta_{t+1}^{*}\!-\!\theta_{t}^*\|_{2}^2\big],\label{1T3S6}
\end{flalign}
where we defined $ \hat{c}_{\bar{\theta}0}\triangleq\frac{\lambda^2_{0}\mu^2+6\lambda_{0}\mu+8}{2\mu}$ and used Lemma~\ref{l2}.

We further characterize the term $\|\boldsymbol{s}_{t+1}-\boldsymbol{s}_{t}\|_{F}^2$ in~\eqref{1T3S6} by using
~\eqref{D14},~\eqref{1T2S3} and Lemma~\ref{l5}:
\begin{equation}
\begin{aligned}
\|\boldsymbol{s}_{t+1}-\boldsymbol{s}_{t}\|_{F}^2&\leq 3\delta_{F,C}^2\|C\|_{C}^{2}\|\boldsymbol{s}_{t}-\omega\bar{s}_{t}\|_{C}^2\\
&\quad+3\|\boldsymbol{\zeta}_{C,t}\|_{F}^2+6m(\kappa^2+D^2)\lambda_{t}^2.\label{1T3S7}
\end{aligned}
\end{equation}
Plugging~\eqref{1T3S7} into~\eqref{1T3S6}, we arrive at
\begin{equation}
	\mathbb{E}\big[\|\bar{\theta}_{t+1}-\theta_{t+1}^*\|_{2}^{2}\big]\le \left(1-\frac{\lambda_{t}\mu}{4}\right)\mathbb{E}\big[\|\bar{\theta}_{t}-\theta_{t}^*\|_{2}^2\big]+\Phi_{\bar{\theta},t},\label{1T3S8}
\end{equation}
where $\Phi_{\bar{\theta},t}$ is given by
\begin{flalign}
	&\Phi_{\bar{\theta},t}=\hat{c}_{\bar{\theta}0}\Bigg(\frac{2\lambda_t\kappa^2}{t+1}+2\lambda_t\delta_{F,R}^2L^2\mathbb{E}\left[\|\boldsymbol{\theta}_{t}-\mathbf{1}\bar{\theta}_{t}\|_{R}^2\right]\nonumber\\
	&\quad+\frac{3\delta_{F,C}^2\|C\|_{C}^{2}\|u\|_{2}^2c_{z}^2}{ m^2}\frac{\gamma^{2t}_{z}}{\lambda_{t}}\mathbb{E}\left[\|\boldsymbol{s}_{t}-\omega\bar{s}_{t}\|_{C}^2\right]\nonumber\\
	&\quad+\frac{3\|u\|_{2}^2c_{z}^2}{ m^2}\frac{\gamma_{z}^{2t}}{\lambda_{t}}\mathbb{E}\left[\|\boldsymbol{\zeta}_{C,t}\|_{F}^2\right]+\frac{6(\kappa^2+D^2)\|u\|_{2}^2c_{z}^2}{m}\gamma_{z}^{2t}\lambda_{t}\Bigg)\nonumber\\
	&\quad +2\Big(1\!+\!\frac{\lambda_{0}\mu}{4}\Big)\mathbb{E}\big[\|\bar{\vartheta}_{R,t}\|_{2}^2\!+\!\|\bar{\zeta}_{C,t}\|_{2}^2\big]+\frac{\hat{c}_{\bar{\theta}0}\mu}{4}\Big(\frac{4D\kappa}{\mu}\frac{\lambda_t}{\sqrt{t+1}}\nonumber\\
	&\quad+D^2\lambda_t^2\Big)+\Big(1+\frac{4}{\lambda_{t}\mu}\Big)\mathbb{E}\big[\|\theta_{t+1}^{*}-\theta_{t}^*\|_{2}^2\big].\label{1T3Sdelta}
\end{flalign}

By iterating~\eqref{1T3S8} from $0$ to $t$, we arrive at
\begin{equation}
	\begin{aligned}
		&\mathbb{E}[\|\bar{\theta}_{t+1}-\theta_{t+1}^*\|_{2}^{2}]\leq \prod_{p=0}^{t}\left(1-\frac{\lambda_{p}\mu}{4}\right)\mathbb{E}[\|\bar{\theta}_{0}-\theta_{0}^*\|_{2}^{2}]\\
		&\quad+\sum_{p=1}^{t}\prod_{q=p}^{t}\left(1-\frac{\lambda_{q}\mu}{4}\right)\Phi_{\bar{\theta},p-1}+\Phi_{\bar{\theta},t}.\label{1T3S9}
	\end{aligned}
\end{equation}
Since $\ln(1-u)\leq -u$ holds for all $u\in(0,1)$, we always have
$\prod_{p=0}^{t}(1-\frac{\lambda_{p}\mu}{4})\leq e^{-\frac{\mu}{4}\sum_{p=0}^{t}\lambda_{p}}$.
Hence, inequality~\eqref{1T3S9} can be rewritten as follows:
\begin{equation}
	\begin{aligned}
		&\mathbb{E}[\|\bar{\theta}_{t+1}-\theta_{t+1}^*\|_{2}^{2}]\leq e^{-\frac{\mu}{4}\sum_{p=0}^{t}\lambda_{p}}\mathbb{E}\left[\|\bar{\theta}_{0}-\theta_{0}^*\|_{2}^{2}\right]\\
		&\quad+\sum_{p=1}^{t}\Phi_{\bar{\theta},p-1}e^{-\frac{\mu}{4}\sum_{q=p}^{t}\lambda_{q}}+\Phi_{\bar{\theta},t}.\label{1T3S11}
	\end{aligned}
\end{equation}

We estimate the first term on the right hand side of~\eqref{1T3S11}. Since $\frac{\lambda_{0}}{(p+1)^{v}}\geq \frac{\lambda_{0}}{(t+1)^{v}}$ holds for all $t\geq p$ and $(t+1)^{v}\leq 2^{v}t^{v}$ holds for all $t>0$, we have
\begin{equation}
	\sum_{p=0}^{t}\lambda_{p}=\sum_{p=0}^{t}\frac{\lambda_{0}}{(p+1)^{v}}\geq \frac{\lambda_{0}}{(t+1)^{v}}(t+1)\geq  \frac{\lambda_{0}}{2^{v}t^{v-1}},\label{jizhun}
\end{equation}
which implies $e^{\frac{\mu}{4}\sum_{p=0}^{t}\lambda_{p}}\geq e^{\frac{\mu}{4}\frac{\lambda_{0}}{2^{v}t^{v-1}}}.$ Using Taylor expansion $e^{x}=\sum_{n=0}^{\infty}\frac{x^{n}}{n!}$, we have that there must exist some $n_{0}\in{\mathbb{N}^{+}}$ such that $e^{x}\geq \frac{x^{n_{0}}}{n_{0}!}$ holds when $x$ is nonnegative. By setting $n_{0}\triangleq\lceil \frac{1}{1-v} \rceil$, we have $(1-v)n_{0}\geq 1$, which further implies 
\begin{flalign}
	e^{\frac{\mu}{4}\sum_{p=0}^{t}\lambda_{p}}\!\geq\!\frac{1}{n_{0}!}\Big(\frac{\mu\lambda_{0}}{4\times2^{v}}\Big)^{n_{0}}t^{(1-v)n_{0}}\!\geq\!\frac{\Big(\frac{\mu\lambda_{0}}{4\times2^{v}}\Big)^{\frac{1}{1-v}}t}{(\frac{1}{1-v}+1)!}.\label{1T0ii}
\end{flalign}
Substituting~\eqref{1T0ii} into the first term on the right hand side of~\eqref{1T3S11}, we arrive at
\begin{equation}
	e^{-\frac{\mu}{4}\sum_{p=0}^{t}\lambda_{p}}\mathbb{E}\big[\|\bar{\theta}_{0}-\theta_{0}^*\|_{2}^{2}\big]\leq c_{\bar{\theta}1}t^{-1},\label{1T0result}
\end{equation}
where the constant $c_{\bar{\theta}1}$ is given by
\begin{equation}
	c_{\bar{\theta}1}=(\frac{2-v}{1-v})!\big(\frac{\mu\lambda_{0}}{4\times2^{v}}\big)^{\frac{1}{v-1}}\mathbb{E}[\|\bar{\theta}_{0}-\theta_{0}^*\|_{2}^{2}].\label{cbart1}
	\vspace{-0.5em}
\end{equation}

We proceed to analyze the second and third terms on the right hand side of~\eqref{1T3S11}. We select a constant $\alpha\!\in\!(v,\frac{1+v}{2})$. Since $e^{-\frac{\mu}{4}\sum_{q=\lceil t-t^{\alpha}\rceil+1}^{t}\lambda_{q}}<1$ is valid and $e^{-\frac{\mu}{4}\sum_{q=p}^{t}\lambda_{q}}\leq e^{-\frac{\mu}{4}\sum_{q=\lceil t-t^\alpha\rceil}^{t}\lambda_{q}}$ holds for all $p\in[1,\lceil t-t^{\alpha}\rceil]$, we obtain
\begin{flalign}
	&\sum_{p=1}^{t}\Phi_{\bar{\theta},p-1}e^{-\frac{\mu}{4}\sum_{q=p}^{t}\lambda_{q}}\!+\!\Phi_{\bar{\theta},t}\!\leq\!\sum_{p=1}^{\lceil t-t^\alpha\rceil}\!\!\Phi_{\bar{\theta},p-1}e^{-\frac{\mu}{4}\sum_{q=\lceil t-t^\alpha\rceil}^{t}\lambda_{q}}\nonumber\\
	&\quad+\sum_{p=\lceil t-t^{\alpha}\rceil+1}^{t}\!\!\!\Phi_{\bar{\theta},p-1}e^{-\frac{\mu}{4}\sum_{q=p}^{t}\lambda_{q}}+\Phi_{\bar{\theta},t}\nonumber\\
	&\leq \sum_{p=0}^{\lfloor t-t^\alpha\rfloor}\Phi_{\bar{\theta},p}e^{-\frac{\mu}{4}\sum_{q=\lceil t-t^{\alpha}\rceil}^{t}\lambda_{q}}+\sum_{p=\lceil t-t^{\alpha}\rceil}^{t}\Phi_{\bar{\theta},p}. \label{1T3S12}
\end{flalign}

We now analyze the first term on the right hand side of~\eqref{1T3S12}. To this end, we first characterize the term $e^{-\frac{\mu}{4}\sum_{q=\lceil t-t^{\alpha}\rceil}^{t}\lambda_{q}}$. Given $\frac{1}{(q+1)^{v}}\geq\frac{1}{(t+1)^{v}}$ for all $q\in[\lceil t-t^{\alpha} \rceil,t]$, we have
\begin{flalign}
	\sum_{q=\lceil t-t^{\alpha} \rceil}^{t}\lambda_{q}&=\!\!\sum_{q=\lceil t-t^{\alpha} \rceil}^{t}\frac{\lambda_{0}}{(q+1)^{v}}\geq \frac{\lambda_{0}}{(t+1)^{v}}(t-\lceil t-t^{\alpha} \rceil+1),\nonumber\\
	&\geq \frac{\lambda_{0}t^{\alpha}}{(t+1)^{v}}\geq \frac{\lambda_{0}t^{\alpha-v}}{2^{v}},\label{aiai}
\end{flalign}
where we have used $\lceil t-t^{\alpha} \rceil\leq t-t^{\alpha}+1$ and $(t+1)^{v}\leq 2^{v}t^{v}$.

Inequality~\eqref{aiai} implies $e^{\frac{\mu}{4}\sum_{q=\lceil t-t^{\alpha} \rceil}^{t}\lambda_{q}}\geq e^{\frac{\mu}{4}\frac{\lambda_{0}t^{\alpha-v}}{2^{v}}}.$
Using an argument similar to the derivation of~\eqref{1T0ii}, we set $n_{0}=\lceil \frac{1}{\alpha-v} \rceil$ (i.e., $(\alpha-v)n_{0}\geq 1$) for the Taylor expansion to obtain
\begin{equation}
	\begin{aligned}
		e^{\frac{\mu}{4}\sum_{q=\lceil t-t^{\alpha} \rceil}^{t}\lambda_{q}}\geq\frac{1}{(\frac{1}{\alpha-v}+1)!}\Big(\frac{\mu\lambda_{0}}{4\times2^{v}}\Big)^{\frac{1}{\alpha-v}}t.\label{1T122ii}
	\end{aligned}
\end{equation}

Substituting~\eqref{1T122ii} into the first term on the right hand side of~\eqref{1T3S12} leads to
\begin{equation}
	\sum_{p=0}^{\lfloor t-t^\alpha\rfloor}\Phi_{\bar{\theta},p}e^{-\frac{\mu}{4}\sum_{q=\lceil t-t^{\alpha}\rceil}^{t}\lambda_{q}}<\Big(\Phi_{\bar{\theta},0}+\sum_{p=1}^{\infty}\Phi_{\bar{\theta},p}\Big)ct^{-1},\label{1T13}
\end{equation}
where the constant $c$ is given by $c=(\frac{\alpha-v+1}{\alpha-v})!\left(\frac{\mu\lambda_{0}}{4\times2^{v}}\right)^{\frac{1}{v-\alpha}}$.

To proceed, we need to estimate an upper bound on $\Phi_{\bar{\theta},t}$. To this end, we first prove the following relations:

1) By using~\eqref{l8result} and $t^{-2v}\leq 4^{v}(t+1)^{-2v}$, we have
\begin{equation}
	\begin{aligned}
		&\lambda_{t}\mathbb{E}\left[\|\boldsymbol{\theta}_{t}-\mathbf{1}\bar{\theta}_{t}\|_{R}^2\right]\leq\lambda_{0}\Big(\frac{4c_{\boldsymbol{\theta}1}}{(t+1)^{v+2}}+\frac{4^{v}c_{\boldsymbol{\theta}2}}{(t+1)^{3v}}\\
		&\quad+\frac{4^{\varsigma_{\vartheta}}c_{\boldsymbol{\theta}3}}{(t+1)^{v+2\varsigma_{\vartheta}}}+\frac{4^{\varsigma_{\zeta}}c_{\boldsymbol{\theta}4}}{(t+1)^{v+2\varsigma_{\zeta}}}\Big).\label{f1T3S12i}
	\end{aligned}
\end{equation}

2) Lemma~\ref{l6} implies $\gamma_{z}^{2t}\leq \frac{16}{(e\ln(\gamma_{z}))^4t^4}\leq \frac{16\times2^{4}}{(e\ln(\gamma_{z}))^4(t+1)^4}$. Combing this relationship with~\eqref{l7result} in Lemma~\ref{l7}, we obtain
\begin{flalign}
	&\frac{\gamma_{z}^{2t}}{\lambda_{t}}\mathbb{E}\left[\|\boldsymbol{s}_{t}-\omega\bar{s}_{t}\|_{C}^2\right]\leq \frac{16\times2^{4}(c_{\boldsymbol{s}1}t^{-2}+c_{\boldsymbol{s}2}t^{-2v}+c_{\boldsymbol{s}3}t^{-2\varsigma_{\zeta}})}{\lambda_{0}(e\ln(\gamma_{z}))^4(t+1)^{4-v}}\nonumber\\
	&\leq \frac{16\times2^{6}c_{\boldsymbol{s}1}}{\lambda_{0}(e\ln(\gamma_{z}))^4(t+1)^{6-v}}+\frac{16\times2^{4+2v}c_{\boldsymbol{s}2}}{\lambda_{0}(e\ln(\gamma_{z}))^4(t+1)^{4+v}}\nonumber\\
	&\quad+\frac{16\times2^{4+2\varsigma_{\zeta}}c_{\boldsymbol{s}3}}{\lambda_{0}(e\ln(\gamma_{z}))^4(t+1)^{4-v+2\varsigma_{\zeta}}}.
\end{flalign}

3) Using the definitions $\sigma_{\zeta}^{+}\!=\!\max_{i\in[m]}\{\sigma_{i,0,\zeta}\}$ and
$\varsigma_{\zeta}\!=\!\min_{i\in[m]}\{\varsigma_{i,\zeta}\}$, we have
\begin{equation}
	\frac{\gamma_{z}^{2t}}{\lambda_{t}}\|\boldsymbol{\zeta}_{C,t}\|_{F}^2\leq \frac{16\times2^{4}\sum_{i,j}(C_{ij})^2(\sigma_{\zeta}^{+})^2}{\lambda_{0}(e\ln(\gamma_{z}))^4(t+1)^{4-v+2\varsigma_{\zeta}}}.
\end{equation}

4) Lemma~\ref{l6} implies $\gamma_{z}^{2t}\lambda_{t}\leq \frac{16\times2^{4}\lambda_{0}}{(e\ln(\gamma_{z}))^4(t+1)^{4+v}}.$ 

5) Using~\eqref{L1result} in Lemma~\ref{l9} yields
\begin{equation}
	\frac{\mathbb{E}\left[\|\theta_{t+1}^{*}-\theta_{t}^*\|_{2}^2\right]}{\lambda_{t}}\!\leq\!\frac{16(\kappa^2\!+\!D^2)}{\lambda_{0}(t+1)^{2-v}}\left(\frac{2}{\mu^2}\!+\!\frac{1}{L^2}\right).\label{f1T3S12ii}
\end{equation}

Substituting~\eqref{f1T3S12i}-\eqref{f1T3S12ii} into~\eqref{1T3Sdelta} and using the definitions of $\lambda_{t}$, $\sigma_{\vartheta}^{+}$, $\varsigma_{\vartheta}$, $\sigma_{\zeta}^{+}$, and
$\varsigma_{\zeta}$, we obtain
\begin{equation}
\begin{aligned}
&\Phi_{\bar{\theta},t}\leq \tau_{\bar{\theta}1}(t+1)^{-1-v}+\tau_{\bar{\theta}2}(t+1)^{-v-2}+\tau_{\bar{\theta}3}(t+1)^{-3v}\\
&\quad+\tau_{\bar{\theta}4}(t+1)^{-v-2\varsigma_{\vartheta}}+\tau_{\bar{\theta}5}(t+1)^{-v-2\varsigma_{\zeta}}+\tau_{\bar{\theta}6}(t+1)^{-6+v}\\
&\quad +\tau_{\bar{\theta}7}(t+1)^{-4-v}+\tau_{\bar{\theta}8}(t+1)^{-4+v-2\varsigma_{\zeta}}\\
&\quad +\tau_{\bar{\theta}9}(t+1)^{-4+v-2\varsigma_{\zeta}}+\tau_{\bar{\theta}10}(t+1)^{-4-v}+\tau_{\bar{\theta}11}(t+1)^{-2\varsigma_{\vartheta}}\\
&\quad +\tau_{\bar{\theta}12}(t+1)^{-2\varsigma_{\zeta}}+\tau_{\bar{\theta}13}(t+1)^{-v-\frac{1}{2}}+\tau_{\bar{\theta}14}(t+1)^{-2v}\\
&\quad +\tau_{\bar{\theta}15}(t+1)^{-2}+\tau_{\bar{\theta}16}(t+1)^{-2+v},\label{deltas}
\end{aligned}
\end{equation}
with the constants $\small\tau_{\bar{\theta}1}\!\!=\!\!2\hat{c}_{\bar{\theta}0}\kappa^2\lambda_0$, $\small\tau_{\bar{\theta}2}\!\!=\!\!\frac{4c_{\boldsymbol{\theta}1}\delta_{F,R}^2L^2\tau_{\bar{\theta}1}}{\kappa^2}$, $\tau_{\bar{\theta}3}\!\!=\!\!\frac{4^{v}c_{\boldsymbol{\theta}2}\tau_{\bar{\theta}2}}{4c_{\boldsymbol{\theta}1}}$, $\tau_{\bar{\theta}4}\!\!=\!\!\frac{4^{\varsigma_{\vartheta}}c_{\boldsymbol{\theta}3}\tau_{\bar{\theta}2}}{4c_{\boldsymbol{\theta}1}}$,~$\tau_{\bar{\theta}5}\!\!=\!\!\frac{4^{\varsigma_{\zeta}}c_{\boldsymbol{\theta}4}\tau_{\bar{\theta}2}}{4c_{\boldsymbol{\theta}1}}$, $\tau_{\bar{\theta}6}\!\!=\!\!\frac{3\times 2^{10}c_{\boldsymbol{s}1}\hat{c}_{\bar{\theta}0}\delta_{F,C}^2\|C\|_{C}^{2}\|u\|_{2}^2c_{z}^2}{m^2\lambda_{0}(e\ln(\gamma_{z}))^4}$, $\tau_{\bar{\theta}7}\!\!=\!\!\frac{4^{v}c_{\boldsymbol{s}2}\tau_{\bar{\theta}6}}{4c_{\boldsymbol{s}1}}$, $\tau_{\bar{\theta}8}\!\!=\!\!\frac{ 4^{\varsigma_{\zeta}}c_{\boldsymbol{s}3}\tau_{\bar{\theta}6}}{4c_{\boldsymbol{s}1}}$, $\small \tau_{\bar{\theta}9}\!=\!\frac{3\times 2^{8}\hat{c}_{\bar{\theta}0}\|u\|_{2}^2c_{z}^2\sum_{i,j}(C_{ij})^2(\sigma_{\zeta}^{+})^2}{m^2\lambda_{0}(e\ln(\gamma_{z}))^4}$, $\small\tau_{\bar{\theta}10}\!=\!\frac{3\times 2^{9}\hat{c}_{\bar{\theta}0}(\kappa^2+D^2)\|u\|_{2}^2c_{z}^2\lambda_{0}}{m(e\ln(\gamma_{z}))^4}$, $\small\tau_{\bar{\theta}11}\!=\!\frac{(4+\lambda_{0}\mu)\sum_{i,j}(R_{ij})^2(\sigma_{\vartheta}^{+})^2}{2}$,~$\small\tau_{\bar{\theta}12}\!=\!\frac{(4+\lambda_{0}\mu)\sum_{i,j}(C_{ij})^2(\sigma_{\zeta}^{+})^2}{2}$, $\small\tau_{\bar{\theta}13}\!=\!\hat{c}_{\bar{\theta}0}D\kappa\lambda_{0}$, $\small\tau_{\bar{\theta}14}\!=\!\frac{\hat{c}_{\bar{\theta}0}\mu\lambda_{0}^2D^2}{4}$, $\small\tau_{\bar{\theta}15}\!=\!\frac{16(\kappa^2\!+\!D^2)(2L^2+\mu^2)}{\mu^2L^2}$, and $\small\tau_{\bar{\theta}16}\!=\!\frac{4\tau_{\bar{\theta}15}}{\lambda_{0}\mu}$ (in which $\hat{c}_{\bar{\theta}0}$ is defined in~\eqref{1T3S6}).

By plugging~\eqref{deltas} into $\sum_{p=1}^{\infty}\Phi_{\bar{\theta},p}$, we can estimate the second term on the right hand side of~\eqref{1T13}. To illustrate this idea, we use $\sum_{p=1}^{\infty}\tau_{\bar{\theta}1}(p+1)^{-1-v}$ as an example:
\begin{equation}
	\sum_{p=1}^{\infty}\frac{\tau_{\bar{\theta}1}}{(t+1)^{1+v}}\!\leq\! \int_{1}^{\infty}\frac{\tau_{\bar{\theta}1}}{x^{1+v}}dx\!\leq\! \frac{\tau_{\bar{\theta}1}}{(1+v-1)2^{1-(1+v)}}.\label{example1}
\end{equation}
Applying an argument similar to the derivation of~\eqref{example1} to the other items on the right hand of $\sum_{p=1}^{\infty}\Phi_{\bar{\theta},p}$ yields
\begin{flalign}
	&\sum_{p=1}^{\infty}\Phi_{\bar{\theta},p}\leq \frac{\tau_{\bar{\theta}1}2^{v}}{v}+\frac{\tau_{\bar{\theta}2}2^{v+1}}{v+1}+\frac{\tau_{\bar{\theta}3}2^{3v-1}}{3v-1}+\frac{\tau_{\bar{\theta}4}2^{2\varsigma_{\vartheta}+v-1}}{v+2\varsigma_{\vartheta}-1}\nonumber\\
	&+\frac{\tau_{\bar{\theta}5}2^{2\varsigma_{\zeta}+v-1}}{v+2\varsigma_{\zeta}-1}+\frac{\tau_{\bar{\theta}6}2^{5-v}}{5-v}+\frac{\tau_{\bar{\theta}7}2^{v+3}}{v+3}+\frac{\tau_{\bar{\theta}8}2^{2\varsigma_{\zeta}-v+3}}{2\varsigma_{\zeta}-v+3}\nonumber\\ &+\frac{\tau_{\bar{\theta}9}2^{2\varsigma_{\zeta}-v+3}}{2\varsigma_{\zeta}-v+3}+\frac{\tau_{\bar{\theta}10}2^{v+3}}{v+3}+\frac{\tau_{\bar{\theta}11}2^{2\varsigma_{\vartheta}-1}}{2\varsigma_{\vartheta}-1}+\frac{\tau_{\bar{\theta}12}2^{2\varsigma_{\zeta}-1}}{2\varsigma_{\zeta}-1}\nonumber\\
	&+\frac{\tau_{\bar{\theta}13}2^{v-\frac{1}{2}}}{v-\frac{1}{2}}+\frac{\tau_{\bar{\theta}14}2^{2v-1}}{2v-1}+2\tau_{\bar{\theta}15}+\frac{\tau_{\bar{\theta}16}2^{1-v}}{1-v}\triangleq c'.\label{deltasp}
\end{flalign}

Substituting~$\Phi_{\bar{\theta},0}=\sum_{i=1}^{16}\tau_{\bar{\theta}i}$ and~\eqref{deltasp} into~\eqref{1T13} yields that the first term on the right hand side of~\eqref{1T3S12} satisfies
\begin{equation}
	\sum_{p=0}^{\lfloor t-t^\alpha\rfloor}\Phi_{\bar{\theta},p}e^{-\frac{\mu}{4}\sum_{q=\lceil t-t^{\alpha}\rceil}^{t}\lambda_{q}}<c_{\bar{\theta}2}t^{-1},\label{1Tdiyibu}
\end{equation}
where the constant $c_{\bar{\theta}2}$ is given by
\begin{equation}
	\small
	c_{\bar{\theta}2}\!=\!\left(\frac{\alpha-v+1}{\alpha-v}\right)!\left(\frac{\mu\lambda_{0}}{4\times2^{v}}\right)^{\frac{1}{v-\alpha}}\left(\sum_{i=1}^{16}\tau_{\bar{\theta}i}+c'\right).\label{cbart2}
\end{equation}

By plugging~\eqref{deltas} into $\sum_{p=\lceil t-t^{\alpha}\rceil}^{t}\Phi_{\bar{\theta},p}$, we can estimate the second term on the right hand side of~\eqref{1T3S12}. To illustrate this idea, we use $\sum_{p=\lceil t-t^{\alpha}\rceil}^{t}\tau_{\bar{\theta}2}(p+1)^{-v-2}$ as an example:

Since the relation $\frac{1}{(p+1)^{v+2}}\leq\frac{1}{(\lceil t-t^{\alpha} \rceil+1)^{v+2}}$ holds for all $p\in[\lceil t-t^{\alpha} \rceil,t]$ and any $\alpha\in(v,\frac{1+v}{2})$, we have
\begin{equation}
	\sum_{p=\lceil t-t^{\alpha} \rceil}^{t}\frac{1}{(p+1)^{v+2}}\leq \frac{1}{(\lceil t-t^{\alpha} \rceil+1)^{v+2}}(t-\lceil t-t^{\alpha} \rceil+1).\nonumber
\end{equation}

We next prove that the relation~$\lceil t-t^{\alpha} \rceil+1\geq t(1-\alpha)$ is valid for all~$\alpha\in(0,1)$. To this end, we consider function $f(t)=\alpha t-t^{\alpha}+1$, whose derivative satisfies $f'(t)=\alpha-\alpha t^{\alpha-1}$. For any $t>1$, we have $f'(t)>0$ and  
$f(t)\geq0$, which imply $t-t^{\alpha}+1\geq t(1-\alpha)$ and further
\begin{equation}
	\begin{aligned}
		\sum_{p=\lceil t-t^{\alpha} \rceil}^{t}\frac{1}{(p+1)^{v+2}}\!\leq\! \frac{t^{\alpha}+1}{t^{v+2}(1-\alpha)^{v+2}}\!\leq\! \frac{2t^{\alpha-(v+2)}}{(1-\alpha)^{v+2}},\label{1T10ii}
	\end{aligned}
\end{equation}
where we used $t^{\alpha}+1\leq 2t^{\alpha}$ for all $t>0$ in the last inequality.

Using an argument similar to the derivation of~\eqref{1T10ii} to the other items on the right hand side of $\sum_{p=\lceil t-t^{\alpha}\rceil}^{t}\Phi_{\bar{\theta},p}$ yields
\begin{flalign}
	&\sum_{p=\lceil t-t^{\alpha}\rceil}^{t}\Phi_{\bar{\theta},p}\leq 
	\frac{2\tau_{\bar{\theta}1}t^{\alpha-(v+1)}}{(1-\alpha)^{v+1}}+\frac{2\tau_{\bar{\theta}2}t^{\alpha-(v+2)}}{(1-\alpha)^{v+2}}+\frac{2\tau_{\bar{\theta}3}t^{\alpha-3v}}{(1-\alpha)^{3v}}\nonumber\\
	&\quad+\frac{2\tau_{\bar{\theta}4}t^{\alpha-(v+2\varsigma_{\vartheta})}}{(1-\alpha)^{v+2\varsigma_{\vartheta}}}+\frac{2\tau_{\bar{\theta}5}t^{\alpha-(v+2\varsigma_{\zeta})}}{(1-\alpha)^{v+2\varsigma_{\zeta}}}+\frac{2\tau_{\bar{\theta}6}t^{\alpha-(6-v)}}{(1-\alpha)^{6-v}}\nonumber\\
	&\quad+\frac{2(\tau_{\bar{\theta}7}+\tau_{\bar{\theta}10})t^{\alpha-(v+4)}}{(1-\alpha)^{v+4}}+\frac{2(\tau_{\bar{\theta}8}+\tau_{\bar{\theta}9})t^{\alpha-(4-v+2\varsigma_{\zeta})}}{(1-\alpha)^{4-v+2\varsigma_{\zeta}}}\nonumber\\
	&\quad+\frac{2\tau_{\bar{\theta}11}t^{\alpha-2\varsigma_{\vartheta}}}{(1-\alpha)^{2\varsigma_{\vartheta}}}+\frac{2\tau_{\bar{\theta}12}t^{\alpha-2\varsigma_{\zeta}}}{(1-\alpha)^{2\varsigma_{\zeta}}}+\frac{2\tau_{\bar{\theta}13}t^{\alpha-(v+\frac{1}{2})}}{(1-\alpha)^{v+\frac{1}{2}}}\nonumber\\
	&\quad +\frac{2\tau_{\bar{\theta}14}t^{\alpha-2v}}{(1-\alpha)^{2v}}+\frac{2\tau_{\bar{\theta}15}t^{\alpha-2}}{(1-\alpha)^{2}}+\frac{2\tau_{\bar{\theta}16}t^{\alpha-(2-v)}}{(1-\alpha)^{2-v}}.\label{1Tdierbu}
\end{flalign}

Substituting~\eqref{1Tdiyibu} and~\eqref{1Tdierbu} into~\eqref{1T3S12} and then plugging~\eqref{1T0result} and~\eqref{1T3S12} into~\eqref{1T3S11}, we arrive at
\begin{flalign}
	&\mathbb{E}[\|\bar{\theta}_{t+1}-\theta_{t+1}^*\|_{2}^{2}]\nonumber\\
	&\leq (c_{\bar{\theta}1}+c_{\bar{\theta}2})t^{-1}+c_{\bar{\theta}3}t^{\alpha-(v+1)}\nonumber\\
	&\quad+c_{\bar{\theta}4}t^{\alpha-(v+2)}+c_{\bar{\theta}5}t^{\alpha-3v}+c_{\bar{\theta}6}t^{\alpha-(v+2\varsigma_{\vartheta})}+c_{\bar{\theta}7}t^{\alpha-(v+2\varsigma_{\zeta})}\nonumber\\
	&\quad +c_{\bar{\theta}8}t^{\alpha-(6-v)}+c_{\bar{\theta}9}t^{\alpha-(v+4)}
	+c_{\bar{\theta}10}t^{\alpha-(4-v+2\varsigma_{\zeta})}\nonumber\\
	&\quad +c_{\bar{\theta}11}t^{\alpha-2\varsigma_{\vartheta}}+c_{\bar{\theta}12}t^{\alpha-2\varsigma_{\zeta}}+c_{\bar{\theta}13}t^{\alpha-(v+\frac{1}{2})}\nonumber\\
	&\quad+c_{\bar{\theta}14}t^{\alpha-2v}+c_{\bar{\theta}15}t^{\alpha-2}+c_{\bar{\theta}16}t^{\alpha-(2-v)},\label{1111}
\end{flalign}
for all $t>0$, where the constants $c_{\bar{\theta}1}$ is given in~\eqref{cbart1}, $c_{\bar{\theta}2}$ is given in~\eqref{cbart2}, and  
$c_{\bar{\theta}3}$ to $c_{\bar{\theta}16}$ are given by
\begin{equation}
	\small
	\left\{\begin{aligned}
		&c_{\bar{\theta}3}\!=\!\frac{2\tau_{\bar{\theta}1}}{(1-\alpha)^{v+1}},~ c_{\bar{\theta}4}\!=\!\frac{2\tau_{\bar{\theta}2}}{(1-\alpha)^{v+2}},~c_{\bar{\theta}5}\!=\!\frac{2\tau_{\bar{\theta}3}}{(1-\alpha)^{3v}},\\
		&c_{\bar{\theta}6}\!=\!\frac{2\tau_{\bar{\theta}4}}{(1-\alpha)^{v+2\varsigma_{\vartheta}}},~c_{\bar{\theta}7}=\frac{2\tau_{\bar{\theta}5}}{(1-\alpha)^{v+2\varsigma_{\zeta}}},~c_{\bar{\theta}8}=\frac{2\tau_{\bar{\theta}6}}{(1-\alpha)^{6-v}},\\
		&c_{\bar{\theta}9}\!=\!\frac{2(\tau_{\bar{\theta}7}+\tau_{\bar{\theta}10})}{(1-\alpha)^{v+4}},~c_{\bar{\theta}10}\!=\!\frac{2(\tau_{\bar{\theta}8}+\tau_{\bar{\theta}9})}{(1-\alpha)^{4-v+2\varsigma_{\zeta}}},~c_{\bar{\theta}11}=\frac{2\tau_{\bar{\theta}11}}{(1-\alpha)^{2\varsigma_{\vartheta}}},\\
		&c_{\bar{\theta}12}=\frac{2\tau_{\bar{\theta}12}}{(1-\alpha)^{2\varsigma_{\zeta}}},~c_{\bar{\theta}13}\!=\!\frac{2\tau_{\bar{\theta}13}}{(1-\alpha)^{v+\frac{1}{2}}},~c_{\bar{\theta}14}=\frac{2\tau_{\bar{\theta}14}}{(1-\alpha)^{2v}},\\
		&c_{\bar{\theta}15}=\frac{2\tau_{\bar{\theta}15}}{(1-\alpha)^{2}},~c_{\bar{\theta}16}\!=\!\frac{2\tau_{\bar{\theta}16}}{(1-\alpha)^{2-v}}.\label{cbart315}
	\end{aligned}
	\right.
\end{equation}

By plugging $\Phi_{\bar{\theta},0}=\sum_{i=1}^{16}\tau_{\bar{\theta}i}$ into~\eqref{1T3S8}, we obtain $\mathbb{E}\big[\|\bar{\theta}_{1}-\theta_{1}^*\|_{2}^{2}\big]\leq c_{\bar{\theta}17}$, where $c_{\bar{\theta}17}$ is given by
\begin{equation}
	c_{\bar{\theta}17}=\big(1-\frac{\lambda_{0}\mu}{4}\big)\mathbb{E}\big[\|\bar{\theta}_{0}-\theta_{0}^*\|_{2}^2\big]+\sum_{i=1}^{16}\tau_{\bar{\theta}i}.\label{cbart16}
	\vspace{-0.2em}
\end{equation}
Combing $\mathbb{E}\big[\|\bar{\theta}_{1}-\theta_{1}^*\|_{2}^{2}\big]\leq c_{\bar{\theta}17}$ and~\eqref{1111}, we arrive at
\vspace{-0.2em}
\begin{flalign}
	&\mathbb{E}[\|\bar{\theta}_{t}-\theta_{t}^*\|_{2}^{2}]\nonumber\\
	&\leq \max\{c_{\bar{\theta}1},c_{\bar{\theta}2},c_{\bar{\theta}17}\}t^{-1}+c_{\bar{\theta}3}t^{\alpha-(v+1)}\nonumber\\
	&\quad+c_{\bar{\theta}4}t^{\alpha-(v+2)}+c_{\bar{\theta}5}t^{\alpha-3v}+c_{\bar{\theta}6}t^{\alpha-(v+2\varsigma_{\vartheta})}+c_{\bar{\theta}7}t^{\alpha-(v+2\varsigma_{\zeta})}\nonumber\\
	&\quad +c_{\bar{\theta}8}t^{\alpha-(6-v)}+c_{\bar{\theta}9}t^{\alpha-(v+4)}
	+c_{\bar{\theta}10}t^{\alpha-(4-v+2\varsigma_{\zeta})}\nonumber\\
	&\quad +c_{\bar{\theta}11}t^{\alpha-2\varsigma_{\vartheta}}+c_{\bar{\theta}12}t^{\alpha-2\varsigma_{\zeta}}+c_{\bar{\theta}13}t^{\alpha-(v+\frac{1}{2})}\nonumber\\
	&\quad+c_{\bar{\theta}14}t^{\alpha-2v}+c_{\bar{\theta}15}t^{\alpha-2}+c_{\bar{\theta}16}t^{\alpha-(2-v)},\label{l10result}
\end{flalign}
for all $t>0$. Here, the constant $\alpha$ satisfies $\alpha\in(v,\frac{1+v}{2})$ and $c_{\bar{\theta}1}$ to $c_{\bar{\theta}17}$ are defined in~\eqref{cbart1},~\eqref{cbart2},~\eqref{cbart315} and~\eqref{cbart16}.

By using Lemma~\ref{l4} and Lemma~\ref{l5}, we obtain
\begin{equation}
	\|\boldsymbol{\theta}_{t}-\mathbf{1}\theta_{t}^*\|_{F}^2\leq2\delta_{F,R}^2\|\boldsymbol{\theta}_{t}-\mathbf{1}\bar{\theta}_{t}\|_{R}^2+2m\|\bar{\theta}_{t}-\theta_{t}^*\|_{2}^2.\label{1tzong}
\end{equation}
Taking the expectation on both sides of~\eqref{1tzong} and combining the result with~\eqref{l8result} and~\eqref{l10result}, we arrive at~\eqref{t1result} in Theorem~\ref{t1}.

\subsection{Proof of Theorem 3}
For the convenience of derivation, we introduce an auxiliary variable $s\in[0,t]$. 

By plugging~\eqref{1s} into~\eqref{1T2S1}, one has
\begin{flalign}
&\mathbb{E}\big[\|\bar{\theta}_{t+1-s}-\theta^*\|_{2}^2\big]\nonumber\\
&\leq \mathbb{E}\big[\|\bar{\theta}_{t-s}-\theta^*+\lambda_{t-s}\nabla\bar{f}_{t-s}(\mathbf{1}\bar{\theta}_{t-s})-\lambda_{t-s}\nabla\bar{f}_{t-s}(\boldsymbol{\theta}_{t-s})\nonumber\\
&\quad-\frac{u^{T}(Z_{t-s}^{-1}-U^{-1})}{m}(\boldsymbol{s}_{t+1-s}-\boldsymbol{s}_{t-s})+\bar{\vartheta}_{R,t-s}-\bar{\zeta}_{C,t-s}\nonumber\\
&\quad-\lambda_{t-s}\nabla\bar{f}_{t-s}(\mathbf{1}\bar{\theta}_{t-s})\|_{2}^2\big].\label{3T2}
\end{flalign}
By using the relation $(a+b+c)^2=a^2+b^2+c^2+2ab+2bc+2ac$ for vectors $a,b,c$ and the property $\mathbb{E}[\bar{\vartheta}_{R,t-s}]=\mathbb{E}[\bar{\vartheta}_{C,t-s}]=0$ in Assumption~\ref{a4}, we rewrite~\eqref{3T2} as follows:
\begin{flalign}
	&\mathbb{E}\big[\|\bar{\theta}_{t+1-s}-\theta^*\|_{2}^2\big]\nonumber\\
	&\leq \mathbb{E}\big[\|\bar{\theta}_{t-s}\!-\!\theta^*\|_{2}^2\big]\!+\!\mathbb{E}\Big[\Big\|\lambda_{t-s}\nabla\bar{f}_{t-s}(\mathbf{1}\bar{\theta}_{t-s})\!-\!\lambda_{t-s}\nabla\bar{f}_{t-s}(\boldsymbol{\theta}_{t-s})\nonumber\\
	&\quad-\frac{u^{T}(Z_{t-s}^{-1}-U^{-1})}{m}(\boldsymbol{s}_{t+1-s}-\boldsymbol{s}_{t-s})\Big\|_{2}^2\Big]\nonumber\\
	&\quad+\mathbb{E}\big[\|\bar{\vartheta}_{R,t-s}-\bar{\zeta}_{C,t-s}\|_{2}^2\big]+\mathbb{E}\big[\|\lambda_{t-s}\nabla\bar{f}_{t-s}(\mathbf{1}\bar{\theta}_{t-s})\|_{2}^2\big]\nonumber\\
	&\quad+2\mathbb{E}\Big[\Big\langle\bar{\theta}_{t-s}-\theta^*,\lambda_{t-s}\nabla\bar{f}_{t-s}(\mathbf{1}\bar{\theta}_{t-s})-\lambda_{t-s}\nabla\bar{f}_{t-s}(\boldsymbol{\theta}_{t-s})\nonumber\\
	&\quad+\bar{\vartheta}_{R,t-s}-\bar{\zeta}_{C,t-s}-\frac{u^{T}(Z_{t-s}^{-1}-U^{-1})}{m}(\boldsymbol{s}_{t+1-s}-\boldsymbol{s}_{t-s})\Big\rangle_{2}\Big]\nonumber\\
	&\quad+2\lambda_{t-s}\mathbb{E}\Big[\Big\langle\nabla\bar{f}_{t-s}(\mathbf{1}\bar{\theta}_{t-s})\!-\!\nabla\bar{f}_{t-s}(\boldsymbol{\theta}_{t-s})\!+\!\bar{\vartheta}_{R,t-s}\!-\!\bar{\zeta}_{C,t-s}\nonumber\\
	&\quad-\frac{u^{T}(Z_{t-s}^{-1}-U^{-1})}{m}(\boldsymbol{s}_{t+1-s}\!-\!\boldsymbol{s}_{t-s}),\lambda_{t-s}\nabla\bar{f}_{t-s}(\mathbf{1}\bar{\theta}_{t-s})\Big\rangle_{2}\Big]\nonumber\\
	&\quad+2\mathbb{E}\big[\big\langle\bar{\theta}_{t-s}-\theta^*,\lambda_{t-s}\nabla\bar{f}_{t-s}(\mathbf{1}\bar{\theta}_{t-s})\big\rangle_{2}\big].\label{3T3}
\end{flalign}

Next we characterize each item on the right hand side of~\eqref{3T3}:

1) By using Assumption~\ref{a3}(iii) and~\eqref{1C8}, we have
\begin{equation}
\begin{aligned}
&\mathbb{E}\left[\|\lambda_{t-s}\nabla\bar{f}_{t-s}(\mathbf{1}\bar{\theta}_{t-s})-\lambda_{t-s}\nabla\bar{f}_{t-s}(\boldsymbol{\theta}_{t-s})\|_2^2\right]\\
&\leq3\lambda_{t-s}^2\mathbb{E}\left[\|\nabla\bar{f}(\mathbf{1}\bar{\theta}_{t-s})-\nabla\bar{f}(\boldsymbol{\theta}_{t-s})\|_2^2\right]\\
&\quad+3\lambda_{t-s}^2\mathbb{E}\left[\|\nabla\bar{f}_{t-s}(\mathbf{1}\bar{\theta}_{t-s})-\nabla\bar{f}(\mathbf{1}\bar{\theta}_{t-s})\|_2^2\right]\\
&\quad+3\lambda_{t-s}^2\mathbb{E}\left[\|\nabla\bar{f}_{t-s}(\boldsymbol{\theta}_{t-s}))-\nabla\bar{f}(\boldsymbol{\theta}_{t-s}))\|_2^2\right]\\
&\leq 3L^2\delta_{F,R}^2\lambda_{t-s}^2\mathbb{E}\left[\|\boldsymbol{\theta}_{t-s}-\mathbf{1}\bar{\theta}_{t-s}\|_{R}^2\right]+\frac{6\kappa^2\lambda_{t-s}^2}{t-s+1}. \label{3T3i1}
\end{aligned}
\end{equation}
By using Lemma~\ref{l2}, we plug~\eqref{3T3i1} into the second term on the right hand side of~\eqref{3T3} to obtain
\begin{flalign}
&\mathbb{E}\Big[\Big\|\lambda_{t-s}\nabla\bar{f}_{t-s}(\mathbf{1}\bar{\theta}_{t-s})-\lambda_{t-s}\nabla\bar{f}_{t-s}(\boldsymbol{\theta}_{t-s})\nonumber\\
&\quad-\frac{u^{T}(Z_{t-s}^{-1}-U^{-1})}{m}(\boldsymbol{s}_{t+1-s}-\boldsymbol{s}_{t-s})\Big\|_{2}^2\Big]\nonumber\\
&\leq 6L^2\delta_{F,R}^2\lambda_{t-s}^2\mathbb{E}\big[\big\|\boldsymbol{\theta}_{t-s}-\mathbf{1}\bar{\theta}_{t-s}\big\|_{R}^2\big]+\frac{12\kappa^2\lambda^2_{t-s}}{t-s+1}\nonumber\\
&\quad+\frac{2\|u\|_{2}^2c_{z}^2}{m^2}\gamma_{z}^{2(t-s)}\mathbb{E}\big[\big\|\boldsymbol{s}_{t+1-s}-\boldsymbol{s}_{t-s}\big\|_{F}^2\big].\label{3T3i}
\end{flalign}

2) Based on the definitions of $\bar{\vartheta}_{R,t-s}$ and $\bar{\zeta}_{C,t-s}$, we have
\begin{equation}
\begin{aligned}
&\mathbb{E}\left[\|\bar{\vartheta}_{R,t-s}-\bar{\zeta}_{C,t-s}\|_{2}^2\right]\\
&\leq \frac{2\|u\|_{2}^2(\sigma_{\vartheta}^{+})^2\sum_{i,j}(R_{ij})^2}{m^2(t-s+1)^{2\varsigma_{\vartheta}}}+\frac{2(\sigma_{\zeta}^{+})^2\sum_{i,j}(C_{ij})^2}{m^2(t-s+1)^{2\varsigma_{\zeta}}}.\label{3T3ii}
\end{aligned}
\end{equation}

3) Using an argument similar to the derivation of~\eqref{D14} yields
\begin{equation}
\begin{aligned}
	&\mathbb{E}\left[\|\lambda_{t-s}\nabla\bar{f}_{t-s}(\mathbf{1}\bar{\theta}_{t-s})\|_{2}^2\right]\\
	&\leq\frac{1}{t-s+1}\sum_{k=0}^{t-s}\mathbb{E}[\|\nabla l(\bar{\theta}_{t-s},\xi_{i,k})\|_{2}^2]\\
	&\leq 2(\kappa^2+D^2)\lambda_{t-s}^2.\label{3T3iii}
\end{aligned}
\end{equation}

4) By using Assumption~\ref{a4} and~\eqref{3T3i} and defining $a_{t-s}\triangleq\frac{1}{(t-s+1)^{r}}$ with $r\in(\frac{1}{2},v)$, we obtain
\begin{flalign}
&2\mathbb{E}\Big[\Big\langle\bar{\theta}_{t-s}-\theta^*,\lambda_{t-s}\nabla\bar{f}_{t-s}(\mathbf{1}\bar{\theta}_{t-s})-\lambda_{t-s}\nabla\bar{f}_{t-s}(\boldsymbol{\theta}_{t-s})\nonumber\\
&\quad+\bar{\vartheta}_{R,t-s}-\bar{\zeta}_{C,t-s}-\frac{u^{T}(Z_{t-s}^{-1}-U^{-1})}{m}(\boldsymbol{s}_{t+1-s}-\boldsymbol{s}_{t-s})\Big\rangle_{2}\Big]\nonumber\\
&\leq a_{t-s}\lambda_{t-s}\mathbb{E}\big[\big\|\bar{\theta}_{t-s}-\theta^*\big\|_{2}^{2}\big]\nonumber\\
&\quad+\frac{1}{a_{t-s}\lambda_{t-s}}\mathbb{E}\Big[\Big\|\lambda_{t-s}\nabla\bar{f}_{t-s}(\mathbf{1}\bar{\theta}_{t-s})-\lambda_{t-s}\nabla\bar{f}_{t-s}(\boldsymbol{\theta}_{t-s})\nonumber\\
&\quad-\frac{u^{T}(Z_{t-s}^{-1}-U^{-1})}{m}(\boldsymbol{s}_{t+1-s}-\boldsymbol{s}_{t-s})\Big\|_{2}^2\Big]\nonumber\\
&\leq a_{t-s}\lambda_{t-s}\mathbb{E}\big[\big\|\bar{\theta}_{t-s}-\theta^*\big\|_{2}^{2}\big]\nonumber\\
&\quad+\frac{6L^2\delta_{F,R}^2\lambda_{t-s}}{a_{t-s}}\mathbb{E}\big[\big\|\boldsymbol{\theta}_{t-s}-\mathbf{1}\bar{\theta}_{t-s}\big\|_{R}^2\big]\!+\!\frac{12\kappa^2\lambda_{t-s}}{(t-s+1)a_{t-s}}\nonumber\\
&\quad+\!\!\frac{2\|u\|_{2}^2c_{z}^2}{m^2}\frac{\gamma_{z}^{2(t-s)}}{a_{t-s}\lambda_{t-s}}\mathbb{E}\big[\big\|\boldsymbol{s}_{t+1-s}-\boldsymbol{s}_{t-s}\big\|_{F}^2\big].\label{3T3iv}
\end{flalign}

5) By utilizing Assumption~\ref{a4},~\eqref{3T3i},~\eqref{3T3iii}, and
the relation $2\langle a,b\rangle\leq \|a\|^2+\|b\|^2$ for any vectors $a$ and $b$, we have
\begin{equation}
\begin{aligned}
&2\mathbb{E}\Big[\Big\langle\lambda_{t-s}\nabla\bar{f}_{t-s}(\mathbf{1}\bar{\theta}_{t-s})-\lambda_{t-s}\nabla\bar{f}_{t-s}(\boldsymbol{\theta}_{t-s})+\bar{\vartheta}_{t-s}^{R}-\bar{\zeta}_{t-s}^{C}\\
&\quad-\frac{u^{T}(Z_{t-s}^{-1}\!-\!U^{-1})}{m}(\boldsymbol{s}_{t+1-s}\!-\!\boldsymbol{s}_{t-s}),\lambda_{t-s}\nabla\bar{f}_{t-s}(\mathbf{1}\bar{\theta}_{t-s})\Big\rangle_{2}\Big]\\
&\leq 2(\kappa^2+D^2)\lambda_{t-s}^2+\mathbb{E}\Big[\Big\|\lambda_{t-s}\nabla\bar{f}_{t-s}(\mathbf{1}\bar{\theta}_{t-s})+\bar{\vartheta}_{t-s}^{R}-\bar{\zeta}_{t-s}^{C}\\
&\quad-\!\lambda_{t-s}\nabla\bar{f}_{t-s}(\boldsymbol{\theta}_{t-s})\!-\!\frac{u^{T}(Z_{t-s}^{-1}-U^{-1})}{m}(\boldsymbol{s}_{t+1-s}\!-\!\boldsymbol{s}_{t-s})\Big\|_{2}^2\Big]\\
&\leq 2(\kappa^2+D^2)\lambda_{t-s}^2+6L^2\delta_{F,R}^2\lambda_{t-s}^2\mathbb{E}\big[\big\|\boldsymbol{\theta}_{t-s}-\mathbf{1}\bar{\theta}_{t-s}\big\|_{R}^2\big]\\
&\quad+\frac{2\|u\|_{2}^2c_{z}^2}{m^2}\gamma_{z}^{2(t-s)}\mathbb{E}\big[\big\|\boldsymbol{s}_{t+1-s}-\boldsymbol{s}_{t-s}\big\|_{F}^2\big]+\frac{12\kappa^2\lambda^2_{t-s}}{t-s+1}. \nonumber
\end{aligned}
\end{equation}

6) The last term on the right hand side of~\eqref{3T3} satisfies
\begin{equation}
\begin{aligned}
&2\mathbb{E}\big[\big\langle\bar{\theta}_{t-s}-\theta^*,\lambda_{t-s}\nabla\bar{f}_{t-s}(\mathbf{1}\bar{\theta}_{t-s})\big\rangle_{2}\big]\\
&= 2\lambda_{t-s}\mathbb{E}\big[\big\langle \bar{\theta}_{t-s}-\theta^*,\nabla\bar{f}(\mathbf{1}\bar{\theta}_{t-s})\big\rangle_{2}\big]\\
&\quad-2\lambda_{t-s}\mathbb{E}\big[\big\langle \bar{\theta}_{t-s}-\theta^*,\nabla\bar{f}(\mathbf{1}\bar{\theta}_{t-s})-\nabla\bar{f}_{t-s}(\mathbf{1}\bar{\theta}_{t-s})\big\rangle_{2}\big].\label{3T3iv1}
\end{aligned}
\end{equation}

Assumption~\ref{a2}(iii) and the relation $\nabla F(\bar{\theta}_{t-s})=\nabla\bar{f}(\mathbf{1}\bar{\theta}_{t-s})$ imply that the first term on the right hand side of~\eqref{3T3iv1} satisfies
\begin{equation}
\begin{aligned}
2\lambda_{t-s}\mathbb{E}\left[\left\langle\bar{\theta}_{t-s}-\theta^*,\nabla\bar{f}(\mathbf{1}\bar{\theta}_{t-s})\right\rangle_{2}\right]\!\geq\! 2\lambda_{t-s}(F(\bar{\theta}_{t-s})\!-\!F(\theta^*)).\label{3T3iv2}
\end{aligned}
\end{equation}

By employing the Young's inequality, the second term on the right hand side of~\eqref{3T3iv1} satisfies
\begin{flalign}
&-2\lambda_{t-s}\mathbb{E}\left[\left\langle \bar{\theta}_{t-s}-\theta^*,\nabla\bar{f}(\mathbf{1}\bar{\theta}_{t-s})-\nabla\bar{f}_{t-s}(\mathbf{1}\bar{\theta}_{t-s})\right\rangle_{2}\right]\nonumber\\
&\geq -\lambda_{t-s}a_{t-s}\mathbb{E}\left[\left\|\bar{\theta}_{t-s}-\theta^*\right\|_{2}^2\right]\nonumber\\
&\quad-\frac{\lambda_{t-s}}{a_{t-s}}\mathbb{E}\left[\left\|\nabla\bar{f}(\mathbf{1}\bar{\theta}_{t-s})-\nabla\bar{f}_{t-s}(\mathbf{1}\bar{\theta}_{t-s})\right\|_{2}^2\right].\label{3T3iv3}
\end{flalign}

We further characterize the second term on the right hand side of~\eqref{3T3iv3}. To this end, we define an auxiliary random variable 
$\psi(\bar{\theta}_{t-s})\triangleq\nabla l(\bar{\theta}_{t-s},\xi_{i})-\mathbb{E}\left[\nabla l(\bar{\theta}_{t-s},\xi_{i})\right].$
Since data samples $\xi_{i}$ is independently and identically distributed across iterations, we have that the following equation always holds:
\begin{equation}
\mathbb{E}\left[\sum_{k=0}^{t-s}\psi(\theta_{i,t-s},\xi_{i,k})\!\times\!\sum_{k=t-s+1}^{t+1}\psi(\theta_{i,t-s},\xi_{i,k})\right]=0.\label{3T3iv4}
\end{equation}
Given that Assumption~\ref{a3}(ii) implies $\mathbb{E}[\|\psi(\theta_{i,t-s},\xi_{i,k})\|^2]\leq\kappa^2$, we can obtain the following inequality based on~\eqref{3T3iv4}:
\begin{flalign}
&\mathbb{E}\big[\big\|\nabla f_{i}(\theta_{i,t-s})-\nabla f_{i,t-s}(\theta_{i,t-s})\big\|_{2}^2\big]\nonumber\\
&=\mathbb{E}\Bigg[\Bigg\|\frac{1}{t-s+1}\sum_{k=0}^{t-s}\psi(\theta_{i,t-s},\xi_{i,k})\Bigg\|_{2}^2\Bigg]\leq \frac{\kappa^2}{t-s+1}.\label{3T3iv5}
\end{flalign}
Substituting~\eqref{3T3iv5} into \eqref{3T3iv3} yields
\begin{flalign}
&-2\lambda_{t-s}\mathbb{E}\big[\big\langle\bar{\theta}_{t-s}-\theta^*,\nabla\bar{f}(\mathbf{1}\bar{\theta}_{t-s})-\nabla\bar{f}_{t-s}(\mathbf{1}\bar{\theta}_{t-s})\big\rangle_{2}\big]\nonumber\\
&\geq -\lambda_{t-s}a_{t-s}\mathbb{E}\big[\big\|\bar{\theta}_{t-s}-\theta^*\big\|_{2}^2\big]-\frac{\lambda_{t-s}\kappa^2}{a_{t-s}(t-s+1)}.\label{3T3ivjieguo}
\end{flalign}

Plugging inequalities~\eqref{3T3iv2} and~\eqref{3T3ivjieguo} into~\eqref{3T3iv1} and considering the relation $\mathbb{E}[F(\bar{\theta}_{t-s})-F(\theta^*)]\geq \mathbb{E}[F(\theta_{i,t+1})-F(\theta^*)]$, the last term on the right hand side of~\eqref{3T3} satisfies
\begin{equation}
\begin{aligned}
&2\mathbb{E}\big[\big\langle\bar{\theta}_{t-s}-\theta^*,\lambda_{t-s}\nabla\bar{f}_{t-s}(\mathbf{1}\bar{\theta}_{t-s})\big\rangle_{2}\big]\\
&\leq -2\lambda_{t-s}\mathbb{E}\big[F(\theta_{i,t+1})-F(\theta^*)\big]+\frac{\lambda_{t-s}\kappa^2}{a_{t-s}(t-s+1)}\\
&\quad+\lambda_{t-s}a_{t-s}\mathbb{E}\big[\big\|\bar{\theta}_{t-s}-\theta^*\big\|_{2}^2\big].\label{3T3ivv}
\end{aligned}
\end{equation}

Substituting~\eqref{3T3i}-\eqref{3T3ivv} into~\eqref{3T3}, we arrive at
\begin{equation}
\begin{aligned}
&\mathbb{E}\left[\|\bar{\theta}_{t+1-s}-\theta^*\|_{2}^{2}\right]\leq -2\lambda_{t-s}\mathbb{E}\big[F(\theta_{i,t+1})-F(\theta^*)\big]\\
&\quad +(1+2\lambda_{t-s}a_{t-s})\mathbb{E}\left[\|\bar{\theta}_{t-s}-\theta^*\|_{2}^{2}\right]+\Phi_{t-s},\label{3T4}
\end{aligned}
\end{equation}
where the term $\Phi_{t-s}$ is given by
\begin{flalign}
&\Phi_{t-s}=6L^2\delta_{F,R}^2\left(2\lambda_{t-s}^2+\frac{\lambda_{t-s}}{a_{t-s}}\right)\mathbb{E}\big[\big\|\boldsymbol{\theta}_{t-s}-\mathbf{1}\bar{\theta}_{t-s}\big\|_{R}^2\big]\nonumber\\
&\quad+\frac{2\|u\|_{2}^2c_{z}^2}{m^2}\Big(2\gamma_{z}^{2(t-s)}+\frac{\gamma_{z}^{2(t-s)}}{a_{t-s}\lambda_{t-s}}\Big)\mathbb{E}\left[\|\boldsymbol{s}_{t+1-s}-\boldsymbol{s}_{t-s}\|_{F}^2\right]\nonumber\\
&\quad+\frac{2\|u\|_{2}^2(\sigma_{\vartheta}^{+})^2\sum_{i,j}(R_{ij})^2}{m^2(t-s+1)^{2\varsigma_{\vartheta}}}+\frac{2(\sigma_{\zeta}^{+})^2\sum_{i,j}(C_{ij})^2}{m^2(t-s+1)^{2\varsigma_{\zeta}}}\nonumber\\
&\quad+4(\kappa^2+D^2)\lambda_{t-s}^2 +\frac{\kappa^2(24\lambda_{t-s}^2+\frac{13\lambda_{t-s}}{a_{t-s}})}{(t-s+1)}.\label{3delta}
\end{flalign}

We define $\bar{t}=t-s+1$ for all $\bar{t}\geq 1$ and drop the negative term $-2\lambda_{t-s}\mathbb{E}\big[F(\theta_{i,t+1})-F(\theta^*)\big]$ to rewrite~\eqref{3T4} as follows:
\begin{equation}
\mathbb{E}\left[\|\bar{\theta}_{\bar{t}}-\theta^*\|_{2}^{2}\right]\leq (1+2\lambda_{\bar{t}-1}a_{\bar{t}-1})\mathbb{E}\left[\|\bar{\theta}_{\bar{t}-1}-\theta^*\|_{2}^{2}\right]+\Phi_{\bar{t}-1}.\label{3T5}
\end{equation}

By iterating~\eqref{3T5} from $0$ to $\bar{t}-1$, one yields
\begin{flalign}
&\mathbb{E}[\|\bar{\theta}_{\bar{t}}-\theta^*\|_{2}^{2}]
\leq \left(\prod_{p=0}^{\bar{t}-1}\left(1+2\lambda_{p}a_{p}\right)\right)\mathbb{E}\left[\left\|\bar{\theta}_{0}-\theta^*\right\|_{2}^{2}\right]\nonumber\\
&\quad+\sum_{p=0}^{\bar{t}-1}\left(\prod_{q=p+1}^{\bar{t}-1}\left(1+2\lambda_{q}a_{q}\right)\right)\Phi_{p}\nonumber\\
&\leq\left(\prod_{p=0}^{\bar{t}-1}(1+2\lambda_{p}a_{p})\right)\left(\mathbb{E}\left[\left\|\bar{\theta}_{0}-\theta^*\right\|_{2}^{2}\right]+\sum_{p=0}^{\bar{t}-1}\Phi_{p}\right),\label{3T6}
\end{flalign}
where we have used $\prod_{p=1}^{\bar{t}-1}(1+2\lambda_{p}a_{p})\leq \prod_{p=0}^{\bar{t}-1}(1+2\lambda_{p}a_{p})$ in the last inequality.

By using the relation $\ln(1+u)\leq u$ valid for any $u\geq 0$ and the definition $a_{p}=\frac{1}{(p+1)^{r}}$ with $r\in(\frac{1}{2},v)$, we have 
\begin{equation}
\begin{aligned}
&\ln\left(\prod_{p=0}^{\bar{t}-1}(1+2\lambda_{p}a_{p})\right)=\sum_{p=0}^{\bar{t}-1}\ln\left(1+2\lambda_{p}a_{p}\right)\leq\sum_{p=0}^{\bar{t}-1}2\lambda_{p}a_{p}\\
&\leq2\lambda_{0}+\sum_{p=2}^{\bar{t}}\frac{2\lambda_{0}}{p^{v+r}}\leq 2\lambda_{0}+2\lambda_{0}\int_{p=1}^{\infty}\frac{1}{p^{v+r}}\leq\frac{2\lambda_0(r+v)}{r+v-1},\nonumber
\end{aligned}
\end{equation}
which implies $\prod_{p=0}^{\bar{t}-1}(1+2\lambda_{p}a_{p})\leq e^{\frac{2\lambda_0(r+v)}{r+v-1}}$. We use this relation and replace $\bar{t}$ with $t-s+1$ to rewrite~\eqref{3T6} as follows:
\begin{equation}
\mathbb{E}[\|\bar{\theta}_{t-s+1}-\theta^*\|_{2}^{2}]
\leq e^{\frac{2\lambda_0(r+v)}{r+v-1}}\Big(\mathbb{E}\left[\|\bar{\theta}_{0}-\theta^*\|_{2}^{2}\right]+\sum_{s=0}^{t}\Phi_{t-s}\Big),\label{3T8}
\end{equation}
where in the derivation we used $\sum_{p=0}^{t-s}\Phi_{p}=\sum_{s=0}^{t}\Phi_{t-s}$. 
 
We proceed to estimate an upper bound on $\sum_{s=0}^{t}\Phi_{t-s}$, where $\Phi_{t-s}$ is defined in~\eqref{3delta}.

1) The relation $a_{t-s}\lambda_{t-s}\leq\lambda_{0}$ implies $\lambda_{t-s}^2\leq\frac{\lambda_{t-s}}{a_{t-s}}\lambda_{0}$. Combing this relationship with \eqref{l8result} in Lemma~\ref{l8} and the inequality $(t-s+1)^{p}\leq 2^{p}(t-s)^{p}$, we obtain
\begin{equation}
\begin{aligned}
&\sum_{s=0}^{t}(2\lambda_{t-s}^2+\frac{\lambda_{t-s}}{a_{t-s}})\mathbb{E}\left[\|\boldsymbol{\theta}_{t-s}-\mathbf{1}\bar{\theta}_{t-s}\|_{R}^2\right]\\
&\leq (2\lambda_{0}+1)\sum_{s=0}^{t}\frac{\lambda_{t-s}}{a_{t-s}}\left(\frac{4c_{\boldsymbol{\theta}1}}{(t-s+1)^{2}}+\frac{4^{v}c_{\boldsymbol{\theta}2}}{(t-s+1)^{2v}}\right.\\
&\left.\quad+\frac{4^{\varsigma_{\vartheta}}c_{\boldsymbol{\theta}3}}{(t-s+1)^{2\varsigma_{\vartheta}}}+\frac{4^{\varsigma_{\zeta}}c_{\boldsymbol{\theta}4}}{(t-s+1)^{2\varsigma_{\zeta}}}\right)\\
&\leq \lambda_{0}(2\lambda_{0}+1)\left(\frac{4c_{\boldsymbol{\theta}1}(v-r+2)}{v-r+1}+\frac{4^{v}c_{\boldsymbol{\theta}2}(3v-r)}{3v-r-1}\right.\\
&\left.\quad+\frac{4^{\varsigma_{\vartheta}}c_{\boldsymbol{\theta}3}(v-r+2\varsigma_{\vartheta})}{v-r+2\varsigma_{\vartheta}-1}+\frac{4^{\varsigma_{\zeta}}c_{\boldsymbol{\theta}4}(v-r+2\varsigma_{\zeta})}{v-r+2\varsigma_{\zeta}-1}\right),\label{3deltai}
\end{aligned}
\end{equation}
where we have used the following inequality in the derivation of the second inequality:
\begin{flalign}
\sum_{s=0}^{t}\frac{1}{(t-s+1)^{p}}=1+\sum_{s=2}^{t+1}\frac{1}{s^{p}}\leq1+\int_{s=1}^{\infty}\frac{1}{s^{p}}\leq \frac{p}{p-1},\label{jifen}
\end{flalign}
for all $p>1$.

2) The relation $a_{t-s}\lambda_{t-s}\leq\lambda_{0}$ implies $\gamma_{z}^{2(t-s)}\leq\frac{\gamma_{z}^{2(t-s)}}{a_{t-s}\lambda_{t-s}}\lambda_{0}$. By using this relation, taking the squared norm of~\eqref{1T2S3}, and utilizing inequality~\eqref{D14}, we obtain
\begin{flalign}
&\sum_{s=0}^{t}\left(2\gamma_{z}^{2(t-s)}+\frac{\gamma_{z}^{2(t-s)}}{a_{t-s}\lambda_{t-s}}\right)\mathbb{E}\left[\|\boldsymbol{s}_{t+1-s}-\boldsymbol{s}_{t-s}\|_{F}^2\right]\nonumber\\
&\leq (2\lambda_{0}+1)\sum_{s=0}^{t}\frac{\gamma_{z}^{2(t-s)}}{a_{t-s}\lambda_{t-s}}\Big(3\delta_{F,C}^2\|C\|_{C}^{2}\|\boldsymbol{s}_{t-s}-v\bar{s}_{t-s}\|_{C}^2\nonumber\\
&\quad+\frac{3(\sum_{i,j}C_{ij})^2(\sigma_{\zeta}^{+})^2}{(t+1-s)^{2\varsigma_{\zeta}}}+\frac{6m(\kappa^2+D^2)\lambda_{0}^2}{(t+1-s)^{2v}}\Big).\label{3deltaii}
\end{flalign}
Lemma~\ref{l6} implies $\gamma_{z}^{2(t-s)}\leq\frac{16\times 2^{4}}{(\ln(\gamma_{z})e)^4(t-s+1)^4}$. Substituting this relationship and~\eqref{l7result} into~\eqref{3deltaii} and using an argument similar to the derivation of~\eqref{3deltai}, we have
\begin{equation}
\begin{aligned}
&\sum_{s=0}^{t}\Big(2\gamma_{z}^{2(t-s)}+\frac{\gamma_{z}^{2(t-s)}}{a_{t-s}\lambda_{t-s}}\Big)\mathbb{E}\left[\|\boldsymbol{s}_{t+1-s}-\boldsymbol{s}_{t-s}\|_{F}^2\right]\\
&\leq\frac{2^{8}(2\lambda_{0}+1)}{(\ln(\gamma_{z})e)^{4}\lambda_{0}}\Bigg(3\delta_{F,C}^2\|C\|_{C}^{2}\Big(\frac{4c_{\boldsymbol{s}1}(6-v-r)}{5-v-r}\\
&\quad+\frac{4^{v}c_{\boldsymbol{s}2}(4+v-r)}{3+v-r}+\frac{4^{\varsigma_{\zeta}}c_{\boldsymbol{s}3}(4+2\varsigma_{\zeta}-v-r)}{3+2\varsigma_{\zeta}-v-r}\Big)\\
&\quad+\frac{3(\sum_{i,j}C_{ij})^2(\sigma_{\zeta}^{+})^2(4+2\varsigma_{\zeta}-v-r)}{3+2\varsigma_{\zeta}-v-r}\\
&\quad+\frac{6m(\kappa^2+D^2)\lambda_{0}^2(4+v-r)}{3+v-r}\Bigg).\label{3deltaii1}
\end{aligned}
\end{equation}

3) Applying~\eqref{jifen} to the rest of terms on the right hand side of $\sum_{s=0}^{t}\Phi_{t-s}$, we have
\begin{equation}
\small
\left\{\begin{aligned}
&\sum_{s=0}^{t}\left(\frac{1}{(t-s+1)^{2\varsigma_{\vartheta}}}\!+\!\frac{1}{(t-s+1)^{2\varsigma_{\zeta}}}\right)\!\leq\! \frac{2\varsigma_{\vartheta}}{2\varsigma_{\vartheta}-1}\!+\!\frac{2\varsigma_{\zeta}}{2\varsigma_{\zeta}-1},\\
&\sum_{s=0}^{t}\left(\lambda_{t-s}^2+\frac{\lambda_{t-s}^2}{t-s+1}\right)\leq\frac{2v\lambda_{0}^2}{2v-1}+\frac{(2v+1)\lambda_{0}^2}{2v},\\
&\sum_{s=0}^{t}\frac{\lambda_{t-s}}{a_{t-s}(t-s+1)}\leq \frac{(v+1-r)\lambda_{0}}{v-r}.\label{3deltaiii}
	\end{aligned}
\right.
\end{equation}

Substituting~\eqref{3deltai}-\eqref{3deltaiii} into~$\sum_{s=0}^{t}\Phi_{t-s}$, we arrive at
\begin{flalign}
\sum_{s=0}^{t}\Phi_{t-s}=\bar{c}_{\boldsymbol{\theta}1}+\bar{c}_{\boldsymbol{\theta}2}+\bar{c}_{\boldsymbol{\theta}3},\label{3T9}
\end{flalign}
where the constants $\bar{c}_{\boldsymbol{\theta}1}$, $\bar{c}_{\boldsymbol{\theta}1}$, and $\bar{c}_{\boldsymbol{\theta}3}$ are given in the statement of Theorem~\ref{t3}.

Substituting~\eqref{3T9} into~\eqref{3T8}, we can obtain
\begin{flalign}
&\mathbb{E}[\|\bar{\theta}_{t-s+1}-\theta^*\|_{2}^{2}]\nonumber\\
&\leq e^{\frac{2\lambda_0(r+v)}{r+v-1}}\Big(\mathbb{E}\left[\|\bar{\theta}_{0}-\theta^*\|_{2}^{2}\right]+\sum_{i=1}^{3}\bar{c}_{\boldsymbol{\theta}i}\Big)\triangleq\bar{c}'.\label{3T10}
\end{flalign}

We proceed to sum up both sides of~\eqref{3T4} from $0$ to $t$:
\begin{flalign}
&2\sum_{s=0}^{t}\lambda_{t-s}\mathbb{E}\big[F(\theta_{i,t+1})-F(\theta^*)\big]\leq -\sum_{s=0}^{t}\mathbb{E}\big[\big\|\bar{\theta}_{t-s+1}-\theta^*\big\|_{2}^{2}\big]\nonumber\\ 
&\quad+\sum_{s=0}^{t}(1+2\lambda_{t-s}a_{t-s})\mathbb{E}\big[\big\|\bar{\theta}_{t-s}-\theta^*\big\|_{2}^{2}\big]	
+\sum_{i=1}^{3}\bar{c}_{\boldsymbol{\theta}i},\label{3T11}
\end{flalign}
The first and second terms on the right hand side of~\eqref{3T11} can be simplified as
\begin{flalign}
&\sum_{s=0}^{t}(1+2\lambda_{t-s}a_{t-s})\mathbb{E}\big[\big\|\bar{\theta}_{t-s}-\theta^*\big\|_{2}^{2}\big]-\sum_{s=0}^{t}\mathbb{E}\big[\big\|\bar{\theta}_{t-s+1}-\theta^*\big\|_{2}^{2}\big]\nonumber\\ 
&\leq\sum_{s=0}^{t-1}2\lambda_{t-s}a_{t-s}\mathbb{E}\big[\|\bar{\theta}_{t-s}-\theta^*\|_{2}^{2}\big] +(1+2\lambda_{0}a_{0})\mathbb{E}\big[\|\bar{\theta}_{0}-\theta^*\|_{2}^{2}\big]\nonumber\\
&\leq\sum_{s=0}^{t-1}2\lambda_{t-s}a_{t-s}\bar{c}'\!+\!(1+2\lambda_{0})\mathbb{E}\big[\|\bar{\theta}_{0}-\theta^*\|_{2}^{2}\big]\!\leq\!\bar{c}_{\boldsymbol{\theta}4},\label{3T12}
\end{flalign}
where we have used~\eqref{3T10} and the relation $\sum_{s=0}^{t-1}2\lambda_{t-s}a_{t-s}\leq\frac{2(v+r)}{v+r-1}$ in the last inequality. The constant $\bar{c}_{\boldsymbol{\theta}4}$ is given in the statement of Theorem~\ref{t3}.

Given $2\sum_{s=0}^{t}\lambda_{s}\geq 2\int_{0}^{t+1}\frac{\lambda_{0}}{(x+1)^{v}}dx$, we have
\begin{equation}
2\sum_{s=0}^{t}\lambda_{t-s}=2\sum_{s=0}^{t}\lambda_{s}\geq \frac{2\lambda_{0}(t+2)^{1-v}-2\lambda_{0}}{1-v}.\label{3T13}
\end{equation}

Substituting~\eqref{3T12} and~\eqref{3T13} into~\eqref{3T11} yields
\begin{equation}
\begin{aligned}
\mathbb{E}\big[F(\theta_{i,t+1})-F(\theta^*)\big]\leq \frac{(1-v)\sum_{i=1}^{4}\bar{c}_{\boldsymbol{\theta}i}}{2\lambda_{0}\big((t+2)^{1-v}-1\big)}.
\end{aligned}
\end{equation}
Since $(t+2)^{1-v}\geq 2^{1-v}$ implies $1\leq \frac{1}{2^{1-v}}(t+2)^{1-v}$, we obtain $2\lambda_{0}\big((t+2)^{1-v}-1\big)\geq 2\lambda_{0}(1-\frac{1}{2^{1-v}})(t+1)^{1-v}$, which implies~\eqref{t3result} in Theorem~\ref{t3}.

\subsection{Proof of Lemma~\ref{l12}}
We first prove that the following inequality is valid for all $p\in[0,t]$:
\begin{equation}
	\frac{\beta_{t}}{\beta_{t-p}}\geq c_{0}\Big(1-\frac{c}{2}\Big)^{p},\label{4L1}
\end{equation}
where $\beta_{t}=\frac{\beta_{0}}{(t+1)^{q}}$ is given in the statement of Lemma~\ref{l12} and the constant $c_{0}$ is given by $c_{0}=\big(\frac{e\ln(\frac{2}{2-c})}{2q}\big)^{q}$. 

When constants $a,b,c,d>0$ satisfy $\frac{c}{a}\geq\frac{d}{b}$, the relationship $\frac{d}{b}\leq\frac{c+d}{a+b}\leq\frac{c}{a}$ always holds. This inequality further implies $\frac{t+2-p}{t+1}=\frac{t+1-p+1}{t+1-p+p}\geq \frac{1}{p}$ and $\frac{t-p+1}{t-p+2}\geq\frac{1}{2}$. Therefore, we have
\begin{equation}
	\frac{\beta_{t}}{\beta_{t-p}}=\Big(\frac{t-p+1}{t+1}\Big)^{q}\geq \frac{1}{(2p)^{q}}.\label{4L2}
\end{equation}
Consider a convex function $f(x)=-q\ln(x)-x\ln(1\!-\!\frac{c}{2}):  \mathbb{R}^{+}\!\rightarrow\!\mathbb{R}$ (where $\mathbb{R}^{+}$ represents the set of positive real numbers), whose derivative satisfies $f'(x)=-\frac{q}{x}-\ln(1-\frac{c}{2})$, implying the minimum point at $x^*=-\frac{q}{\ln(1-\frac{c}{2})}$ with the minimal value $f(x^*)=-q\ln\big(-\frac{q}{\ln(1-\frac{c}{2})}\big)+\frac{q}{\ln(1-\frac{c}{2})}\ln\big(1-\frac{c}{2}\big)=\ln(c_{0}2^{q}).$ Hence, for any $p\in \mathbb{N}^{+}$, we have $f(p)\geq \ln(c_{0}2^{q}),$ i.e.,$-q\ln(p)-p\ln\big(1-\frac{c}{2}\big)\geq \ln(c_{0}2^{q}),$
which is equivalent to $\ln\left((1-\frac{c}{2})^{p}\right)\leq \ln\big(\frac{1}{c_{0}2^{q}p^{q}}\big)$. Combining this relation with~\eqref{4L2} yields~\eqref{4L1}. 

We proceed to prove Lemma~\ref{l12}. By iterating $v_{t+1}\leq(1-c)v_{t}+\beta_{t}$ from $0$ to $t$, we obtain
\begin{equation}
	v_{t+1}\leq (1-c)(1-c)^{t}v_{0}+\sum_{p=0}^{t}\beta_{t-p}(1-c)^{p}.\label{4L3}
\end{equation}
By using the relation $(1-c)^{p}\leq(1-\frac{c}{2})^{2p}$, we have
\begin{equation}
	\sum_{p=0}^{t}\beta_{t-p}(1-c)^{p}\leq \beta_{t}\sum_{p=0}^{t}\frac{\beta_{t-p}}{\beta_{t}}\Big(1-\frac{c}{2}\Big)^{2p}.\nonumber
\end{equation}
Based on~\eqref{4L1}, one yields $\frac{\beta_{t-p}}{\beta_{t}}(1-\frac{c}{2})^{p}\leq \frac{1}{c_{0}}$, which implies
\begin{equation}
	\sum_{p=0}^{t}\beta_{t-p}(1-c)^{p}\leq \frac{\beta_{t}}{c_{0}}\sum_{p=0}^{t}\Big(1-\frac{c}{2}\Big)^{p}\leq \frac{2\beta_{t}}{c c_{0}}.\label{4L5}
\end{equation}
Using again inequality~\eqref{4L1}, we obtain
\begin{equation}
	(1-c)^{t}v_{0}\leq \Big(1-\frac{c}{2}\Big)^{t}v_{0}\leq\frac{v_{0}}{c_{0}\beta_{0}}\beta_{t}.\label{4L6}
\end{equation}
Substituting~\eqref{4L5} and~\eqref{4L6} into~\eqref{4L3}, we arrive at $v_{t}\leq \frac{1}{c_{0}}\big(\frac{v_{0}(1-c)}{\beta_{0}}+\frac{2}{c}\big)\beta_{t-1}$. Further using the relationship $\beta_{t-1}\leq 2^{q}\beta_{t}$ and the definition of $c_{0}$ yields Lemma~\ref{l12}.

\bibliographystyle{ieeetr}  
\bibliography{nonconvexquantization}

\end{document}